\documentclass[11pt]{article}

\usepackage[noheader]{styles/redhat_arxiv}

\usepackage{natbib}

\newcommand{\arxivversion}{1}

\usepackage[utf8]{inputenc}
\usepackage[T1]{fontenc}
\usepackage{amsmath, amssymb}
\usepackage{mathtools}
\usepackage{bm}
\usepackage{microtype}
\usepackage{booktabs}
\usepackage{multirow}
\usepackage{colortbl}  
\usepackage{nicefrac}
\usepackage{graphicx}
\usepackage{subcaption}
\usepackage{xcolor}
\usepackage{enumitem}
\usepackage{bbm}  

\definecolor{OIorange}{HTML}{E69F00}
\definecolor{OIsky}   {HTML}{56B4E9}
\definecolor{OIgreen} {HTML}{009E73}
\definecolor{OIyellow}{HTML}{F0E442}
\definecolor{OIblue}  {HTML}{0072B2}
\definecolor{OIverm}  {HTML}{D55E00}
\definecolor{OIpurp}  {HTML}{CC79A7}
\definecolor{OIgrey}  {gray}{0.20}

\colorlet{spectral}{OIblue}
\colorlet{factual}{OIgreen}
\colorlet{hallucinate}{OIverm}
\colorlet{baseline1}{OIorange}
\colorlet{baseline2}{OIgrey}
\colorlet{highlight}{OIyellow}
\colorlet{neutral}{OIgrey}

\usepackage{float}

\usepackage{siunitx}
\usepackage{makecell}
\usepackage{array}
\renewcommand{\arraystretch}{1.25}

\newcolumntype{R}[1]{>{\raggedleft\arraybackslash}p{#1}}

\usepackage{tikz}
\usepackage{styles/pgfplots_academic}
\usetikzlibrary{intersections,arrows.meta,positioning,shapes.geometric,calc,patterns,decorations.pathreplacing}

\pgfplotsset{
  cheegerband/.style={academic-single},
  violation/.style={academic-single, ytick=\empty},
  cvar/.style={academic-single-bar}
}

\tikzset{myarrow/.style={-{Latex[length=2.2mm]}, line width=1pt}}

\usepackage{hyperref}
\hypersetup{
  colorlinks=true,
  linkcolor=OIblue,
  citecolor=OIblue,
  urlcolor=OIblue
}

\usepackage{cleveref}

\crefformat{equation}{Eq.~(#2#1#3)}
\Crefformat{equation}{Equation~(#2#1#3)}
\crefrangeformat{equation}{Eqs.~(#3#1#4)--(#5#2#6)}
\Crefrangeformat{equation}{Equations~(#3#1#4)--(#5#2#6)}
\crefmultiformat{equation}{Eqs.~(#2#1#3)}{, (#2#1#3)}{, (#2#1#3)}{, and~(#2#1#3)}
\Crefmultiformat{equation}{Equations~(#2#1#3)}{, (#2#1#3)}{, (#2#1#3)}{, and~(#2#1#3)}

\crefformat{figure}{Fig.~#2#1#3}
\Crefformat{figure}{Figure~#2#1#3}
\crefrangeformat{figure}{Figs.~#3#1#4--#5#2#6}
\Crefrangeformat{figure}{Figures~#3#1#4--#5#2#6}
\crefmultiformat{figure}{Figs.~#2#1#3}{, #2#1#3}{, #2#1#3}{, and~#2#1#3}
\Crefmultiformat{figure}{Figures~#2#1#3}{, #2#1#3}{, #2#1#3}{, and~#2#1#3}

\crefformat{table}{Table~#2#1#3}
\Crefformat{table}{Table~#2#1#3}
\crefrangeformat{table}{Tables~#3#1#4--#5#2#6}
\Crefrangeformat{table}{Tables~#3#1#4--#5#2#6}
\crefmultiformat{table}{Tables~#2#1#3}{, #2#1#3}{, #2#1#3}{, and~#2#1#3}
\Crefmultiformat{table}{Tables~#2#1#3}{, #2#1#3}{, #2#1#3}{, and~#2#1#3}

\crefformat{section}{Sec.~#2#1#3}
\Crefformat{section}{Section~#2#1#3}
\crefrangeformat{section}{Secs.~#3#1#4--#5#2#6}
\Crefrangeformat{section}{Sections~#3#1#4--#5#2#6}
\crefmultiformat{section}{Secs.~#2#1#3}{, #2#1#3}{, #2#1#3}{, and~#2#1#3}
\Crefmultiformat{section}{Sections~#2#1#3}{, #2#1#3}{, #2#1#3}{, and~#2#1#3}

\crefformat{appendix}{Appendix~#2#1#3}
\Crefformat{appendix}{Appendix~#2#1#3}
\crefrangeformat{appendix}{Appendices~#3#1#4--#5#2#6}
\Crefrangeformat{appendix}{Appendices~#3#1#4--#5#2#6}

\crefname{proposition}{Proposition}{Propositions}
\Crefname{proposition}{Proposition}{Propositions}
\crefname{theorem}{Theorem}{Theorems}
\Crefname{theorem}{Theorem}{Theorems}
\crefname{lemma}{Lemma}{Lemmas}
\Crefname{lemma}{Lemma}{Lemmas}
\crefname{corollary}{Corollary}{Corollaries}
\Crefname{corollary}{Corollary}{Corollaries}
\crefname{definition}{Definition}{Definitions}
\Crefname{definition}{Definition}{Definitions}
\crefname{remark}{Remark}{Remarks}
\Crefname{remark}{Remark}{Remarks}
\crefname{assumption}{Assumption}{Assumptions}
\Crefname{assumption}{Assumption}{Assumptions}
\crefname{example}{Example}{Examples}
\Crefname{example}{Example}{Examples}
\crefname{hypothesis}{Hypothesis}{Hypotheses}
\Crefname{hypothesis}{Hypothesis}{Hypotheses}

\crefname{algorithm}{Algorithm}{Algorithms}
\Crefname{algorithm}{Algorithm}{Algorithms}

\crefname{enumi}{item}{items}
\Crefname{enumi}{Item}{Items}

\crefformat{subsection}{Sec.~#2#1#3}
\Crefformat{subsection}{Section~#2#1#3}
\crefrangeformat{subsection}{Secs.~#3#1#4--#5#2#6}
\Crefrangeformat{subsection}{Sections~#3#1#4--#5#2#6}

\crefformat{subsubsection}{Sec.~#2#1#3}
\Crefformat{subsubsection}{Section~#2#1#3}

\crefformat{}{Sec.~#2#1#3}
\Crefformat{}{Section~#2#1#3}
\crefrangeformat{}{Secs.~#3#1#4--#5#2#6}
\Crefrangeformat{}{Sections~#3#1#4--#5#2#6}
\crefmultiformat{}{Secs.~#2#1#3}{, #2#1#3}{, #2#1#3}{, and~#2#1#3}
\Crefmultiformat{}{Sections~#2#1#3}{, #2#1#3}{, #2#1#3}{, and~#2#1#3}

\usepackage{xr}
\newcommand{\suppref}[1]{\NoHyper\S\ref{S-#1}\endNoHyper}

\providecommand{\arxivversion}{0}

\def\1{\bm{1}}

\DeclareMathAlphabet{\mathsfit}{\encodingdefault}{\sfdefault}{m}{sl}
\SetMathAlphabet{\mathsfit}{bold}{\encodingdefault}{\sfdefault}{bx}{n}

\newcommand{\M}{\mathcal{M}}           
\newcommand{\B}{\mathbf{B}}            
\newcommand{\DQ}{D_Q}                  
\newcommand{\DK}{D_K}                  
\newcommand{\Hh}{\mathcal{H}}          
\newcommand{\Msym}{\M_{\mathrm{sym}}}  
\newcommand{\Masym}{\M_{\mathrm{asym}}} 

\newcommand{\sig}{\sigma}              

\newcommand{\vol}{\mathrm{vol}}        

\DeclareMathOperator{\trace}{Tr}

\newcommand{\reals}{\mathbb{R}}

\makeatletter
\@ifundefined{theorem}{%
  \newtheorem{theorem}{Theorem}
  \newtheorem{lemma}[theorem]{Lemma}
  \newtheorem{proposition}[theorem]{Proposition}
  \newtheorem{corollary}[theorem]{Corollary}
  \newtheorem{definition}[theorem]{Definition}
  
  \newtheorem{remark}[theorem]{Remark}
}{}
\@ifundefined{hypothesis}{}{}
\@ifundefined{assumption}{}{}
\@ifundefined{observation}{}{}
\makeatother

\makeatletter
\@ifundefined{proof}{%
  \newenvironment{proof}[1][\proofname]{\par\noindent\textbf{#1.}\space}{\hfill$\square$\par\medskip}
  \providecommand{\proofname}{Proof}
}{}
\makeatother

\providecommand{\Name}[2][]{#2}
\providecommand{\Email}[1]{\texttt{#1}}
\providecommand{\addr}{}
\providecommand{\nametag}[1]{#1}
\providecommand{\AND}{\and}

\usepackage{pifont}

\providecommand{\phihat}{\widehat{\phi}}

\newcommand{\lenAurocHalueval}{0.96}

\newcommand{\lenAurocMedhallu}{0.74}

\newcommand{\lenAurocTruthfulqa}{0.56}

\newcommand{\pythiaHalPhiDiffuse}{0.822}
\newcommand{\pythiaHalPhiDiffuseLo}{0.798}
\newcommand{\pythiaHalPhiDiffuseHi}{0.845}
\newcommand{\pythiaHalPhiBottleneck}{0.617}
\newcommand{\pythiaHalPhiBottleneckLo}{0.590}
\newcommand{\pythiaHalPhiBottleneckHi}{0.645}
\newcommand{\pythiaHalSigmaTwoStd}{0.838}
\newcommand{\pythiaHalSigmaTwoStdLo}{0.815}
\newcommand{\pythiaHalSigmaTwoStdHi}{0.862}

\newcommand{\confoundMaxDelta}{0.28}
\newcommand{\confoundOCPythiaHal}{0.13}
\newcommand{\confoundHiddenMeanDelta}{+0.09}
\newcommand{\confoundLogitMeanDelta}{-0.01}

\newcommand{\gabflantFiveDecoderHaluevalGonly}{0.779}

\newcommand{\gabflantFiveDecoderHaluevalDelta}{-0.012}
\newcommand{\gabflantFiveDecoderHaluevalDeltaLo}{-0.017}
\newcommand{\gabflantFiveDecoderHaluevalDeltaHi}{-0.008}

\newcommand{\gabllamaHaluevalGonly}{0.539}

\newcommand{\gabllamaHaluevalDelta}{+0.025}
\newcommand{\gabllamaHaluevalDeltaLo}{+0.014}
\newcommand{\gabllamaHaluevalDeltaHi}{+0.034}

\newcommand{\gabNumCells}{15}

\newcommand{\gabMeanDelta}{+0.012}
\newcommand{\gabMaxDelta}{+0.089}
\newcommand{\gabMaxDeltaLo}{+0.013}
\newcommand{\gabMaxDeltaHi}{+0.126}
\newcommand{\gabMaxDeltaCell}{Pythia/MedHallu}

\newcommand{\gabNumSigPositive}{5}

\newcommand{\floorViolGPTLo}{18}
\newcommand{\floorViolGPTHi}{19}
\newcommand{\tstarGPTLo}{0.35}
\newcommand{\tstarGPTHi}{0.37}
\newcommand{\floorViolGPTHalu}{18}
\newcommand{\floorViolGPTNormLo}{36}
\newcommand{\floorViolGPTNormHi}{42}

\newcommand{\floorViolPyLo}{45}
\newcommand{\floorViolPyHi}{63}
\newcommand{\tstarPyLo}{0.47}
\newcommand{\tstarPyHi}{0.66}
\newcommand{\floorViolPyHalu}{51}

\newcommand{\floorViolFlanLo}{80}
\newcommand{\floorViolFlanHi}{82}
\newcommand{\tstarFlanLo}{0.44}
\newcommand{\tstarFlanHi}{0.47}
\newcommand{\floorViolFlanHalu}{80}

\begin{document}

\title[Self-Attention as Transport]{Self-Attention as Transport:\\
Orientation Blindness and the Limits of Symmetric Spectral Diagnostics}

\author{\Name{Dominik Dahlem\nametag{$^\ast$}} \\ \addr Red Hat AI \\ \Email{ddahlem@redhat.com}
\AND
\Name{Diego Maniloff} \\ \addr Red Hat AI \\ \Email{dmanilof@redhat.com}
\AND
\Name{Mac Misiura} \\ \addr Red Hat AI \\ \Email{mmisiura@redhat.com}}

\maketitle

{\small $^\ast$Corresponding author.}\vspace{0.5em}

\begin{abstract}
  Every attention head defines a degree-normalized transport operator, and a
  growing family of diagnostics reads model behavior (hallucination among
  them) from its spectrum.  We ask what such diagnostics can and cannot infer.
  The operator splits orthogonally into a symmetric part governing transport
  \emph{capacity} and an antisymmetric part encoding \emph{orientation}.  We prove
  an identifiability limit: every transpose-invariant spectral diagnostic is
  \emph{orientation-blind} (unable to distinguish an operator from its transpose,
  hence blind to the orientation of information flow), with a transpose-stability bound
  limiting any Lipschitz diagnostic's transpose sensitivity by the asymmetry
  coefficient $G$.  This bounds what spectral diagnostics of the attention operator
  can resolve (e.g.\ LapEigvals and the attention-spectral branch of LLM-Check).  On
  the surviving axis, a closed-form bipartite-Cheeger landscape shows uniform causal
  attention obeys an $n$-independent \emph{temporal-cut} floor $\phi \ge 1/5$ while
  window attention pierces it as $O(w/n)$; the floor is an idealized benchmark, not
  an empirical attractor, and the fraction of real heads falling below it is itself an
  empirically stable architectural descriptor.  The two-axis diagnostic ($\phi$ for capacity, $G$ for
  asymmetry magnitude) yields a falsifiable polarity prediction, borne out
  \emph{in sign} under length-controlled, forced-scoring evaluation across
  decoder-only, encoder-only, and encoder--decoder models (capacity-axis signal
  0.62--0.84 LC-AUROC): polarity reverses between HaluEval and MedHallu,
  directionally as predicted though asymmetric in strength, with decision
  polarity calibrated per regime.
\end{abstract}

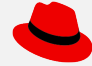
\ifnum\arxivversion=0
\fi

\section{Introduction}

Large language models hallucinate, and the failures are not all alike.
Some samples concentrate attention on a narrow set of positions while
ignoring relevant context; others spread attention so thinly that no
signal carries.  These two routing pathologies look quantitatively
similar through standard spectral lenses, yet they are mechanistically
opposite.  Attention-based diagnosis must distinguish them, and the
obstacle is structural rather than a matter of finding a better statistic.

Every attention head in a transformer~\citep{Vaswani2017Attention}
defines a bipartite transport operator between queries and keys.  Degree-normalising this operator
yields a scale-invariant representation $\M$ whose spectral properties
encode how information routes through the network.  This paper asks
what attention-based diagnostics can and cannot measure as a function
of which mathematical object they analyse, not to build a stronger
hallucination detector, but to characterise the structural boundary.

We establish a structural boundary on what these diagnostics can
resolve.  The transport operator decomposes orthogonally under the
Hilbert--Schmidt inner product into a
symmetric component $\Msym$ governing transport \emph{capacity} and an
antisymmetric component $\Masym$ encoding routing \emph{orientation}.
Every spectral diagnostic depending only on singular values or on the
symmetric component is invariant under transpose: structurally
\emph{orientation-blind} (\cref{thm:orientation-blindness}).  A matching \emph{transpose-stability
bound} limits the transpose sensitivity of any Lipschitz diagnostic by
$\|\Masym\|_F$, within a fixed normalized class
(\cref{prop:ob-converse}).  This places
a precise limit on what symmetric spectral methods such as
LLM-Check~\citep{Sriramanan2024LLMCheck},
EigenTrack~\citep{Ettori2025EigenTrack}, and
LapEigvals~\citep{Binkowski2025LapEigvals} can resolve, regardless of
which symmetric statistic they extract.

The symmetric axis nevertheless supports a rich diagnostic.  The
classical Cheeger inequality provides a two-sided bound relating
conductance $\phi$ to the spectral gap $1-\sigma_2$, yielding a
two-sided certificate: low conductance indicates a bottleneck regime,
anomalously high conductance diffuse mixing.  A
degree sufficiency theorem (\cref{prop:degree-sufficiency}) formalises
when coupling structure contributes beyond degree heterogeneity.  The
antisymmetric axis is targeted by the asymmetry coefficient $G$, the
normalised Frobenius distance to the symmetric subspace, which measures the
\emph{magnitude} of the antisymmetric residual.

A closed-form bipartite-Cheeger landscape provides a reference model for canonical
causal routing.  Uniform causal attention satisfies an $n$-independent
\emph{temporal-cut} floor $\phi \ge 1/5$, while window attention falls below it
on balanced cuts as $O(w/n)$ (\cref{sec:cheeger}).  The fraction of empirical heads
falling below $1/5$ becomes a population-level architectural descriptor
distinguishing position-encoding regimes.

Length-robust evaluation is a prerequisite, not a contribution: prior
work established that a shared length bias in detector scores and
correctness labels skews AUROC rankings~\citep{Santilli2025LengthBias}
and that length heuristics alone can match complex detectors under
standard automatic evaluation~\citep{Janiak2025IllusionProgress}.
Spectral features inherit length dependence through three diagnosable
confounding channels; length-controlled AUROC deflates raw scores by
up to $\confoundMaxDelta$ points (\cref{sec:length-confound}).

Under this protocol, transport diagnostics retain interpretable signal:
between-dataset polarity variation (bottleneck routing on HaluEval,
diffuse routing on MedHallu) reflects regime-dependent failure modes
predicted by the two-sided theory (\cref{sec:evaluation}).

\medskip
\noindent
The paper's central result is an \textbf{identifiability limit}: every symmetric spectral
diagnostic of an attention operator is structurally orientation-blind
(\cref{thm:orientation-blindness}), and the rest of the paper characterises what remains
accessible once that limit is in force.  The converse (\cref{prop:ob-converse}) pins the
discarded information to the antisymmetric residual $\|\Masym\|_F$; on the surviving axis,
conductance is the symmetric summary with a two-sided structural certificate, a closed-form
Cheeger landscape says what it means architecturally, and the experiments validate the
consequences under length control.
Concretely:

\begin{enumerate}[nosep]
\item The \textbf{orientation-blindness theorem} (\cref{thm:orientation-blindness}), with its
      projection reading, a quantitative converse, and a rigidity corollary: any diagnostic
      depending only on singular values or on the symmetric component is
      transpose-invariant; any transpose-sensitive functional must depend on the
      antisymmetric component proportionally to $\|\Masym\|_F$ (\cref{prop:ob-converse});
      and orientation is \emph{exactly} what the transpose-invariant class discards, no
      more and no less (\cref{cor:transpose-rigidity}).  This is a precise limit on what
      spectral methods can resolve through the attention operator, regardless of which
      statistic they extract: the eigenvalue- and singular-value-based features of
      LLM-Check, EigenTrack, and LapEigvals lie inside the blind class
      (\cref{cor:transpose-invariant-diagnostics}).
\item The \textbf{symmetric axis that survives the limit}: reading the classical
      symmetric--antisymmetric split as transport \emph{capacity} versus routing
      \emph{orientation}, with Cheeger conductance as a
      two-sided capacity certificate and a degree-sufficiency theorem
      (\cref{prop:degree-sufficiency}) formalizing when coupling structure contributes
      beyond degree heterogeneity (\cref{sec:cheeger}).
\item A closed-form \textbf{conductance landscape} (the paper's principal new
      derivation) that says what conductance means architecturally:
      $\phi(S_t)\ge u(t)/(2+u(t))$, $u(t)=H_n-H_t$, for uniform causal
      attention (\cref{lem:uniform-causal-conductance}), with an $n$-independent floor
      $\phi\ge 1/5$ (\cref{cor:uc-one-fifth}; sharp asymptotic value
      \cref{cor:uc-sharp-floor}) and window attention piercing it as $O(w/n)$
      (\cref{lem:window-conductance}).  This is an \emph{architectural reference model}, not
      a claim about real heads: the population fraction of heads piercing $1/5$ is the
      observable that distinguishes position-encoding regimes.
\item \textbf{Empirical validation} of the consequences under length control: transport
      features retain interpretable signal, and the polarity of the transport signal is
      regime-dependent rather than universal (bottleneck routing on HaluEval, diffuse on
      MedHallu).  The symmetric (capacity) axis carries the dominant signal in decoder-only
      transformers and transfers unchanged to a modern decoder panel (Qwen2.5, SmolLM2-1.7B,
      LLaMA-3.1-8B); the antisymmetric (asymmetry-magnitude) axis is sparse, with
      architecture-dependent exceptions in Flan-T5 cross-attention and Pythia RoPE offered
      as explanatory hypotheses rather than validated mechanisms (\cref{sec:evaluation}).
\end{enumerate}

\medskip
\noindent
\textbf{Provenance of the formal results.}
Every proof in the paper uses classical matrix analysis: the Hermitian
dilation, orthogonal projection, Cauchy--Schwarz, Weyl's inequality,
harmonic sums, and single-variable convexity.  This is by design: an
identifiability boundary is useful exactly when any reader can verify it
with standard tools.  Classical infrastructure enters with attribution
(the Jordan--Wielandt correspondence, the Fan--Hoffman nearest-symmetric
property, Weyl's inequality, the Cheeger inequality), and one-step records
such as the Lipschitz converse (\cref{prop:ob-converse}) are stated for
completeness.  What is new is the question and its answer: the
orientation-blindness theorem and its rigidity converse resolve exactly
what the transpose-invariant class can identify; the closed-form
sweep-conductance landscape is new as a mathematical object, with exact
constants (the $1/5$ floor, the sharp asymptotic value $c^\star$, the
causal ceiling $G \le 1/\sqrt{2}$); and
\cref{cor:transpose-invariant-diagnostics} places the published spectral
detectors inside the blind class.  Proof steps carry inline provenance
tags (\emph{foundations}, \emph{bridge}, \emph{contribution}; see the
footnote in \cref{sec:orientation-blindness}).

\noindent
Features are computed without labeled data; calibrated decision-making
requires a modest labeled set to establish polarity and thresholds
(\cref{sec:discussion}).  \cref{tab:claim-audit} is the reader's map: it ties
each headline claim to its supporting result, its empirical evidence, and the
scope within which it holds.


\begin{table}[t]
\centering
\footnotesize
\setlength{\tabcolsep}{4pt}
\renewcommand{\arraystretch}{1.25}
\caption{\textbf{Claim audit.} Each headline claim, the formal result that
supports it, the empirical evidence, and the scope within which the claim holds.
A \textsc{status} tag marks each claim as a \textsc{theorem} proved here, a
theory-derived \textsc{prediction}, or an \textsc{empirical} observation (no theorem
predicts it). Throughout, $\phi$ is the exact
(NP-hard) conductance of the raw attention graph and $\phihat$ its spectral-sweep
estimator; reported LC-AUROC values use label-informed polarity and therefore
measure \emph{signal strength}, not deployed-detector accuracy
(\cref{sec:eval-protocol}).}
\label{tab:claim-audit}
\begin{tabular}{@{}p{2.9cm}p{2.5cm}p{3.6cm}p{4.0cm}@{}}
\toprule
\textbf{Claim} & \textbf{Status \& support} & \textbf{Empirical evidence} & \textbf{Scope / limitation} \\
\midrule
Symmetric spectral diagnostics are orientation-blind.
& \textsc{Theorem.}~\cref{thm:orientation-blindness}; converse \cref{prop:ob-converse}.
& Proof; \cref{cor:transpose-invariant-diagnostics} lists affected methods (LLM-Check, EigenTrack, LapEigvals).
& Bounds only transpose-invariant functionals of $\Msym$; does not constrain hidden-state methods. \\

Conductance is a two-sided capacity axis (bottleneck vs.\ diffuse).
& \textsc{Theorem.}~\cref{thm:two-sided-diagnostic}; Cheeger inequality.
& Diffuse tail $\pythiaHalPhiDiffuse\,[\pythiaHalPhiDiffuseLo,\pythiaHalPhiDiffuseHi]$ vs.\ bottleneck tail $\pythiaHalPhiBottleneck$ (HaluEval); terciles $47.5\%$/$45.1\%$.
& Measured via the estimator $\phihat$, not the exact object $\phi$; the sweep ratio is empirically within $2\times$ on matched nulls. \\

Uniform causal attention obeys a floor $\phi\ge 1/5$; window attention pierces it as $O(w/n)$.
& \textsc{Theorem.}~\cref{lem:uniform-causal-conductance}, \cref{cor:uc-one-fifth}, \cref{lem:window-conductance}.
& \cref{tab:landscape-empirical-summary}: raw-graph temporal-cut floor-violation fraction separates architectures ($\floorViolGPTLo$--$\floorViolGPTHi\%$ / $\floorViolPyLo$--$\floorViolPyHi\%$ / $\floorViolFlanLo$--$\floorViolFlanHi\%$).
& The floor is an \emph{idealized-architecture} benchmark, not an empirical attractor; the \emph{violation fraction} is the observable, not a theory failure. \\

Polarity is regime-dependent and reverses across benchmarks.
& \textsc{Prediction.}~Falsifiable consequence of the two-sided theory.
& Reversal HaluEval (low-OC) $\leftrightarrow$ MedHallu (high-OC); \cref{fig:phi-sigma2-scatter}.
& Label-informed flip selects polarity; within-dataset polarity is bin-consistent in only roughly half of model--dataset pairings (7 of 12), a signal-strength claim, not a deployed detector. \\

$\phihat$ is largely degree-reducible; $\sigma_2$ carries coupling beyond degree.
& \textsc{Empirical {\&} prop.}~\cref{prop:degree-sufficiency} supports the $\sigma_2$ half only.
& Degree-preserving nulls: $\phihat$ z-AUROC $0.51$--$0.57$ (chance); $\sigma_2$ $0.72$--$0.80$.
& The $\phihat$-side reducibility is \emph{empirical}; \cref{prop:degree-sufficiency} bounds only $\sigma_2$'s coupling term. \\

The antisymmetric axis $G$ is sparse but architecture-dependent.
& \textsc{Empirical} (sparsity); \textsc{theorem} (floors).~\cref{prop:asymmetry-energy}, \cref{prop:firstcol-g-bound}, \cref{cor:exponential-g-bound} give floors under sink/decay.
& $G$ $0.53$--$0.63$ in most configurations; exceptions Flan-T5 $0.78$, Pythia $0.82$ (HaluEval).
& The proposed mechanisms (cross-attention grounding, RoPE recency) are \textsc{hypotheses} consistent with the patterns, not confirmed causes. \\

Transport features carry length-independent signal.
& \textsc{Empirical} (protocol).~Length-controlled evaluation (\cref{sec:eval-protocol}).
& $0.62$--$0.84$ LC-AUROC after control; response length alone collapses to ${\approx}0.50$.
& Computed under \emph{forced scoring} of benchmark-provided responses; no claim about online generation dynamics. \\
\bottomrule
\end{tabular}
\end{table}

\section{Preliminaries: Attention as a Transport Operator}
\label{sec:transport-operators}

We fix the paper's central object: the degree-normalized bipartite transport
operator $\M$, a scale-invariant view of an attention head that applies to any
row-stochastic matrix.  Our use of ``transport'' is structural, not literal:
unlike Sinkformers~\citep{Sander2021Sinkformers} (Sinkhorn doubly-stochastic
optimal-transport plans) or the Markov-chain reading of
\citet{Erel2025AttentionMarkovChains}, we solve no transport problem; we
analyze the existing attention mechanism \emph{as} an operator whose routing
quality its spectral geometry bounds.  \cref{fig:pipeline-schematic} previews the
diagnostic pipeline; the objects it names are defined below.

\begin{figure}[t]
\centering
\resizebox{\linewidth}{!}{%
\begin{tikzpicture}[
    >=Latex,
    font=\footnotesize,
    node distance=5mm,
    nodebox/.style={rounded corners=2pt, thick, align=center, inner sep=3pt,
                    minimum height=10mm, text width=2.3cm},
    stage/.style={nodebox, draw=black!55, fill=black!3},
    eval/.style={nodebox, draw=OIgreen!70, fill=OIgreen!12},
    capnode/.style={nodebox, draw=OIblue!75, fill=OIblue!10},
    dirnode/.style={nodebox, draw=OIverm!75, fill=OIverm!10},
    flow/.style={-Latex, thick, draw=black!65},
    lbl/.style={font=\scriptsize\itshape, text=black!65},
]

\node[stage] (A) {Attention $\B$\\[-1pt] \scriptsize row-stochastic};
\node[stage, right=of A] (Bip) {Bipartite\\[-1pt] $A_{\mathrm{bip}}=\Hh(\B)$};
\node[stage, right=of Bip] (M) {Operator $\M$\\[-1pt] \scriptsize $\DQ^{-1/2}\B\DK^{-1/2}$};

\node[stage, right=14mm of M] (AGG) {Aggregate\\[-1pt] \scriptsize CVaR, std, mean};
\node[capnode, above=7mm of AGG] (CAP)
  {Symmetric\\[-1pt] \scriptsize capacity: $\phihat,\ \sigma_2$};
\node[dirnode, below=7mm of AGG] (DIR)
  {Antisymmetric\\[-1pt] \scriptsize direction: $G$};
\node[eval, right=of AGG] (LC) {LC-AUROC\\[-1pt] \scriptsize length-controlled};

\draw[flow] (A) -- (Bip);
\draw[flow] (Bip) -- (M);

\draw[flow, draw=OIblue!75] (M.north) to[out=90,in=180] (CAP.west);
\draw[flow, draw=OIverm!75] (M.south) to[out=-90,in=180] (DIR.west);

\draw[flow, draw=OIblue!75] (CAP.south) -- (AGG.north);
\draw[flow, draw=OIverm!75] (DIR.north) -- (AGG.south);

\draw[flow] (AGG) -- (LC);

\node[lbl, above=0.8mm of CAP] {\textcolor{OIblue!80}{orientation-blind}
  vs.\ \textcolor{OIverm!80}{direction-sensitive} axes};

\end{tikzpicture}%
}
\caption{\textbf{The transport-diagnostic pipeline.}  Each attention head's
row-stochastic matrix $\B$ is read as a bipartite graph $A_{\mathrm{bip}}=\Hh(\B)$
and degree-normalized into the operator $\M$.  The Hilbert--Schmidt-orthogonal
split of $\M$ defines two axes: the \emph{symmetric} component $\Msym$ carries
transport \emph{capacity}, read by conductance $\phihat$ and the second singular
value $\sigma_2$ (every transpose-invariant spectral diagnostic lives here, and
is therefore orientation-blind; \cref{thm:orientation-blindness}); the
\emph{antisymmetric} component $\Masym$ carries \emph{direction}, read by the
asymmetry coefficient $G$ (\cref{prop:ob-converse}).  Per-axis features are
aggregated across layers and heads (CVaR, standard deviation) and evaluated under
length-controlled AUROC (\cref{sec:eval-protocol}).  Theoretical statements
concern the exact object $\phi$ on $A_{\mathrm{bip}}$; empirical results concern
the estimator $\phihat$.}
\label{fig:pipeline-schematic}
\end{figure}

\begin{center}
\fbox{\parbox{\dimexpr\linewidth-2\fboxsep-2\fboxrule\relax}{%
\small\textbf{Notation: object vs.\ estimator.}~The \emph{object} $\phi$ is the
exact, scale-invariant, NP-hard \citep{SimaSchaeffer2006NPConductance} conductance
of the raw bipartite attention graph $\Hh(\B)$ (defined below).  The \emph{estimator}
$\phihat$ approximates it from finite attention matrices via the degree-normalized
operator $\M$, inheriting finite-size and degree-normalization dependence.
\textbf{All theoretical results in this paper concern the object $\phi$ on
$\Hh(\B)$; all empirical results concern the estimator $\phihat$ under
length-controlled evaluation.}  Observed length correlations in $\phihat$ are
estimator artifacts, not length-dependence of $\phi$.}}
\end{center}

\begin{definition}[Degree-normalized bipartite operator]
\label{def:bipartite-operator}
Let $\B \in \mathbb{R}_{\ge 0}^{n_q \times n_k}$ denote the row-stochastic attention matrix for a head (mask applied; softmax over keys), with degree matrices $\DQ = \mathrm{diag}(\B \mathbf{1})$ and $\DK = \mathrm{diag}(\B^\top \mathbf{1})$.
The \emph{degree-normalized cross-operator} is
\begin{equation}
\M = \DQ^{-1/2}\, \B\, \DK^{-1/2}.
\label{eq:M}
\end{equation}
This operator induces a weighted bipartite graph between queries and keys, and
depends only on $(Q,K)$ and the mask, not on values.
\end{definition}

\paragraph{Two bipartite graphs.}
The symmetric block embedding
\begin{equation}
\Hh(X) =
\begin{pmatrix}
0 & X\\
X^\top & 0
\end{pmatrix}
\label{eq:dilation}
\end{equation}
turns a rectangular matrix into a bipartite graph on $Q \cup K$.  Two such
graphs play distinct roles, related by degree normalization.  The \emph{raw
attention graph} $A_{\mathrm{bip}} = \Hh(\B)$ carries edge weights $\B_{ij}$;
\emph{graph conductance is a property of $A_{\mathrm{bip}}$}.  The
\emph{normalized adjacency} $N_{\mathrm{bip}} = \Hh(\M)$ carries the normalized
weights $\M_{ij}$ and satisfies $N_{\mathrm{bip}} = D^{-1/2} A_{\mathrm{bip}}
D^{-1/2}$, where $D$ is the degree matrix of $A_{\mathrm{bip}}$.  The dilation
$\Hh(\M)$ is symmetric with off-diagonal blocks $\M$ and $\M^\top$, turning the
rectangular operator into a symmetric spectral problem whose spectrum is exactly
the (signed) singular values of $\M$.

\begin{lemma}[Dilation spectrum; Jordan--Wielandt]
\label{lem:dilation-spectrum}
Let $\M \in \mathbb{R}^{n_q \times n_k}$ have singular values
$\sigma_1 \ge \cdots \ge \sigma_r \ge 0$, $r = \min(n_q, n_k)$.  Then
$\Hh(\M)$ has eigenvalues $\{\pm\sigma_1, \dots, \pm\sigma_r\}$ together with
$|n_q - n_k|$ zeros, so the bipartite Laplacian
$\mathcal{L}_{\mathrm{bip}} = I - \Hh(\M)$ has eigenvalues $1 \pm \sigma_i$.
This is the classical Jordan--Wielandt correspondence~\citep{hornMatrixAnalysis2012,golub13}.
\end{lemma}

\noindent So the SVD of $\M$ supplies the spectrum of
$N_{\mathrm{bip}} = \Hh(\M)$ at no extra cost.  The Cheeger inequality
(\cref{sec:cheeger}) bridges these two graphs: conductance on
$A_{\mathrm{bip}}$ to the spectral gap of $N_{\mathrm{bip}}$.


\begin{figure*}[!ht]
\centering
\begin{tikzpicture}[
    >=Latex,
    node/.style={circle, draw, thick, inner sep=1pt, minimum size=5pt},
    qnode/.style={node, fill=OIblue!25, draw=OIblue!70},
    knode/.style={node, fill=OIgreen!25, draw=OIgreen!70},
    edge/.style={-Latex, line cap=round},
    vweak/.style={edge, draw=black!20, line width=0.35pt},
    weak/.style={edge, draw=black!30, line width=0.5pt},
    med/.style={edge, draw=black!55, line width=0.75pt},
    strong/.style={edge, draw=black!80, line width=1.0pt},
    panel/.style={rounded corners=2pt, draw=black!40, thick, fill=white},
    paneltitle/.style={font=\footnotesize\bfseries, anchor=south},
    panelsub/.style={font=\tiny, text=black!70, anchor=south},
    annot/.style={font=\tiny, text=black!80},
    metric/.style={font=\tiny\itshape, text=black!70, anchor=north west},
]

\def\panelw{3.0cm}
\def\panelh{2.5cm}
\def\panelsep{0.15cm}
\def\orthosep{0.35cm}
\def\qx{0.50}
\def\kx{2.50}

\node[panel, minimum width=\panelw, minimum height=\panelh, anchor=north west] (A) at (0,0) {};
\node[paneltitle] at ($(A.north)+(0,0.25)$) {(a) Bottleneck};
\node[panelsub] at ($(A.north)+(0,0.05)$) {Concentrated};

\foreach \i/\y in {1/0.70, 2/1.15, 3/1.60} {
    \node[qnode] (Aq\i) at ($(A.south west)+(\qx,\y)$) {};
    \node[knode] (Ak\i) at ($(A.south west)+(\kx,\y)$) {};
}
\node[annot, anchor=south] at ($(Aq3.north)+(0,0.06)$) {$Q$};
\node[annot, anchor=south] at ($(Ak3.north)+(0,0.06)$) {$K$};

\draw[strong] (Aq1) -- (Ak2);
\draw[strong] (Aq2) -- (Ak2);
\draw[strong] (Aq3) -- (Ak2);
\draw[vweak] (Aq1) -- (Ak1);
\draw[vweak] (Aq3) -- (Ak3);

\draw[decorate, decoration={brace, amplitude=2pt, mirror}, thick, draw=OIverm!80]
    ($(A.south west)+(1.45,0.85)$) -- ($(A.south west)+(1.45,1.45)$);
\node[font=\tiny, text=OIverm!90, anchor=south east]
    at ($(A.south west)+(1.40,1.50)$) {cut};

\node[metric] at ($(A.south west)+(0.08,0.40)$) {$\phihat$ \textbf{low}};

\node[panel, minimum width=\panelw, minimum height=\panelh, anchor=north west] (B) at ($(A.north east)+(\panelsep,0)$) {};
\node[paneltitle] at ($(B.north)+(0,0.25)$) {(b) Intermediate};
\node[panelsub] at ($(B.north)+(0,0.05)$) {Selective};

\foreach \i/\y in {1/0.70, 2/1.15, 3/1.60} {
    \node[qnode] (Bq\i) at ($(B.south west)+(\qx,\y)$) {};
    \node[knode] (Bk\i) at ($(B.south west)+(\kx,\y)$) {};
}
\node[annot, anchor=south] at ($(Bq3.north)+(0,0.06)$) {$Q$};
\node[annot, anchor=south] at ($(Bk3.north)+(0,0.06)$) {$K$};

\draw[med] (Bq1) -- (Bk1);
\draw[weak] (Bq1) -- (Bk2);
\draw[weak] (Bq2) -- (Bk1);
\draw[med] (Bq2) -- (Bk2);
\draw[weak] (Bq2) -- (Bk3);
\draw[weak] (Bq3) -- (Bk2);
\draw[med] (Bq3) -- (Bk3);

\node[metric] at ($(B.south west)+(0.08,0.40)$) {$\phihat$ optimal; $G$ ok};

\node[panel, minimum width=\panelw, minimum height=\panelh, anchor=north west] (C) at ($(B.north east)+(\panelsep,0)$) {};
\node[paneltitle] at ($(C.north)+(0,0.25)$) {(c) Diffuse};
\node[panelsub] at ($(C.north)+(0,0.05)$) {Uniform};

\foreach \i/\y in {1/0.70, 2/1.15, 3/1.60} {
    \node[qnode] (Cq\i) at ($(C.south west)+(\qx,\y)$) {};
    \node[knode] (Ck\i) at ($(C.south west)+(\kx,\y)$) {};
}
\node[annot, anchor=south] at ($(Cq3.north)+(0,0.06)$) {$Q$};
\node[annot, anchor=south] at ($(Ck3.north)+(0,0.06)$) {$K$};

\foreach \i in {1,2,3} {
    \foreach \j in {1,2,3} {
        \draw[vweak] (Cq\i) -- (Ck\j);
    }
}

\node[metric] at ($(C.south west)+(0.08,0.40)$) {$\phihat$ \textbf{high}};

\draw[densely dotted, thick, black!50]
    ($(C.north east)+(0.18,0.30)$) -- ($(C.south east)+(0.18,-0.1)$);

\node[panel, minimum width=\panelw, minimum height=\panelh, anchor=north west] (D) at ($(C.north east)+(\orthosep,0)$) {};
\node[paneltitle] at ($(D.north)+(0,0.25)$) {(d) Self-attending};
\node[panelsub] at ($(D.north)+(0,0.05)$) {Temporal isolation};

\foreach \i/\y in {1/0.70, 2/1.15, 3/1.60} {
    \node[qnode] (Dq\i) at ($(D.south west)+(\qx,\y)$) {};
    \node[knode] (Dk\i) at ($(D.south west)+(\kx,\y)$) {};
}
\node[annot, anchor=south] at ($(Dq3.north)+(0,0.06)$) {$Q$};
\node[annot, anchor=south] at ($(Dk3.north)+(0,0.06)$) {$K$};

\draw[strong] (Dq1) -- (Dk1);
\draw[strong] (Dq2) -- (Dk2);
\draw[strong] (Dq3) -- (Dk3);
\draw[vweak] (Dq2) -- (Dk1);
\draw[vweak] (Dq3) -- (Dk1);
\draw[vweak] (Dq3) -- (Dk2);

\node[metric] at ($(D.south west)+(0.08,0.40)$) {$\phihat$ ok; $G \to 0$};

\end{tikzpicture}

\caption{\textbf{Conductance has a bounded intermediate band; asymmetry detects temporal isolation.}
Each attention head defines bipartite transport between queries ($Q$) and keys ($K$).
\textbf{(a)--(c)}~Conductance spectrum: the intermediate band sits between too low (bottleneck, a) and too high (diffuse dilution, c).
\textbf{(a)}~Bottleneck: concentrated attention yields low $\phihat$, with little weight on the rest of the context.
\textbf{(b)}~Intermediate: selective routing with task-relevant weighting; $\phihat$ in the intermediate band.
\textbf{(c)}~Diffuse: uniform attention yields high $\phihat$ but dilutes task-relevant signal (Theorem~\ref{thm:two-sided-diagnostic}).
\textbf{(d)}~Self-attending (orthogonal axis): as attention concentrates on the diagonal, historical context is ignored; $\phihat$ appears normal but $G \to 0$, detectable only by asymmetry coefficient.}
\label{fig:concept-transport}
\end{figure*}
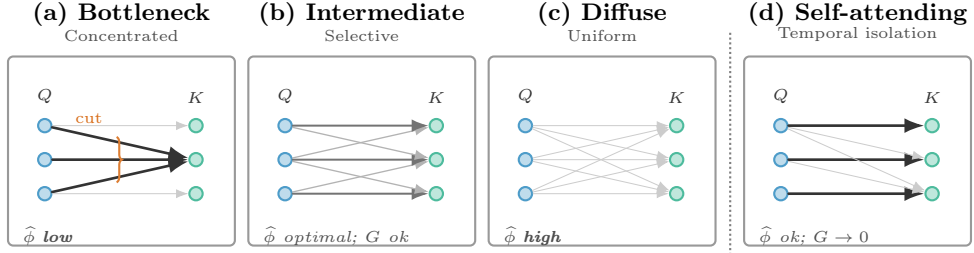

\medskip\noindent
With $\M$ defined as a transport operator, we ask what its spectral structure reveals about routing quality.


\section{Spectral Structure of Attention}
\label{sec:cheeger}

With the transport operator $\M$ in hand, the natural diagnostic
question is: how well does $\M$ move information between query and
key sets?  On the raw bipartite attention graph $A_{\mathrm{bip}} =
\Hh(\B)$ this is measured by graph conductance, and the Cheeger
inequality bounds that conductance in terms of the spectral gap of the
normalized adjacency $N_{\mathrm{bip}} = \Hh(\M)$, making the second
singular value $\sigma_2(\M)$ a computable proxy for transport quality.
This section
develops the two diagnostic tools that follow from this framing:
conductance and its closed-form architectural benchmark, the spectral
gap and a degree-sufficiency theorem that explains when coupling
structure contributes beyond degree heterogeneity, and a two-sided
diagnostic that certifies both bottleneck and diffuse routing failures.

\paragraph{Conductance.}
Graph conductance is a property of the \emph{raw attention graph}
$A_{\mathrm{bip}} = \Hh(\B)$ (\cref{sec:transport-operators}): for a bipartite
set $S\subseteq Q\cup K$,
\[
\phi(S) = \frac{w(S,\bar S)}{\min(\mathrm{vol}(S), \mathrm{vol}(\bar S))},
\]
where the cut weight $w(S,\bar S)$ and the volumes $\mathrm{vol}(\cdot)$ are
computed from the edge weights $\B_{ij}$ of $A_{\mathrm{bip}}$.
Because $A_{\mathrm{bip}}$ is symmetric (hence reversible as a Markov chain
after volume-normalisation), the classical Cheeger machinery applies even
when $\B$ itself is not doubly-stochastic.
Slow mixing means information cannot move efficiently between distant tokens; a bottleneck is any small subset of heads that traps information flow.
Low conductance indicates bottlenecked transport (over-concentrated routing).

\paragraph{The Cheeger inequality.}
The classical normalized-Laplacian Cheeger bridge relates the conductance $\phi$ of
$A_{\mathrm{bip}}=\Hh(\B)$ to the spectral gap $1-\sigma_2$ of $N_{\mathrm{bip}}=\Hh(\M)$,
where $\sigma_2$ is the second singular value of $\M$ (\cref{lem:dilation-spectrum}):
\begin{equation}
\frac{\phi^2}{2} \leq 1 - \sigma_2 \leq 2\phi.
\label{eq:cheeger-inequality}
\end{equation}
The lower bound is the discrete Cheeger inequality
\citep{cheegerLowerBoundSmallest2015}; the bipartite-dilation form is the standard
specialization \citep[Ch.~2]{Chung1997}.  It is classical; our use is interpretive:
reading $\Hh(\M)$ as the relevant graph identifies \emph{which} routing object the
symmetric spectrum can certify (and, via orientation blindness, which it cannot), and the
closed-form landscape of \cref{thm:two-sided-diagnostic} turns ``a bottleneck cut exists''
into a per-architecture benchmark.

\paragraph{Spectral norm variability as transport diagnostic.}
When $\sigma_2$ is close to $1$ the spectral gap is small, so by
\eqref{eq:cheeger-inequality} conductance is low, a bottleneck; aggregating the
across-layer standard deviation of $\sigma_2$ therefore measures cross-layer variability
in transport capacity.  Exact $\phi$ is NP-hard, so we use a spectral-sweep estimator
$\phihat$: rank vertices by the second singular vector of $\M$ and minimise the sweep
conductance over the resulting threshold cuts (\cref{sec:asymmetric-guessing}).  Carried
out on $N_{\mathrm{bip}}$, $\phihat$ coincides with $\phi$ on $A_{\mathrm{bip}}$ when
column degrees are near-regular and otherwise differs by a degree-ratio factor $\kappa$
(\cref{prop:degree-sufficiency}); like any finite-size estimator it inherits degree and
length dependence, motivating the length-controlled protocol (\cref{sec:evaluation}).

\begin{theorem}[Spectral stability under near-regular column degrees]
\label{prop:degree-sufficiency}
Let $A \in \mathbb{R}^{n_q \times n_k}$ be row-stochastic with positive column degrees
$d_j = \sum_i A_{ij}$. Define $\M = D_Q^{-1/2} A D_K^{-1/2}$ (with $D_Q = I$ for
row-stochastic $A$), $\bar{d} = (\sum_j d_j)/n_k$ the mean column degree, and
$\kappa = d_{\max}/d_{\min}$ the degree ratio. Then:
\begin{enumerate}[nosep,label=(\roman*)]
\item \textbf{Sharp constraint:} $\sigma_1(\M) = 1$.
\item \textbf{Spectral perturbation:}
$|\sigma_2(\M) - \sigma_2(A)/\sqrt{\bar{d}}|
  \leq \max_j \bigl|1 - \sqrt{d_j/\bar{d}}\,\bigr|
  \leq \sqrt{\kappa} - 1$.
\item \textbf{Doubly-stochastic exactness:} When $A$ is doubly-stochastic ($\kappa = 1$),
$\M = A$ and the bound vanishes.
\end{enumerate}
\end{theorem}

\noindent
Part (i) uses only row-stochasticity and positive column degrees: $\sigma_1(\M)=1$
holds for \emph{every} such $A$, regular or not; the proof is the standard
normalized-operator argument \citep[cf.][Ch.~1]{Chung1997} adapted to the
rectangular row-stochastic case.  The near-regularity parameter
$\kappa$ enters only through the tightness of the perturbation bound in part (ii),
where the content lies: it delimits when coupling structure contributes to
$\sigma_2$ beyond degree heterogeneity.

\begin{proof}
(i) Let $\sqrt{d_K} \in \mathbb{R}^{n_k}$ denote the vector with entries
$(\sqrt{d_K})_j = \sqrt{d_j}$.  Since $D_Q = I$,
\[
(\M \sqrt{d_K})_i \;=\; \textstyle\sum_j A_{ij}\, d_j^{-1/2}\,\sqrt{d_j}
\;=\; \textstyle\sum_j A_{ij} \;=\; 1,
\]
so $\M\sqrt{d_K} = \mathbf{1}$, and
$\|\sqrt{d_K}\|^2 = \sum_j d_j = \sum_{i,j} A_{ij} = n_q = \|\mathbf{1}\|^2$;
hence $\|\M\sqrt{d_K}\| = \|\sqrt{d_K}\|$ and $\sigma_1(\M) \geq 1$.
Conversely, for any $x$ set $y = D_K^{-1/2}x$; Cauchy--Schwarz with the
row-stochastic weights $A_{ij}$ gives
$(\M x)_i^2 = \bigl(\sum_j A_{ij}\, y_j\bigr)^2 \leq \sum_j A_{ij}\, y_j^2$,
and summing over $i$ and swapping the sums,
$\|\M x\|^2 \leq \sum_j d_j\, y_j^2 = \|x\|^2$.
Hence $\sigma_1(\M) \leq 1$, and $\sigma_1(\M) = 1$.
(ii) Write $\M - A/\!\sqrt{\bar{d}} = A\Delta$ where $\Delta = D_K^{-1/2} - \bar{d}^{-1/2}I$
is diagonal with $\delta_j = d_j^{-1/2} - \bar{d}^{-1/2}$.
By Weyl's perturbation theorem,
$|\sigma_2(\M) - \sigma_2(A/\!\sqrt{\bar{d}})| \leq
\sigma_1(A\Delta)$.
Apply Cauchy--Schwarz per row with row-stochastic weights $A_{ij}$:
\begin{align*}
\|A\Delta x\|^2
  &= \textstyle\sum_i \bigl(\sum_j A_{ij}\,\delta_j\, x_j\bigr)^{\!2}
   \leq \sum_i \sum_j A_{ij}\,\delta_j^2\, x_j^2 \\
  &= \textstyle\sum_j \delta_j^2\, d_j\, x_j^2
   \leq \max_j(\delta_j^2 d_j)\;\|x\|^2,
\end{align*}
where the first inequality uses $\sum_j A_{ij} = 1$ and the sum swap uses $d_j = \sum_i A_{ij}$.
The key cancellation: $\delta_j^2 \cdot d_j = (1 - \sqrt{d_j/\bar{d}})^2$, giving
$\sigma_1(A\Delta) \leq \max_j |1 - \sqrt{d_j/\bar{d}}|$.
Since $d_{\min} \leq \bar{d} \leq d_{\max}$, each term satisfies
$|1 - \sqrt{d_j/\bar{d}}| \leq \sqrt{\kappa} - 1$
(the case $d_j \geq \bar{d}$ uses $d_j/\bar{d} \leq \kappa$; the case
$d_j < \bar{d}$ uses $d_j/\bar{d} \geq 1/\kappa$ and the AM-GM inequality
$1 - 1/\!\sqrt{\kappa} \leq \sqrt{\kappa} - 1$).
(iii) For doubly-stochastic $A$, $D_K = I$ and $\bar{d} = 1$, so $\M = A$ and the
bound equals $\sqrt{1} - 1 = 0$.
\end{proof}

\noindent
This theorem formalizes the near-regular intuition from the degree-preserving null analysis:
when degree heterogeneity is moderate ($\kappa \approx 1$, equivalently $\delta$ small
for $\kappa \leq 1 + \delta$), the degree normalization is a small perturbation of
uniform scaling. The explicit bound $\sqrt{\kappa} - 1 \leq \delta/2$ for $\kappa \leq
1 + \delta$ quantifies how much coupling structure contributes beyond degrees, a
prediction we test empirically via degree-only baselines, Sinkhorn projection to
doubly-stochastic form, and conditional AUROC residualization (Online Supplement, \suppref{app:degree-sufficiency}).

We hypothesize that factual generation accesses specific high-attention tokens (creating degree heterogeneity), while hallucinated generation involves diffuse or inappropriately concentrated attention.  \cref{prop:degree-sufficiency} bears on this only through normalization: when $\kappa\approx1$, degree normalization is close to a uniform scalar rescaling, so $\sigma_2(\M)$ tracks $\sigma_2(A)$ up to scale.  It does \emph{not} say coupling is reducible to degrees: two matrices with identical degrees can have very different $\sigma_2$ (e.g.\ $I$ versus $\tfrac1n\mathbf{1}\mathbf{1}^\top$, both with $\kappa=1$ but $\sigma_2=1$ versus $0$).  Whether label discrimination is carried by degrees or by coupling given degrees is the separate empirical question of \cref{sec:evaluation}.

\paragraph{Matched null validation.}
A natural objection is that $\phihat$ discrimination reflects finite-size
estimator artifacts rather than genuine learned structure.  Matched null
baselines that preserve attention-specific constraints while randomizing learned
structure are inconsistent with that artifact explanation: hallucination samples
deviate systematically from their own null expectations, and z-scoring against the
nulls retains discrimination
(improving it by $6$--$8$ points on the lower-baseline models;
\cref{tab:null-validation-body}).
Degree-preserving nulls further reveal a clean division of labour: $\phihat$'s
discriminative power is largely a proxy for the degree distribution (which tokens receive
attention), whereas $\sigma_2$ carries the coupling information that survives degree
control.  The two axes are complementary, not redundant: $\phihat$ is mostly degrees,
$\sigma_2$ mostly coupling; \cref{prop:degree-sufficiency} bounds only the $\sigma_2$
coupling term, and the $\phihat$-side reducibility is empirical.  Crucially, this
reducibility is the Cheeger mechanism at work, not a defect: degree concentration is
what produces the low-conductance cuts $\phihat$ detects, so conductance is not
``merely a degree statistic'' but the structural quantity through which degree
concentration manifests as a bottleneck.

\begin{table}[t]
\centering
\small
\setlength{\tabcolsep}{6pt}
\caption{\textbf{Matched-null validation (entropy-matched, HaluEval).}
Spectral discrimination persists after z-scoring each sample against $m{=}50$
entropy-matched nulls: the lower-baseline models gain $6$--$8$ points, the rest
stay near their ceiling.  Hallucination samples therefore deviate
\emph{systematically} from their own null expectation, which rules out the
objection that the signal is a finite-size estimator artifact.  Full results
(all benchmarks, $\sigma_2$ and $\tau$, and the degree-preserving nulls that
decompose the signal) are in the Online Supplement.}
\label{tab:null-validation-body}
\begin{filecontents*}{data/tables/null_validation_body.csv}
# Main-body matched-null validity check (entropy-matched, HaluEval, 2 dp)
# Generated by: scripts/generate_null_body_table.py -- DO NOT EDIT BY HAND
# Source: data/null_validation_entropy_matched.csv (halueval rows,
#   cols auroc_phi_raw/_z, auroc_sigma2_raw/_z)
Architecture,phi_raw,phi_z,sigma2_raw,sigma2_z
BERT,0.69,0.77,0.70,0.79
GPT-2,0.70,0.75,0.73,0.72
Pythia-160M,0.97,0.94,0.98,0.96
Flan-T5 cross,0.97,0.97,0.97,0.92
Flan-T5 dec.,0.95,0.97,0.97,0.97
\end{filecontents*}
\pgfplotstabletypeset[
  col sep=comma,
  string type,
  skip first n=4,  
  columns={Architecture, phi_raw, phi_z, sigma2_raw, sigma2_z},
  every head row/.style={
    before row={\toprule
      & \multicolumn{2}{c}{$\phihat$ AUROC} & \multicolumn{2}{c}{$\sigma_2$ AUROC} \\
      \cmidrule(lr){2-3}\cmidrule(lr){4-5}},
    after row=\midrule
  },
  every last row/.style={after row=\bottomrule},
  columns/Architecture/.style={column name=Architecture, column type=l},
  columns/phi_raw/.style={column name=raw, column type=c},
  columns/phi_z/.style={column name=z-score, column type=c},
  columns/sigma2_raw/.style={column name=raw, column type=c},
  columns/sigma2_z/.style={column name=z-score, column type=c},
]{data/tables/null_validation_body.csv}
\end{table}

The volume mechanism is concrete: high-degree vertices dominate the volume term, so the
characteristic degree distributions of hallucinating models are exactly what the Cheeger
inequality certifies as transport bottlenecks.  We report the quantitative decomposition
in \cref{sec:evaluation}.

\paragraph{Two-sided diagnostic.}
Unlike metrics that respond in only one direction, conductance provides signal on both sides of an intermediate-conductance band (Figure~\ref{fig:concept-transport}).

\begin{theorem}[Temporal-cut separation: window vs.\ uniform-causal]
\label{thm:two-sided-diagnostic}
In decoder self-attention ($n_q = n_k = n$, causal mask):
\begin{enumerate}[nosep,leftmargin=*]
\item \textbf{Bottleneck} ($w$-window attention):
$\phi(S_t) \le w/\min(t,n-t)$ for every nontrivial temporal cut; balanced cuts
therefore satisfy $\phi(S_t)=O(w/n)$, and late-half cuts satisfy
$\phi(S_t)\le w/(n-t)$.
\item \textbf{Diffuse} (uniform causal attention):
$\phi(S_t) \ge 1/5$ uniformly in $n$; the volume-crossover branch is
localised by the formal half-way bound.
\end{enumerate}
The two families occupy opposite ends of the temporal-cut conductance landscape.
This is a structural fact about the matrices, not a claim that either is incorrect:
window attention falls below the temporal-cut reference value on balanced cuts (a
global bottleneck family), while uniform-causal attention has a localised bottleneck
candidate flanked by high-conductance regions on both sides.  Whether either regime
\emph{corresponds} to incorrect responses is a separate, task-dependent empirical
question (\cref{sec:evaluation}).
\end{theorem}

Low $\phihat$ marks over-concentrated routing; high $\phihat$ marks under-concentrated routing.
This two-sided sensitivity distinguishes conductance from prior spectral metrics that respond in only one direction.
The diffuse-side bound is a structural statement: uniform causal attention has conductance bounded away from zero by an $n$-independent constant.\footnote{The $1/5$ bound is a clean, fully provable constant. The \emph{exact} asymptotic floor of the closed-form sweep landscape is $c^\star = 1 - 2x^\star \approx 0.36431$, where $x^\star \in (0,1)$ is the unique root of $x(2-\ln x)=1$ (equivalently $c^\star = 1 + 2/W_{-1}(-e^{-2})$, with $W_{-1}$ the lower branch of the Lambert $W$ function); see \cref{cor:uc-sharp-floor}. This is the floor of the \emph{sweep} landscape that the diagnostic computes; the graph conductance, a minimum over all cuts, is NP-hard and strictly smaller.}
The claim that diffuse mixing indicates failure is empirical and task-dependent, validated by the tercile analysis in \cref{sec:evaluation}.

The cut-and-volume identities underlying the diffuse-side bound
(\cref{prop:uc-cut-vol-identities}, \cref{prop:dilation-total-vol}) reveal the
\emph{shape} of conductance as a function of cut location:
$\phi(S_t)\ge u(t)/(2+u(t))$ where $u(t)=H_n-H_t$ is the harmonic mass
above $t$ (\cref{lem:uniform-causal-conductance}). This functional form
is the substantive diagnostic content -- it tells us where bottleneck
candidates sit and how conductance varies smoothly around them, not
just that they exist. The same row-sum / column-sum decomposition
template applies to any causal architecture (window, exponential decay,
RoPE-style decay): each yields its own closed-form $\mathrm{cut}(S_t)$
and $\vol(S_t)$, and the conductance-as-functional-of-$t$ characterises
its failure-mode geometry in the same way.

\paragraph{Empirical landscape signatures.}
Real attention does not satisfy the uniform-causal closed form. We
compare the closed-form prediction (\cref{fig:landscape-theory}) with
the empirical landscape from
up to $50$ HaluEval samples per model (length filter $n\ge 32$, exact temporal-cut
conductance on the \emph{raw} bipartite graph $A_{\mathrm{bip}}=\Hh(\B)$ (the object
the $1/5$ floor is proved for), with no Fiedler approximation; median curves with IQR
bands in \cref{fig:landscape-empirical}).
The architectural descriptor lives at the \emph{distribution} level,
not at individual heads: the IQR bands of GPT-2 and Pythia-160M
overlap substantially across $t/n$, so head-by-head separation
between these architectures is not supported by the per-cut data.
What separates them is the population-level fraction of heads
falling below the $1/5$ temporal-cut floor: $\floorViolGPTLo$--$\floorViolGPTHi\%$ for GPT-2,
$\floorViolPyLo$--$\floorViolPyHi\%$ for Pythia-160M, and $\floorViolFlanLo$--$\floorViolFlanHi\%$ for Flan-T5 decoder
across HaluEval, MedHallu, and TruthfulQA
(\cref{tab:landscape-empirical-summary}, which reports 95\% clustered
bootstrap CIs).  Clustering the bootstrap by sample (the independent
unit) widens these intervals well beyond the $\approx 0.6\,$pp a
binomial interval over the $\sim\!7200$ correlated (sample, layer, head)
rows would give, to as much as $\pm 4\,$pp for the architectures with
strongly correlated heads; the architectural ordering GPT-2 $<$
Pythia-160M $<$ Flan-T5 decoder nonetheless remains separated at $95\%$
confidence. Because the empirical temporal-cut conductance is computed on the \emph{same} raw
bipartite graph $A_{\mathrm{bip}}=\Hh(\B)$ that the $1/5$ floor is proved for
(\cref{cor:uc-one-fifth}), the comparison is like-for-like: the floor-violation fraction
is exactly the share of heads whose temporal-cut conductance falls below the theorem's
threshold, not a comparison across mismatched objects.  (On the degree-normalized graph
$\Hh(\M)$ the same heads fall below $1/5$ more often (e.g.\ $\approx\!\floorViolGPTNormLo$--$\floorViolGPTNormHi\%$ for
GPT-2) because normalization lowers conductance; we report the raw-graph numbers as the
theorem-matched quantity.  The architectural ordering GPT-2 $<$ Pythia-160M $<$ Flan-T5
decoder is unchanged under either graph convention; the convention shifts the level of
the descriptor, not the ordering.)
Floor-violation fraction
is the empirically robust architectural descriptor;
worst-cut location $t^\ast/n$ shifts in the same direction
($\tstarGPTLo$--$\tstarGPTHi$ for GPT-2, $\tstarPyLo$--$\tstarPyHi$ for Pythia, $\tstarFlanLo$--$\tstarFlanHi$ for
Flan-T5) but with substantial within-architecture variance and
overlapping ranges, so it should be read as indicative rather than as
a clean architectural classifier.


\begin{figure}[!ht]
\centering

\begin{subfigure}[t]{0.49\textwidth}
\centering
\begin{tikzpicture}
\begin{axis}[
    academic-single,
    width=0.90\linewidth,
    height=0.78\linewidth,
    xlabel={Cut location $t/n$},
    ylabel={Conductance $\phi(S_t)$},
    label style={font=\scriptsize},
    tick label style={font=\tiny},
    xmin=0, xmax=1,
    ymin=0, ymax=1.05,
    xtick={0,0.5,1},
    xticklabels={0,0.5,1},
    ytick={0,0.2,0.5,1},
    yticklabels={0,$\tfrac15$,0.5,1},
    ymajorgrids=true, xmajorgrids=false,
    legend style={
        at={(0.5,1.02)}, anchor=south,
        font=\tiny, draw=none,
        cells={anchor=west},
        fill=white, fill opacity=0.85,
        legend columns=2,
    },
    legend cell align=left,
]
\addplot[dashed, black!50, line width=0.7pt, forget plot]
    coordinates {(0,0.2) (1,0.2)};
\addplot[fill=hallucinate!20, draw=none, forget plot]
    coordinates {(0,0)(1,0)(1,0.2)(0,0.2)} -- cycle;
\addplot[spectral, line width=1.0pt, mark=none, smooth]
    table[x=t_norm, y=phi_uc] {content/figures/conductance_landscape/data/landscape_curves.dat};
\addlegendentry{Uniform causal}
\addplot[hallucinate, line width=1.0pt, mark=none, dashed, smooth]
    table[x=t_norm, y=phi_w5] {content/figures/conductance_landscape/data/landscape_curves.dat};
\addlegendentry{Window $w{=}5$}
\addplot[OIorange, line width=1.0pt, mark=none, densely dashed, smooth]
    table[x=t_norm, y=phi_w20] {content/figures/conductance_landscape/data/landscape_curves.dat};
\addlegendentry{Window $w{=}20$}
\addplot[factual, line width=1.0pt, mark=none, dotted]
    coordinates {(0,0)(1,0)};
\addlegendentry{Diagonal $w{=}1$}
\end{axis}
\end{tikzpicture}
\caption{Theory: failure-mode shapes ($n=100$).}
\label{fig:landscape-theory}
\end{subfigure}%
\hfill
\begin{subfigure}[t]{0.49\textwidth}
\centering
\begin{tikzpicture}
\begin{axis}[
    academic-single,
    width=0.90\linewidth,
    height=0.78\linewidth,
    xlabel={Cut location $t/n$},
    ylabel={Conductance $\phihat(S_t)$},
    label style={font=\scriptsize},
    tick label style={font=\tiny},
    xmin=0, xmax=1,
    ymin=0, ymax=1.05,
    xtick={0,0.5,1},
    xticklabels={0,0.5,1},
    ytick={0,0.2,0.5,1},
    yticklabels={0,$\tfrac15$,0.5,1},
    ymajorgrids=true, xmajorgrids=false,
    legend style={
        at={(0.5,1.02)}, anchor=south,
        font=\tiny, draw=none,
        cells={anchor=west},
        fill=white, fill opacity=0.85,
        legend columns=2,
    },
    legend cell align=left,
]
\addplot[dashed, black!50, line width=0.7pt, forget plot]
    coordinates {(0,0.2) (1,0.2)};
\addplot[black, line width=0.9pt, mark=none, smooth, dashed]
    table[x=t_norm, y=phi_uc] {content/figures/conductance_landscape/data/landscape_curves.dat};
\addlegendentry{Theory (uniform causal)}
\addplot[name path=gpt2_p25_c, draw=none, forget plot]
    table[x=t_norm, y=p25] {content/figures/conductance_landscape/data/empirical_gpt2_halueval.dat};
\addplot[name path=gpt2_p75_c, draw=none, forget plot]
    table[x=t_norm, y=p75] {content/figures/conductance_landscape/data/empirical_gpt2_halueval.dat};
\addplot[fill=spectral, fill opacity=0.18, draw=none, forget plot]
    fill between[of=gpt2_p25_c and gpt2_p75_c];
\addplot[spectral, line width=1.1pt, mark=none, smooth]
    table[x=t_norm, y=median_phi]
    {content/figures/conductance_landscape/data/empirical_gpt2_halueval.dat};
\addlegendentry{GPT-2 (median, IQR)}
\addplot[name path=py_p25_c, draw=none, forget plot]
    table[x=t_norm, y=p25] {content/figures/conductance_landscape/data/empirical_pythia160m_halueval.dat};
\addplot[name path=py_p75_c, draw=none, forget plot]
    table[x=t_norm, y=p75] {content/figures/conductance_landscape/data/empirical_pythia160m_halueval.dat};
\addplot[fill=hallucinate, fill opacity=0.18, draw=none, forget plot]
    fill between[of=py_p25_c and py_p75_c];
\addplot[hallucinate, line width=1.1pt, mark=none, smooth]
    table[x=t_norm, y=median_phi]
    {content/figures/conductance_landscape/data/empirical_pythia160m_halueval.dat};
\addlegendentry{Pythia-160M}
\addplot[name path=ft5_p25_c, draw=none, forget plot]
    table[x=t_norm, y=p25] {content/figures/conductance_landscape/data/empirical_flant5_decoder_halueval.dat};
\addplot[name path=ft5_p75_c, draw=none, forget plot]
    table[x=t_norm, y=p75] {content/figures/conductance_landscape/data/empirical_flant5_decoder_halueval.dat};
\addplot[fill=OIorange, fill opacity=0.18, draw=none, forget plot]
    fill between[of=ft5_p25_c and ft5_p75_c];
\addplot[OIorange, line width=1.1pt, mark=none, smooth]
    table[x=t_norm, y=median_phi]
    {content/figures/conductance_landscape/data/empirical_flant5_decoder_halueval.dat};
\addlegendentry{Flan-T5 dec.}
\end{axis}
\end{tikzpicture}
\caption{Empirics: HaluEval, $\le 50$ samples per model.}
\label{fig:landscape-empirical}
\end{subfigure}

\caption{\textbf{Conductance landscape: theory and empirics.}
\textbf{(a)}~Theoretical $\phi(S_t)$ vs.\ cut location $t/n$ for canonical
causal architectures at $n=100$, derived in closed form from
\cref{prop:uc-cut-vol-identities,lem:window-conductance}.
Uniform causal (blue, solid) is U-shaped with an $n$-independent floor,
strictly above the $1/5$ Cheeger
floor (\cref{cor:uc-one-fifth}). Window attention (red and orange,
dashed) satisfies $\phi\le w/\min(t,n{-}t)$, hence balanced cuts scale
as $O(w/n)$ and pierce the floor when $w\ll n$ (shaded red).
Diagonal (green, dotted) has exact zero bipartite temporal-cut conductance.
Failure modes are \emph{shape-different}, not just value-different.
\textbf{(b)}~Empirical median temporal-cut conductance $\phi(S_t)$ (solid) with
interquartile-range bands ($25^\text{th}$--$75^\text{th}$ percentile across all
(sample, layer, head) tuples) for GPT-2, Pythia-160M, and Flan-T5 decoder on HaluEval
($\le 50$ samples per model after the $n\ge 32$ length filter; exact temporal-cut
conductance on the \emph{raw} bipartite graph $\Hh(\B)$ --- the object the $1/5$ floor
is proved for, no Fiedler approximation).  Black dashed: closed-form uniform-causal
prediction from panel~(a).  The architectures separate by where they sit relative to this
reference: GPT-2's band stays \emph{above} the $1/5$ floor (few cuts violate), Flan-T5
decoder dips well below it, and Pythia-160M is intermediate.  The population-level
signature is the fraction of heads falling below the $1/5$ floor: $\floorViolGPTHalu\%$, $\floorViolPyHalu\%$, $\floorViolFlanHalu\%$
on HaluEval for GPT-2, Pythia, Flan-T5 (\cref{tab:landscape-empirical-summary}).}
\label{fig:conductance-landscape}
\end{figure}
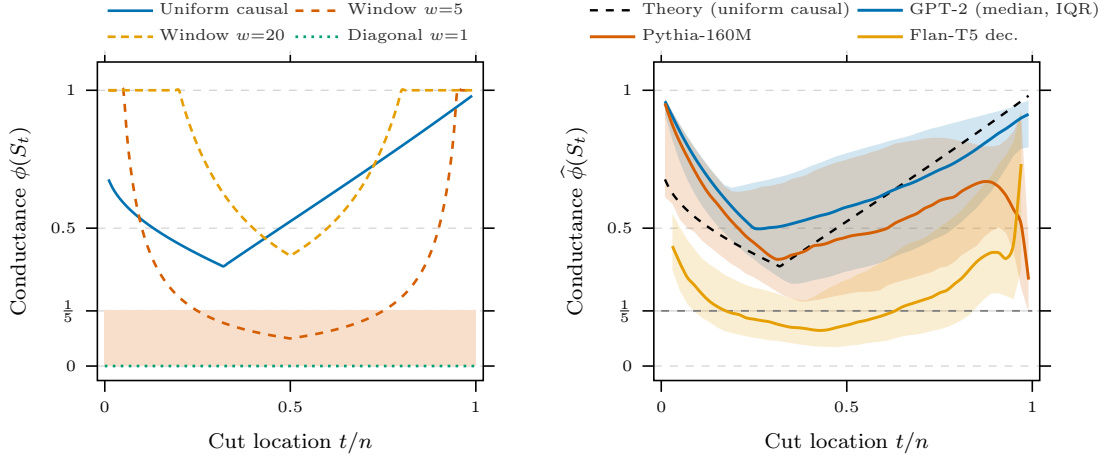


\begin{table}[!ht]
\centering
\small
\setlength{\tabcolsep}{6pt}
\caption{\textbf{Empirical temporal-cut conductance summary.}
Per-(model, data set) statistics over all (sample, layer, head) tuples
(target 50 samples; subject to the $n\ge 32$ length filter), computed by exact
enumeration of temporal cuts on the \emph{raw} bipartite graph $\Hh(\B)$ --- the
object the $1/5$ floor is proved for (\cref{cor:uc-one-fifth}), not the spectral-sweep
estimator $\phihat$: the location $t^\ast/n$ of the worst cut, the minimum temporal-cut
conductance $\min_t\phi(S_t)$, and the fraction of tuples whose minimum falls below the
$1/5$ floor. The floor is a theoretical benchmark for the idealized uniform-causal
pattern, not an empirical attractor; the violation fraction is the architectural
observable. The
floor-violation fraction carries a 95\% \emph{clustered} bootstrap CI
(resampling unit: sample; 5000 resamples), since (layer, head) tuples
within a sample are correlated and a binomial interval over all
$n_{\text{rows}}$ would be anti-conservative. The
theorem-backed uniform-causal prediction gives zero floor violations;
the optimizer location is reported as an empirical landscape statistic,
not as a theoretical prediction.
Empirical results are visualised in
\cref{fig:landscape-empirical}; Flan-T5 decoder/TruthfulQA
omitted because no samples passed the $n\ge 32$ filter (very short
T5 decoder responses on TruthfulQA prompts).}
\label{tab:landscape-empirical-summary}
\begin{filecontents*}{data/tables/landscape_summary.csv}
# Empirical conductance-landscape summary (paper-ready, LaTeX cells).
# Generated by scripts/compute_floor_violation_clustered.py
# tstar/minphi: mean +/- std over (sample,layer,head) tuples.
# viol: floor-violation %% with 95%% clustered bootstrap CI
#   (super/subscript notation, as in the main results table);
#   dagger marks cells with <25 samples (cluster bootstrap unreliable).
Model,Dataset,nrows,tstar,minphi,viol
GPT-2,HaluEval,7200,$0.36\pm0.20$,$0.45\pm0.23$,$18\%_{\scriptscriptstyle 18}^{\scriptscriptstyle 19}$
GPT-2,MedHallu,7200,$0.35\pm0.19$,$0.43\pm0.23$,$19\%_{\scriptscriptstyle 19}^{\scriptscriptstyle 20}$
GPT-2,TruthfulQA,7200,$0.37\pm0.20$,$0.45\pm0.24$,$19\%_{\scriptscriptstyle 18}^{\scriptscriptstyle 19}$
Pythia-160M,HaluEval,7200,$0.66\pm0.32$,$0.24\pm0.21$,$51\%_{\scriptscriptstyle 47}^{\scriptscriptstyle 54}$
Pythia-160M,MedHallu,7200,$0.65\pm0.29$,$0.17\pm0.12$,$63\%_{\scriptscriptstyle 62}^{\scriptscriptstyle 64}$
Pythia-160M,TruthfulQA,7200,$0.47\pm0.25$,$0.29\pm0.24$,$45\%_{\scriptscriptstyle 40}^{\scriptscriptstyle 50}$
Flan-T5 dec.,HaluEval,2400,$0.47\pm0.25$,$0.11\pm0.11$,$80\%_{\scriptscriptstyle 79}^{\scriptscriptstyle 82}$
Flan-T5 dec.,MedHallu,2400,$0.44\pm0.22$,$0.09\pm0.11$,$82\%_{\scriptscriptstyle 81}^{\scriptscriptstyle 84}$
\end{filecontents*}
\pgfplotstabletypeset[
  col sep=comma,
  string type,
  skip first n=6,
  columns={Model, Dataset, nrows, tstar, minphi, viol},
  every head row/.style={
    before row={\toprule
      & & & \textbf{$\langle t^\ast/n\rangle$} & \textbf{$\langle\min_t\phihat\rangle$} & \textbf{Viol. of $1/5$} \\
      & & & (empirical) & (theory: $\phi\ge 0.20$) & (95\% clustered CI) \\},
    after row=\midrule
  },
  every row no 2/.style={after row=\midrule},
  every row no 5/.style={after row=\midrule},
  every last row/.style={after row=\bottomrule},
  columns/Model/.style={column name=\textbf{Model}, column type=l},
  columns/Dataset/.style={column name=\textbf{Dataset}, column type=l},
  columns/nrows/.style={column name=\textbf{$n_{\text{rows}}$}, column type=r},
  columns/tstar/.style={column name={}, column type=c},
  columns/minphi/.style={column name={}, column type=c},
  columns/viol/.style={column name={}, column type=c},
]{data/tables/landscape_summary.csv}
\end{table}

Proofs can be found in \cref{app:proof-two-sided-diagnostic}; further details on regime control, masking compatibility, and the Cheeger bridge corollary can be found in \cref{app:proofs_qk_regimes}.

\medskip\noindent
Conductance and $\sigma_2$ both certify how well the operator $\M$
\emph{transports} between query and key sets, but neither distinguishes
$\M$ from $\M^\top$: they depend on $\M$ only through its singular values, which
are invariant under transpose.  This is a feature for capacity diagnosis and a
constraint for orientation diagnosis.  The next section makes the constraint precise:
every transpose-invariant spectral diagnostic is, by construction, blind to the
orientation of routing.  That limitation, in turn, identifies the residual
structure a directional diagnostic must access.

\section{Limits of Symmetric Spectral Diagnostics}
\label{sec:orientation-blindness}

Conductance answers ``how well does $\M$ transport?'' but not ``in which
direction.''  The underlying algebra is classical: singular values and the
symmetric spectral embeddings (the Hermitian dilation $\Hh(\M)$, the bipartite
Laplacian $I - \Hh(\M)$, the spectrum of $\Msym$) are all invariant under
$\M \mapsto \M^\top$, and that symmetrization destroys recoverable directed
structure is known for directed-graph clustering~\citep{Cucuringu2020Hermitian}.
Our contribution is not a new matrix-analysis fact but the identification of what
this invariance costs an attention diagnostic, and the complement that recovers
what is lost.

\paragraph{Symmetric and antisymmetric components.}
For self-attention ($n_q = n_k = n$), any attention matrix admits a canonical decomposition:
\begin{equation}
\M = \underbrace{\tfrac12(\M+\M^\top)}_{\Msym} + \underbrace{\tfrac12(\M-\M^\top)}_{\Masym}.
\label{eq:symm-antisymm-decomp}
\end{equation}
Under the Hilbert--Schmidt inner product
$\langle A, B \rangle_F := \trace(A^\top B)$~\cite{hornMatrixAnalysis2012},
this decomposition is an \emph{orthogonal projection}:
$\langle S, A \rangle_F = \trace(SA) = -\trace(SA) = 0$ for any symmetric $S$ and antisymmetric $A$.
Geometrically, $\Msym$ is the closest symmetric matrix to $\M$,
and $\Masym$ captures the antisymmetric residual, the difference $\M-\M^\top$ that distinguishes the operator from its transpose.
The decomposition has an established empirical footprint in attention:
\citet{Saponati2025SelfAttentionStructures} show that autoregressive training drives
attention toward directional, skew-dominated structure while bidirectional training
promotes symmetry, and score heads by their asymmetry.  Their result concerns what
training \emph{produces}; ours concerns what diagnostics can \emph{recover}: the
theorem below bounds every transpose-invariant functional regardless of how much
antisymmetric structure training induces.

\paragraph{A two-token example.}
The limitation is concrete already at $n=2$.  Take the degree-normalized
operator and its transpose
\[
\M=\begin{pmatrix}0&0\\[2pt]1&0\end{pmatrix},
\qquad
\M^\top=\begin{pmatrix}0&1\\[2pt]0&0\end{pmatrix},
\]
a forward read (query $1$ attends to key $0$) and its reverse.  Both share the
symmetric part
$\Msym=\bigl(\begin{smallmatrix}0&\nicefrac12\\[1pt]\nicefrac12&0\end{smallmatrix}\bigr)$
and differ only in the \emph{sign} of
$\Masym=\pm\bigl(\begin{smallmatrix}0&-\nicefrac12\\[1pt]\nicefrac12&0\end{smallmatrix}\bigr)$.
Their singular values are identical, $\{1,0\}$, so every singular-value
functional scores them alike; the asymmetry coefficient
$G(\M)=\|\Masym\|_F/(\|\M\|_F+\varepsilon)\approx\nicefrac{1}{\sqrt2}$ likewise
reads the same for both.  No transpose-invariant quantity separates the forward
operator from the backward one; only the sign of $\Masym$ does.  (The skew
matrix $\bigl(\begin{smallmatrix}0&1\\[1pt]-1&0\end{smallmatrix}\bigr)$
occasionally offered to argue that singular values ``come from asymmetry'' is
not a valid attention operator (its entries are not non-negative), but the
point stands without it: a non-symmetric non-negative $\M$, as above, already
exhibits the blindness.)

\paragraph{Why conductance cannot see direction.}
Conductance and $\sigma_2$ depend on $\M$ only through its singular values
(equivalently the spectrum of the Hermitian dilation $\Hh(\M)$), which are invariant
under $\M\mapsto\M^\top$.  (They are \emph{not} functions of $\Msym$ alone: singular
values also depend on antisymmetric energy, as the two-token example shows.)  The
following theorem formalizes the resulting constraint, combining a transpose-invariance
result with its constructive complement.

\begin{theorem}[Orientation blindness of transpose-invariant diagnostics]
\label{thm:orientation-blindness}
Let $\M \in \reals^{n \times n}$ be a square attention transport
operator with symmetric--antisymmetric decomposition
$\M = \Msym + \Masym$
(\cref{sec:transport-operators}).
\begin{enumerate}[label=(\alph*)]
\item \textbf{Orientation blindness.}
Suppose $F$ is either
(i) a function of the singular values of $\M$, or
(ii) a function of the spectrum of $\Msym$.
Then in both cases $F(\M) = F(\M^\top)$, i.e., $F$ is transpose-invariant.
Moreover, under hypothesis (ii), $F$ cannot distinguish $\M$ from
any operator sharing the same symmetric part.
\item \textbf{Projection characterization.}
$\Msym$ is the unique Frobenius-nearest symmetric matrix to $\M$, and the
residual norm equals $\|\Masym\|_F$. The \emph{matrices} $\Msym,\Masym$ thus
orthogonally decompose $\M$; a transpose-visible summary of $\Msym$ together with
the scalar $G(\M)$ reports the transpose-visible component and the normalized
magnitude of the orthogonal residual (these two scalars summarize the two
components; they do not by themselves reconstruct $\M$).
\end{enumerate}
\end{theorem}

\begin{proof}
\textbf{Part (a).}
A concrete certificate of transpose-invariance is the Hermitian
dilation $\Hh(\M) = \bigl(\begin{smallmatrix}0&\M\\ \M^\top&0\end{smallmatrix}\bigr)$
\emph{(bridge)}.\footnote{Italic tags mark proof-step provenance throughout:
\emph{(foundations)} for classical facts used as-is, \emph{(bridge)} for known
results specialized to the attention setting, and \emph{(contribution)} for
steps original to this paper.}
Let $P$ be the block-swap permutation.  Direct block multiplication gives
$P^\top \Hh(\M)\,P = \Hh(\M^\top)$ \emph{(foundations)},
so $\Hh(\M)$ and $\Hh(\M^\top)$ are similar and share the same spectrum.
Since the singular values of $\M$ equal the absolute eigenvalues of $\Hh(\M)$,
any functional of singular values is transpose-invariant.
The bipartite Laplacian $\mathcal{L}_{\mathrm{bip}} = I - \Hh(\M)$
inherits the same invariance (Corollary~\ref{cor:transpose-invariant-diagnostics}).

\textbf{Part (b).}
By orthogonality of $\mathcal{S}$ (symmetric matrices) and $\mathcal{A}$
(antisymmetric matrices) under the Hilbert--Schmidt inner product
\emph{(foundations)},
$\Msym = \pi_{\mathcal{S}}(\M)$ is the orthogonal projection onto $\mathcal{S}$,
with uniqueness from strict convexity of $\|\cdot\|_F$ and residual norm
$\|\M-\Msym\|_F = \|\Masym\|_F$; the nearest-symmetric-matrix property is
classical \citep{FanHoffman1955}.  A transpose-visible summary of $\Msym$ together with $G(\M)$ therefore reports the
two orthogonal components of $\M$ \emph{(contribution)}.

The blindness of this decomposition to routing \emph{orientation} is then
immediate and, crucially, \emph{definitional} rather than a consequence of
part~(a): transpose negates the antisymmetric part,
$(\M^\top)_{\mathrm{asym}} = -\Masym$, while fixing $\Msym$, so any functional
with $F(\M)=F(\M^\top)$ assigns the same value to $\M$ and to its
orientation-reversed twin and cannot resolve the sign of $\Masym$.  Part~(a)
supplies the members of this transpose-invariant family that matter in practice
(singular-value and $\Msym$-spectrum functionals); the blindness itself needs
only $F(\M)=F(\M^\top)$.  The distinction is essential: such functionals are
blind to the \emph{orientation} of $\Masym$, not to its magnitude
(\cref{rem:polarity-not-magnitude}).
\end{proof}

\begin{remark}[Orientation-blindness vs.\ asymmetry-blindness]
\label{rem:polarity-not-magnitude}
Part~(a) says a symmetric spectral functional cannot distinguish $\M$ from
$\M^\top$: it is blind to the \emph{orientation} of the antisymmetric residual,
not to its existence or magnitude.  Singular values, for instance, do depend on
antisymmetric energy, a purely skew operator has $\Msym = 0$ yet nonzero
singular values.  The asymmetry coefficient $G$ is itself transpose-invariant,
$G(\M) = G(\M^\top)$: it quantifies \emph{how much} asymmetry is present while
remaining blind to its sign.  What no transpose-invariant functional can
recover is which of $\M,\M^\top$ is the forward operator, the routing
polarity.
\end{remark}

The blindness is not merely that the diagnostics we name happen to be invariant:
orientation is \emph{exactly} the information the entire transpose-invariant class
discards, and nothing more.

\begin{corollary}[Transpose rigidity: orientation is the complete lost quotient]
\label{cor:transpose-rigidity}
The transpose-invariant functionals jointly separate operators up to transpose: if
$F(\M) = F(\M')$ for \emph{every} functional $F$ satisfying $F(X) = F(X^\top)$ for all $X$,
then $\M' \in \{\M, \M^\top\}$.  Hence an operator and its transpose are the only distinct
pair no transpose-invariant functional resolves, and they differ exactly in the sign of
$\Masym$: within the full algebra of transpose-invariant functionals, the routing
orientation is the complete and sole unrecoverable quotient.  Any fixed spectral subfamily
separates strictly less (\cref{cor:transpose-invariant-diagnostics}).
\end{corollary}

\begin{proof}
The proof is the standard separating-invariant construction, applied to the
transpose action.  The functional
$\mathrm{sep}_{\M}(X) := \|X - \M\|_F^2\,\|X - \M^\top\|_F^2$ is
transpose-invariant in its argument \emph{(foundations)}:
$\|X^\top - \M\|_F = \|(X - \M^\top)^\top\|_F = \|X - \M^\top\|_F$ and
$\|X^\top - \M^\top\|_F = \|X - \M\|_F$, so $\mathrm{sep}_{\M}(X^\top) = \mathrm{sep}_{\M}(X)$;
and it vanishes exactly on $\{\M, \M^\top\}$, a product of squared Frobenius norms being zero
iff a factor is.  Applying the hypothesis to $F = \mathrm{sep}_{\M}$ gives
$\mathrm{sep}_{\M}(\M') = \mathrm{sep}_{\M}(\M) = 0$, hence $\M' \in \{\M, \M^\top\}$
\emph{(contribution)}.
\end{proof}

\begin{remark}[The magnitude $G$ is a lossy summary, not a complete invariant]
\label{rem:g-lossy}
\Cref{cor:transpose-rigidity} also delimits the scope of $G$: the pair $(\Msym, G(\M))$ is
\emph{not} a complete transpose-invariant.  A transpose-invariant functional resolves the
full symmetric matrix $\Masym\Masym^\top$ (for instance its spectrum, since
$(\M^\top)_{\mathrm{asym}}(\M^\top)_{\mathrm{asym}}^\top = \Masym\Masym^\top$), not only its
norm $\|\Masym\|_F$.  The complete recoverable content is $\M$ up to transpose; $G$ is the
natural normalized \emph{scalar} summary of the recoverable antisymmetric magnitude, while
orientation is precisely what it---and every transpose-invariant functional---discards.
\end{remark}

\begin{corollary}[Blindness of common spectral attention diagnostics]
\label{cor:transpose-invariant-diagnostics}
Any diagnostic computed from the eigenvalues or singular values of $\Msym$,
$\Hh(\M)$, or the bipartite Laplacian
$\mathcal{L}_{\mathrm{bip}} = I - \Hh(\M)$ (together with any eigenvector
summary invariant under the block swap $\M \mapsto \M^\top$) is invariant
under $\M \mapsto \M^\top$.  This invariant class includes spectral entropy,
the spectral gap ($1 - \sigma_2$), graph Laplacian eigenvalues, and the
eigenvalue- and singular-value-based features used in LLM-Check, EigenTrack,
and LapEigvals.
\end{corollary}

\begin{proof}
The SVD singular values equal the absolute eigenvalues of $\Hh(\M)$,
which are transpose-invariant by Theorem~\ref{thm:orientation-blindness}(a).
The bipartite Laplacian eigenvalues are $1 \pm \sigma_i$
(\cref{lem:dilation-spectrum}), hence also transpose-invariant.
Eigenvectors themselves are \emph{not} invariant: the block swap sends an
eigenvector $(x,y)$ of $\Hh(\M)$ to $(y,x)$, so an eigenvector statistic that
reads the query and key blocks separately (using token roles or temporal
coordinates) can break the invariance.  The corollary therefore covers
eigenvalue- and singular-value-based diagnostics, and only those eigenvector
summaries that are themselves block-swap symmetric.
\end{proof}

\Cref{tab:method-scope} makes the scope concrete: it pairs each prior diagnostic
with the exact spectral object it extracts and the reason it is
transpose-invariant.  The partition is deliberate: methods that diagnose the
attention operator ($\phihat$, $\sigma_2$, LapEigvals, LLM-Check's attention
branch) are orientation-blind in the strong sense of \cref{thm:orientation-blindness},
whereas hidden-state methods (LLM-Check's hidden branch, EigenTrack) are
transpose-invariant only because a covariance is symmetric, and diagnose the
residual stream rather than attention direction; the theorem bounds the former,
not the latter.
\begin{table}[t]
\centering
\footnotesize
\setlength{\tabcolsep}{4pt}
\renewcommand{\arraystretch}{1.2}
\caption{\textbf{Scope of the orientation-blindness theorem across prior
  spectral diagnostics.}  Methods that diagnose the \emph{attention operator}
  (top block) extract functionals of the symmetric/Laplacian spectrum and are
  orientation-blind by \cref{thm:orientation-blindness}; methods that diagnose a
  \emph{hidden-state covariance} (bottom block) are transpose-invariant on a
  symmetric object but read the residual stream rather than the attention
  operator, so the theorem constrains attention-direction recovery, not these
  methods directly (\cref{sec:discussion}).}
\label{tab:method-scope}
\begin{tabular}{@{}p{1.9cm}p{3.2cm}p{3.7cm}p{3.5cm}@{}}
\toprule
\textbf{Method} & \textbf{Spectral object} & \textbf{Why transpose-invariant}
& \textbf{Relation to \cref{thm:orientation-blindness}} \\
\midrule
\multicolumn{4}{@{}l}{\emph{Attention-operator diagnostics (in scope)}} \\
\addlinespace[2pt]
This work ($\phihat$, $\sigma_2$)
& sweep conductance / second singular value of $\M$
& functionals of singular values of $\M$ (equivalently $\Msym$ spectrum)
& orientation-blind; complemented by $G$ \\
LapEigvals
& eigenvalues of the attention-graph Laplacian
& Laplacian of the symmetrized graph; an $\Msym$-spectrum functional
& orientation-blind \\
LLM-Check (attn.)
& diagonal / spectral statistics of the attention map
& diagonal is fixed by transpose; spectral statistics are singular-value functionals
& orientation-blind \\
\midrule
\multicolumn{4}{@{}l}{\emph{Hidden-state diagnostics (transpose-invariant, but outside the attention-operator scope)}} \\
\addlinespace[2pt]
LLM-Check (hidden)
& eigenvalues of the hidden-state covariance $Z^\top J Z$
& a covariance is symmetric ($=$ its own transpose)
& reads the residual stream, not $\M$; theorem adds no constraint \\
EigenTrack
& covariance-spectrum statistics of sliding-window activations
& hidden-activation covariance is symmetric
& reads activations, not $\M$; theorem adds no constraint \\
\bottomrule
\end{tabular}
\end{table}

\begin{proposition}[Converse: asymmetry bounds transpose discrepancy]
\label{prop:ob-converse}
For any $L$-Lipschitz functional $F$ on
$(\mathbb{R}^{n \times n}, \|\cdot\|_F)$,
\begin{equation}
|F(\M) - F(\M^\top)| \;\leq\; 2L\,\|\Masym\|_F.
\label{eq:ob-converse}
\end{equation}
In particular, if $F(\M) \neq F(\M^\top)$, then
$\|\Masym\|_F \geq |F(\M) - F(\M^\top)| / (2L)$.
\end{proposition}

\begin{proof}
Since $\M^\top = \Msym - \Masym$, we have
$\M - \M^\top = 2\Masym$, so $\|\M - \M^\top\|_F = 2\|\Masym\|_F$.
The Lipschitz bound gives
$|F(\M) - F(\M^\top)| \leq L\,\|\M - \M^\top\|_F = 2L\,\|\Masym\|_F$.
\end{proof}

The bound is not merely an upper estimate: it is attained, so $\|\Masym\|_F$ is the
\emph{exact} orientation-sensitivity capacity of the normalized Lipschitz class.

\begin{proposition}[Tightness: exact transpose-discrepancy capacity]
\label{prop:ob-converse-tight}
The bound \eqref{eq:ob-converse} is achieved.  For the linear template-alignment functional
$F_T(X) = \langle X, T\rangle_F / \|T\|_F$, which is $1$-Lipschitz,
\begin{equation}
\sup_{\substack{F\ 1\text{-Lipschitz}}} \bigl|F(\M) - F(\M^\top)\bigr|
= \max_{T \neq 0}\, \bigl(F_T(\M) - F_T(\M^\top)\bigr)
= \|\M - \M^\top\|_F = 2\,\|\Masym\|_F,
\label{eq:ob-tight}
\end{equation}
attained at the data-driven template $T = \M - \M^\top$ (for symmetric $\M$ the template
vanishes and both sides of \eqref{eq:ob-tight} are zero).  Equivalently
$\|\Masym\|_F = \tfrac12 \sup\{|F(\M) - F(\M^\top)| : F\ 1\text{-Lipschitz}\}$.
\end{proposition}

\begin{proof}
$F_T(\M) - F_T(\M^\top) = \langle \M - \M^\top,\, T\rangle_F / \|T\|_F$
\emph{(foundations)}.  Cauchy--Schwarz bounds this by $\|\M - \M^\top\|_F$, with equality at
$T = \M - \M^\top$, where it equals
$\langle \M - \M^\top, \M - \M^\top\rangle_F / \|\M - \M^\top\|_F = \|\M - \M^\top\|_F$
\emph{(contribution)}.  The upper bound over all $1$-Lipschitz $F$ is
\cref{prop:ob-converse} with $L = 1$; the linear $F_T$ attains it.  The
attainment step is the standard duality argument for Lipschitz functionals,
specialized to the pair $(\M, \M^\top)$.
\end{proof}

\noindent
The maximizer is exactly a \emph{signed} alignment of $\M - \M^\top$ with a template, the
construction the orientation-recovering statistic of \suppref{app:signed-directionality} is
built from: orientation sensitivity is impossible below, and fully realized at, the
$G$-axis magnitude.

\noindent
This is a \emph{transpose-stability bound}: any diagnostic that distinguishes $\M$ from
$\M^\top$ has transpose discrepancy at most
$2L\,\|\Masym\|_F = 2L\,G(\M)(\|\M\|_F + \varepsilon)$.  Because the bound scales with both
the Lipschitz constant $L$ and $\|\M\|_F$, it has comparative content only within a
normalized class of diagnostics (fixed $L$, comparable $\|\M\|_F$); within such a class,
small $G$ across a model's heads implies every Lipschitz diagnostic is approximately
orientation-blind.  The bound limits transpose \emph{sensitivity}; it does not establish
that $G$ \emph{suffices} for direction-sensitive diagnosis, only that direction-sensitivity
is impossible without antisymmetric energy.

\paragraph{The asymmetry-magnitude complement.}
The orthogonal decomposition motivates a natural complement to conductance.
$G$ is the normalized \emph{magnitude} of the antisymmetric residual that symmetric
summaries discard: a structural certificate of how much asymmetry is present, not of
its orientation (recovering orientation requires a genuinely signed statistic;
\suppref{app:signed-directionality}).  The symmetric and antisymmetric \emph{matrices}
$\Msym,\Masym$ together span the Frobenius-orthogonal decomposition of $\M$:
\begin{equation}
G(\M) \;=\; \frac{\|\Masym\|_F}{\|\M\|_F + \varepsilon}
\;=\; \frac{d(\M, \mathcal{S})}{\|\M\|_F + \varepsilon},
\label{eq:G_def}
\end{equation}
where $d(\M, \mathcal{S})$ is the Frobenius distance to the symmetric subspace.
Its extremes pin down the structural meaning of the antisymmetric axis.

\begin{proposition}[Range and extremes of $G$]
\label{prop:g-extremes}
For any $\M$, $G(\M, 0) \in [0, 1]$, with $G = 0$ iff $\M$ is symmetric and
$G = 1$ iff $\M$ is purely skew ($\M^\top = -\M$).  A non-negative attention
matrix cannot be purely skew, so $G < 1$; within the causal class the
antisymmetric-maximal pattern is strict lower-triangular (zero diagonal), where
$G = 1/\sqrt{2}$ exactly, and every causal Toeplitz operator obeys
$G \le 1/\sqrt{2}$.
\end{proposition}

\noindent The symmetry claim is immediate from \eqref{eq:G_def}
($G = 0 \iff \Masym = 0 \iff \M = \M^\top$); the causal $1/\sqrt{2}$ ceiling
follows from the Toeplitz structure of $\Masym$.  Full proofs are in
\cref{app:g-formal-verification}.

\begin{remark}[Optimality of $G$]
\label{rem:g-optimality}
Among diagnostics of the form $D(\M) = \|\M - S\|_F / \|\M\|_F$
with $S \in \mathcal{S}$, the minimum is attained uniquely at
$S = \Msym$ and equals $G(\M, 0)$.  $G$ therefore measures the
minimal normalized Frobenius perturbation required to reach the
transpose-invariant class.
\end{remark}

A geometric illustration of how $G$ captures orientation failures
that $\phihat$ misses appears in the Online Supplement.

\medskip\noindent
This closes the structural analysis of the symmetric--antisymmetric boundary: $\phi$ and $\sigma_2$ are transpose-invariant capacity summaries, $G$ is the normalized magnitude of the antisymmetric residual, and the transpose-stability bound limits what any Lipschitz diagnostic can achieve beyond this decomposition.  The next section develops $G$'s properties under causality and positional encoding.

\section{Transport Diagnostics}
\label{sec:asymmetric-guessing}

The orientation-blindness theorem in \cref{sec:orientation-blindness}
identifies $G = \|\Masym\|_F/(\|\M\|_F+\varepsilon)$ as the residual
that any directional diagnostic must access.  Two questions follow.
First, when is $G > 0$ guaranteed?  Causal masking turns out to force
$G > 0$ for any non-trivial attention pattern, so the question becomes
quantitative: how large is $G$ as a function of the architecture's
positional structure?  Second, given conductance, $\sigma_2$, and $G$,
what diagnostic system do they jointly form?  This section answers
both, then closes by also defining the empirical conductance estimator
$\phihat$ used in the experiments.

\paragraph{Symmetry and causal structure.}
The vanishing locus $G = 0$ is exactly the symmetric subspace $\M = \M^\top$: the
operator equals its transpose under the query--key identification.  (This is distinct
from reversibility of the random walk on the dilation $\Hh(\M)$, which is symmetric for
\emph{every} $\M$ and hence always reversible.)  Causal attention falls outside this
locus by construction: lower-triangular support cannot be symmetric (except trivially),
so a causal head that attends beyond the diagonal must have $G > 0$.

\paragraph{Geometric constraint under causality.}
For causal self-attention, the interpretation of $G = 0$ is determined by a geometric constraint.
Any lower-triangular $\M$ with off-diagonal mass has $(\Masym)_{ij} = \M_{ij}/2 \neq 0$ for $i > j$
(since $\M_{ji} = 0$ by the causal mask).
Therefore, any causal head that attends beyond the diagonal necessarily has $G > 0$.

\begin{proposition}[Causal attention implies $G > 0$]
\label{prop:asymmetry-energy}
For causal self-attention, the intersection of symmetric and lower-triangular is diagonal.
Therefore $G = 0$ implies temporal isolation, where each token attends only to itself.
\end{proposition}

An illustration of the two failure modes (bottleneck and diffuse) and
how they coexist in causally masked self-attention is provided in the
Online Supplement.

\paragraph{Quantitative bounds from positional structure.}
Decoder models exhibit $G$ well above the qualitative $G > 0$ floor: Pythia, for instance, reaches $G \approx 0.5\text{--}0.65$.  Two
\emph{distinct} structural mechanisms can raise $G$, and they yield bounds of
different strength.  The first is a \emph{monotone primacy profile}: rows
whose weight decreases from the earliest key toward the diagonal.  The second
is a \emph{first-token mass floor}: every query retaining a fixed fraction of
attention on the earliest key at all sequence lengths, the attention-sink
pattern documented across decoder language models~\citep{Xiao2024AttentionSink}.
Softmax preserves monotone order in the pre-softmax logits, and the causal
Toeplitz--Frobenius identity converts either profile into an antisymmetric-energy
statement.  The two mechanisms give the bounds below; \emph{they are independent:
the second is not a consequence of the first}.  The supporting softmax-order and
Toeplitz--Frobenius lemmas, and the proofs of these bounds, are in the Online
Supplement, \suppref{suppl:structural-regimes-extra}:

\begin{proposition}[Monotone primacy decay gives an $O(1/\sqrt{n})$ $G$ bound]
\label{prop:monotone-g-bound}
For a causal, row-stochastic matrix with monotone decreasing rows (weight
largest at the earliest key) and $n \ge 2$:
$G(\M, \varepsilon) \ge 1/(4\sqrt{n} + 4\varepsilon)$.
\end{proposition}

\begin{proposition}[First-column mass floor gives an $n$-independent $G$ bound]
\label{prop:firstcol-g-bound}
If a causal, row-stochastic matrix satisfies a first-column mass floor
$\M_{i,0} \ge c > 0$ for every row $i \ge 1$, then
$G(\M, 0) \ge c/2$, independent of sequence length.
\end{proposition}

\begin{corollary}[Exponential decay: $n$-independent $G$ bound]
\label{cor:exponential-g-bound}
For exponential decay $f(k) = Ce^{-\alpha k}$ with $\alpha > 0$, every row
satisfies the first-column floor $\M_{i,0} \ge 1 - e^{-\alpha}$.  By
\cref{prop:firstcol-g-bound}, $G(\M, 0) \ge (1 - e^{-\alpha})/2$, independent
of sequence length.
\end{corollary}

\noindent
The two bounds describe genuinely different regimes.  A monotone primacy
profile alone guarantees only $G > 0$ with an $O(1/\sqrt{n})$ rate, because
the lower bound on antisymmetric energy rests on a single matrix entry.  A
first-column mass floor (the attention-sink structure) instead makes the
antisymmetric energy grow as $\Theta(n)$, cancelling the $\Theta(n)$ growth of
$\|\M\|_F^2$ and leaving a constant.  It is the sink mechanism, not recency
bias, that underwrites a sequence-length-independent floor on $G$; recency
decay (as induced by RoPE~\cite{Su2024RoFormer}) raises attention toward the
diagonal and is a separate effect.  This dependence on architectural
structure (conductance is universal while a tight $G$ bound requires a
first-token floor) is the cost of the tighter bound; we return to this
tradeoff in \cref{sec:discussion}.

\paragraph{Architecture-dependent interpretation.}
The geometric constraint explains why $G$ has different diagnostic value across architectures:
\begin{itemize}[nosep,leftmargin=*]
\item \textbf{Decoder self-attention} (causal):
$G = 0$ implies symmetric, and symmetric $\cap$ causal $=$ diagonal.
Therefore $G = 0$ indicates temporal isolation, a \emph{failure mode}.

\item \textbf{Encoder self-attention} (bidirectional):
$G = 0$ requires only $\M = \M^\top$, achievable with arbitrary off-diagonal structure.
This is \emph{normal operation}, not pathological.
\end{itemize}
For encoder-decoder models, $G$ should be interpreted as a temporal isolation diagnostic
\emph{only for decoder self-attention}.

\paragraph{Empirical conductance estimator.}
Exact conductance minimization is NP-hard~\citep{SimaSchaeffer2006NPConductance}, so we use a spectral sweep on the second singular
vectors of $\M$ (equivalently, the second eigenvector of $\Hh(\M)$) to obtain an
empirical estimate $\phihat$. The sweep evaluates $O(n_q + n_k)$ candidate threshold
cuts; its approximation ratio is unbounded in the worst case but empirically within
$2{\times}$ on matched nulls.
In practice, most heads lie within their intrinsic
spectral band; rather than relying on literal ``violations'', we use $\phihat$
as a continuous structural score whose interpretation is anchored by the classical
Cheeger inequality~\citep{cheegerLowerBoundSmallest2015,Chung1997}.
For causally masked matrices, the sweep operates on the lower-triangular support;
the bipartite Hermitian dilation $\Hh(\M)$ is always square ($n_q + n_k$),
so the procedure applies unchanged to rectangular cross-attention
(\cref{lem:dilation-spectrum}).

\paragraph{Diagnostic synthesis.}
Together, $\phihat$, $\sigma_2$, and $G$ form a structurally grounded diagnostic system: $\phihat$ provides a theory-grounded capacity certificate via the Cheeger inequality; $\sigma_2$ captures coupling structure beyond degree heterogeneity (Theorem~\ref{prop:degree-sufficiency}); $G$ captures the antisymmetric residual that no symmetric method can access (Theorem~\ref{thm:orientation-blindness}).
Alternative matrix norms
($G_{\mathrm{op}}$, $G_{*}$, KL-divergence rate of the induced random walk) partially strengthen the asymmetry
signal in encoder-decoder architectures but do not achieve cross-architecture consistency
(Online Supplement, \suppref{app:alternative-g-norms}).  The diagnostic scope of $G$ (where it discriminates
and where it does not) is assessed empirically in \cref{sec:evaluation}.

\medskip\noindent
The three diagnostics are defined.  Before reporting results, we establish the evaluation protocol required to interpret them without length confounding.


\section{Evaluation Protocol}
\label{sec:length-confound}

Before the transport framework can be tested empirically, a methodological threat must be addressed: spectral features inherit length dependence that can inflate apparent discrimination.  We characterize three distinct confounding channels and establish a length-controlled evaluation protocol that governs all subsequent empirical claims.

When hallucination rate correlates with response length (HaluEval: $r{=}{+}0.70$, $p{<}10^{-30}$; MedHallu: $r{=}{-}0.21$, $p{=}0.037$; TruthfulQA: $r{=}{-}0.13$, $p{=}0.057$), any feature that implicitly encodes length will appear discriminative even if it carries no semantic signal.
We organize length confounding into three channels, using LLM-Check's~\citep{Sriramanan2024LLMCheck} feature families (hidden-state, attention, output) as the organizing scaffold.

\subsection{Three Channels of Length Exposure}

\paragraph{Hidden-state channel.}
Methods that analyze covariance spectra of hidden representations inherit length
dependence through centering and normalization: LLM-Check's centering matrix
produces $\log(n-1)$ scaling in its mean log-eigenvalue score, and
EigenTrack's~\citep{Ettori2025EigenTrack} Marchenko--Pastur reference law depends
on the aspect ratio $\gamma = D/n$, whose edges shift with $n$.

\paragraph{Attention channel.}
Methods that read statistics off attention matrices suffer position--length
coupling: LLM-Check's diagonal feature has $\mathbb{E}[A_{ii}] = 1/i$ under causal
masking ($O(\log n / n)$ dependence) and discards off-diagonal routing structure,
and LapEigvals~\citep{Binkowski2025LapEigvals} couples position to length through
its degree matrix.  In cross-attention ($n_q \neq n_k$), growing query
count with keys fixed scales column degrees, adding explicit length dependence
(Proposition~\ref{prop:length-entanglement}, \cref{app:proofs_qk_regimes}).

\paragraph{Output channel.}
Token-level entropy and perplexity are relatively robust (means converge with
$O(1/n)$ variance) though entropy extrema carry $O(\log n)$ order-statistic
effects.  This channel has the mildest exposure but still warrants control.
Per-method length-dependence derivations for LLM-Check, EigenTrack, and
LapEigvals are in the Online Supplement, \suppref{app:llmcheck-spectral} and
\suppref{app:lapeigvals}.

\subsection{Length-Controlled AUROC}
\label{sec:eval-protocol}

The three-channel taxonomy explains why different spectral methods behave inconsistently across benchmarks: without explicit length control, structural exposure remains and cross-benchmark generalization suffers.
Table~\ref{tab:auroc-variants-main} summarizes the evaluation metrics used throughout; formal definitions appear in the Online Supplement, \suppref{app:evaluation-metrics}.

\begin{table}[h]
\centering
\small
\begin{tabular}{@{}lp{2.8cm}cc@{}}
\toprule
\textbf{Metric} & \textbf{Purpose} & \textbf{Conf.} & \textbf{Assump.} \\
\midrule
Raw & Baseline discrim. & None & -- \\
Resid. & Remove linear len. & Linear & Linear \\
Strat. & Within-bin discrim. & Nonpar. & Bins \\
LC & Strat. + within-bin resid. & Both & Bins + linear \\
Overlap & Balanced-bin & Nonpar. & Prev. \\
\bottomrule
\end{tabular}
\caption{\textbf{AUROC variants and their properties.}
Conf.: confound addressed. Assump.: modeling assumptions.}
\label{tab:auroc-variants-main}
\end{table}

\paragraph{Length-controlled AUROC (LC-AUROC).}
To control for length confounding at two levels, we partition samples by \emph{response length}
(model-specific tokenizer, excluding prompt) into $B$ equal-frequency bins, where $B$ is chosen
adaptively per data set as the smallest value satisfying: (i)~all bins contain $\geq 25$
positive-negative pairs, (ii)~no bin is degenerate, and (iii)~the maximum within-bin Spearman
rank correlation between score and length satisfies $|\rho_b| < 0.10$.
Within each bin $b$, scores are residualized via OLS against length ($r_i = s_i - \hat{s}_i$)
to remove residual within-bin linear correlation, and per-bin AUROC values are aggregated with
pair-weighting:\footnote{Despite the two-stage structure (stratification then residualization),
LC-AUROC is not a semiparametric doubly-robust estimator in the sense of
\citet{RobinsRotnitzkyZhao1994}: it does not guarantee consistency under misspecification of
either the propensity or outcome model.}
\begin{align}
\text{AUROC}_{\text{LC}} &= \sum_b w_b \cdot \text{AUROC}_b(y, r_b), \label{eq:lc-auroc} \\
w_b &= \frac{n_{+,b} \times n_{-,b}}{\sum_{b'} n_{+,b'} \times n_{-,b'}}. \nonumber
\end{align}
Pair-weighting ensures bins with class imbalance have appropriately reduced influence.
Because a transport feature can discriminate with either polarity (high or
low values flagging hallucination), we report \emph{flipped} LC-AUROC,
$\max(\text{AUROC}(s), 1{-}\text{AUROC}(s))$.  Flipping selects the polarity
using the labels, so a flipped value answers ``does this feature carry
signal'' rather than ``does a label-free detector exist'': the features
themselves are computed without labels, but a deployable detector
additionally needs a small calibration set to fix the sign and threshold
(\cref{sec:discussion}).  Unless stated otherwise, reported LC-AUROC values
are flipped, and should be read as evidence of signal strength rather than as
the accuracy of a deployed label-free classifier.
Feature rankings under LC-AUROC are robust to hyperparameter choices: across a grid
of bin counts and correlation thresholds, the mean Kendall $\tau$ between
feature rankings exceeds 0.75 for the majority of configurations
(see Online Supplement, \suppref{app:lc-auroc-sensitivity}).

A representative per-quartile breakdown for GPT-2 on HaluEval, including
bootstrap CIs and the raw vs.\ residualized vs.\ LC-AUROC comparison
that motivates LC-AUROC, is shown in the Online Supplement (\suppref{app:stratified-auroc}).
On HaluEval the length--prevalence coupling is extreme: short responses are almost never
hallucinated and long ones almost always are (quartile prevalence ranges from a few percent
to $\sim$99\%).  Length control therefore rests on limited common support (concentrated
in the middle of the length range, where both classes co-occur), so the HaluEval
length-controlled numbers should be read as the most confounded of the three benchmarks;
\suppref{app:stratified-auroc} reports the per-quartile structure and the
overlap-region restriction that bound this.

\paragraph{Datasets and models.}
We evaluate on three benchmarks chosen to span distinct length-label
correlation regimes: HaluEval~\citep{Li2023HaluEval} ($n{=}10{,}000$;
$r_\ell{=}{+}0.70$, strong positive), TruthfulQA~\citep{Lin2022TruthfulQA}
($n{=}817$; $r_\ell{=}{-}0.13$, near-zero), and
MedHallu~\citep{Pandit2025MedHallu} ($n{=}1{,}000$, the expert-labelled
subset of the MedHallu benchmark; $r_\ell{=}{-}0.21$, weak negative).
Spectral diagnostics are extracted by \emph{forced scoring}: each prompt is concatenated with the benchmark-provided reference response (factual or hallucinated) and passed through the model in a single forward pass with attention outputs enabled; the models under study generate no text, so our claims concern internal processing of labelled responses rather than online generation dynamics, and are associational rather than causal (\cref{sec:discussion}).
Zero-shot methods (OC: conductance $\phihat$; AG: asymmetry coefficient $G$; LLM-Check) are evaluated on a 30\% held-out test split; our EigenTrack port uses 5-fold cross-validation with a linear classifier (logistic regression with dropout 0.1).
Models span two tiers by design.  The \emph{mechanistic grid} comprises GPT-2 and
Pythia-160M (decoder-only), BERT (encoder-only), and Flan-T5 (encoder--decoder with
cross-attention and decoder self-attention): compact models on which the full
measurement program, exact raw-graph temporal-cut landscapes (\cref{sec:cheeger}),
matched-null refits, degree-preserving decompositions, baseline ports, and feature
ablations, is computed exhaustively across all layers, heads, and aggregations.
The \emph{modern panel} comprises Qwen2.5 (0.5B, 1.5B, 3B), SmolLM2-1.7B, and
LLaMA-3.1-8B (grouped-query attention), decoder-only models in current use, on which
we test transfer of the transport diagnostics under the identical protocol
(\cref{sec:llama-generalization}).
We evaluate using bootstrap confidence intervals (1000 samples, 95\% CI) on 15 model-dataset combinations.
All empirical claims in this paper use LC-AUROC as the primary metric, with raw AUROC reported only for comparability.
Features are aggregated across all layers and heads using CVaR (conditional value-at-risk; \citealp{RockafellarUryasev2000CVaR}) and robust statistics (mean, median, std, IQR, range),
which avoids tuning a per-layer or per-head window; the aggregation function itself remains a modeling choice, selected as described in \cref{sec:evaluation}.
As above, reported values use label-informed polarity; deployment additionally requires polarity and threshold calibration (\cref{sec:discussion}).

\paragraph{Baseline method limitations.}
Some method--architecture combinations are undefined and marked ``---'' in Table~\ref{tab:main-results}:
(i)~\emph{Cross-attention baselines}: LLM-Check and EigenTrack analyze self-attention or hidden states; cross-attention is not supported by these methods.
(ii)~\emph{Encoder-only logit entropy}: BERT lacks autoregressive token prediction, so LLM-Check logit features are undefined.
(iii)~\emph{Asymmetric guessing ($G$) for cross-attention}: The Hermitian decomposition requires square attention matrices; cross-attention matrices are rectangular ($n_q \neq n_k$), making $G$ undefined.

\subsection{Empirical Illustration}

We validate this analysis on a composite design set (900 samples: 300 HaluEval, 300 TruthfulQA, 300 MedHallu).
Response length alone achieves AUROC $\lenAurocHalueval$ on HaluEval, $\lenAurocMedhallu$ on MedHallu, and $\lenAurocTruthfulqa$ on TruthfulQA, confirming length is a strong univariate predictor where correlation exists.
Spectral features inherit this confounding: $\lambda_2$ shows $r = 0.78$ correlation with length on GPT-2, and normalizing by $\sqrt{n}$ over-corrects to $r = {-}0.61$, indicating finite-size effects; the Online Supplement (\suppref{app:finite_size_small_nk}) gives the per-method length-confounding derivations and why random-matrix theory fails for attention.

\medskip\noindent
With LC-AUROC defined and confounding channels characterized, we report results organized around the three research questions.

\section{Empirical Results}
\label{sec:evaluation}

With the transport framework defined (\cref{sec:transport-operators,sec:cheeger,sec:orientation-blindness,sec:asymmetric-guessing}) and the evaluation protocol established (\cref{sec:length-confound}), we test three research questions.  We use
hallucination benchmarks as the testbed throughout: they supply \emph{labeled} instances
of routing failure, but the orientation-blindness limit and the conductance landscape are
properties of the attention operator, so the questions below concern what the diagnostics
measure, not hallucination per se.
The experiments characterize what each diagnostic axis measures in
practice: where it carries signal, what structural features drive that
signal, and what it structurally cannot detect.  AUROC serves as
quantitative evidence for these structural characterizations, not as an end in
itself; where we name a candidate mechanism, we mark it as a hypothesis
consistent with the patterns rather than a demonstrated cause.
Models span the two tiers of \cref{sec:length-confound}: Q1--Q3 run on the
mechanistic grid (GPT-2, Pythia-160M, BERT, Flan-T5), where the full measurement
program is exhaustive, and \cref{sec:llama-generalization} tests transfer on the
modern panel (Qwen2.5 at 0.5/1.5/3B, SmolLM2-1.7B, LLaMA-3.1-8B) under the
identical protocol.


\begin{table*}[t]
\centering
\caption{%
  Length-controlled AUROC comparison on the mechanistic grid: zero-shot transport
  features versus baselines (modern-panel transfer in
  \cref{fig:model-panel-bars} and \cref{sec:llama-generalization}).
  \textbf{Left of bar}: Our zero-shot features---$\phihat$ CVaR$_{75}$ (high conductance),
  CVaR$_{25}$ (low conductance); $\sigma_2$ std (spectral variability).
  \textbf{Right of bar}: Baselines---LLM-Check (attn/hidden/logit probes) and EigenTrack (supervised).
  Values show AUROC with 95\% bootstrap confidence intervals (super/subscript notation).
  \textbf{Bold} = best per row; this paper's contribution is diagnostic characterization, not detection ranking.
  Per-row cell shading: green for the top two values (darker = higher), vermillion for the bottom two (darker = lower); tied values share a shade.
  All methods use length-controlled (within-bin residualized, pair-weighted) evaluation.
  Reported values use \emph{label-informed polarity} (the flip is selected on the
  evaluation split): the features are computed without labels, but a deployable
  detector additionally needs a small calibration set to fix the sign and
  threshold (\cref{sec:eval-protocol}).
  $G$ std (asymmetry) results appear in Figure~\ref{fig:g-barchart}.
}
\label{tab:main-results}
\footnotesize
\setlength{\tabcolsep}{5pt}
\begin{filecontents*}{data/tables/main_results_combined.csv}
# Main Results Table: Zero-shot features + Baselines
# Generated: 2026-02-05 (per-dataset data with bootstrap CIs)
# Our features: phi (conductance variants), sigma2 (spectral norm), G (asymmetric guessing)
# Baselines: LLMCheck (attn/hidden/logit) + EigenTrack (all per-dataset)
# Format: LaTeX-ready with super/subscript CIs and \best{} markup
# Best method per row (across ALL features and baselines) is bolded
Model,Dataset,phi_cvar_top,phi_cvar_bottom,phi_mean,sigma2_std,G_std,LLM_attn,LLM_hidden,LLM_logit,EigenTrack
GPT-2,HaluEval,0.68$_{\scriptscriptstyle .66}^{\scriptscriptstyle .71}$,\cellcolor{OIverm!50!white}0.55$_{\scriptscriptstyle .53}^{\scriptscriptstyle .58}$,0.64$_{\scriptscriptstyle .61}^{\scriptscriptstyle .67}$,\cellcolor{OIgreen!50!white}\best{0.77$_{\scriptscriptstyle .74}^{\scriptscriptstyle .79}$},0.63$_{\scriptscriptstyle .60}^{\scriptscriptstyle .66}$,0.64$_{\scriptscriptstyle .57}^{\scriptscriptstyle .75}$,\cellcolor{OIgreen!25!white}0.71$_{\scriptscriptstyle .62}^{\scriptscriptstyle .81}$,0.69$_{\scriptscriptstyle .62}^{\scriptscriptstyle .81}$,\cellcolor{OIverm!25!white}0.63$_{\scriptscriptstyle .61}^{\scriptscriptstyle .66}$
GPT-2,TruthfulQA,\cellcolor{OIgreen!25!white}0.61$_{\scriptscriptstyle .59}^{\scriptscriptstyle .64}$,0.57$_{\scriptscriptstyle .54}^{\scriptscriptstyle .59}$,0.59$_{\scriptscriptstyle .57}^{\scriptscriptstyle .62}$,0.60$_{\scriptscriptstyle .58}^{\scriptscriptstyle .62}$,0.58$_{\scriptscriptstyle .56}^{\scriptscriptstyle .61}$,\cellcolor{OIverm!25!white}0.55$_{\scriptscriptstyle .53}^{\scriptscriptstyle .63}$,\cellcolor{OIgreen!50!white}\best{0.63$_{\scriptscriptstyle .58}^{\scriptscriptstyle .71}$},0.58$_{\scriptscriptstyle .53}^{\scriptscriptstyle .65}$,\cellcolor{OIverm!50!white}0.50$_{\scriptscriptstyle .50}^{\scriptscriptstyle .50}$
GPT-2,MedHallu,0.58$_{\scriptscriptstyle .55}^{\scriptscriptstyle .64}$,0.61$_{\scriptscriptstyle .57}^{\scriptscriptstyle .66}$,0.62$_{\scriptscriptstyle .57}^{\scriptscriptstyle .67}$,\cellcolor{OIverm!50!white}0.55$_{\scriptscriptstyle .52}^{\scriptscriptstyle .61}$,0.61$_{\scriptscriptstyle .57}^{\scriptscriptstyle .66}$,0.62$_{\scriptscriptstyle .57}^{\scriptscriptstyle .69}$,\cellcolor{OIgreen!50!white}\best{0.65$_{\scriptscriptstyle .59}^{\scriptscriptstyle .71}$},\cellcolor{OIgreen!25!white}0.64$_{\scriptscriptstyle .58}^{\scriptscriptstyle .72}$,\cellcolor{OIverm!25!white}0.57$_{\scriptscriptstyle .54}^{\scriptscriptstyle .64}$
BERT,HaluEval,\cellcolor{OIverm!25!white}0.57$_{\scriptscriptstyle .55}^{\scriptscriptstyle .60}$,0.59$_{\scriptscriptstyle .56}^{\scriptscriptstyle .62}$,0.60$_{\scriptscriptstyle .58}^{\scriptscriptstyle .63}$,\cellcolor{OIverm!50!white}0.55$_{\scriptscriptstyle .53}^{\scriptscriptstyle .58}$,0.58$_{\scriptscriptstyle .55}^{\scriptscriptstyle .60}$,\cellcolor{OIgreen!25!white}0.66$_{\scriptscriptstyle .59}^{\scriptscriptstyle .79}$,\cellcolor{OIgreen!50!white}\best{0.70$_{\scriptscriptstyle .59}^{\scriptscriptstyle .83}$},,0.62$_{\scriptscriptstyle .59}^{\scriptscriptstyle .64}$
BERT,TruthfulQA,\cellcolor{OIverm!25!white}0.55$_{\scriptscriptstyle .53}^{\scriptscriptstyle .58}$,0.56$_{\scriptscriptstyle .54}^{\scriptscriptstyle .58}$,0.52$_{\scriptscriptstyle .51}^{\scriptscriptstyle .55}$,\cellcolor{OIverm!25!white}0.55$_{\scriptscriptstyle .53}^{\scriptscriptstyle .58}$,0.59$_{\scriptscriptstyle .57}^{\scriptscriptstyle .62}$,\cellcolor{OIgreen!25!white}0.57$_{\scriptscriptstyle .54}^{\scriptscriptstyle .65}$,\cellcolor{OIgreen!50!white}\best{0.62$_{\scriptscriptstyle .57}^{\scriptscriptstyle .69}$},,\cellcolor{OIverm!50!white}0.52$_{\scriptscriptstyle .51}^{\scriptscriptstyle .56}$
BERT,MedHallu,\cellcolor{OIverm!25!white}0.58$_{\scriptscriptstyle .55}^{\scriptscriptstyle .63}$,\cellcolor{OIverm!50!white}0.54$_{\scriptscriptstyle .52}^{\scriptscriptstyle .60}$,0.58$_{\scriptscriptstyle .55}^{\scriptscriptstyle .63}$,0.63$_{\scriptscriptstyle .59}^{\scriptscriptstyle .68}$,0.59$_{\scriptscriptstyle .56}^{\scriptscriptstyle .64}$,\cellcolor{OIgreen!25!white}0.64$_{\scriptscriptstyle .59}^{\scriptscriptstyle .71}$,\cellcolor{OIgreen!50!white}\best{0.72$_{\scriptscriptstyle .64}^{\scriptscriptstyle .78}$},,\cellcolor{OIverm!50!white}0.54$_{\scriptscriptstyle .53}^{\scriptscriptstyle .62}$
Pythia-160M,HaluEval,\cellcolor{OIgreen!25!white}0.82$_{\scriptscriptstyle .80}^{\scriptscriptstyle .84}$,\cellcolor{OIverm!25!white}0.62$_{\scriptscriptstyle .59}^{\scriptscriptstyle .64}$,0.83$_{\scriptscriptstyle .80}^{\scriptscriptstyle .85}$,\cellcolor{OIgreen!50!white}\best{0.84$_{\scriptscriptstyle .81}^{\scriptscriptstyle .86}$},0.82$_{\scriptscriptstyle .80}^{\scriptscriptstyle .84}$,0.70$_{\scriptscriptstyle .62}^{\scriptscriptstyle .82}$,0.72$_{\scriptscriptstyle .63}^{\scriptscriptstyle .80}$,0.80$_{\scriptscriptstyle .59}^{\scriptscriptstyle .89}$,\cellcolor{OIverm!50!white}0.60$_{\scriptscriptstyle .58}^{\scriptscriptstyle .62}$
Pythia-160M,TruthfulQA,\cellcolor{OIverm!25!white}0.57$_{\scriptscriptstyle .55}^{\scriptscriptstyle .59}$,\cellcolor{OIgreen!25!white}0.58$_{\scriptscriptstyle .56}^{\scriptscriptstyle .61}$,0.57$_{\scriptscriptstyle .55}^{\scriptscriptstyle .60}$,\cellcolor{OIverm!25!white}0.57$_{\scriptscriptstyle .55}^{\scriptscriptstyle .60}$,0.60$_{\scriptscriptstyle .58}^{\scriptscriptstyle .62}$,\cellcolor{OIverm!50!white}0.52$_{\scriptscriptstyle .52}^{\scriptscriptstyle .61}$,\cellcolor{OIgreen!50!white}\best{0.63$_{\scriptscriptstyle .58}^{\scriptscriptstyle .72}$},\cellcolor{OIverm!50!white}0.52$_{\scriptscriptstyle .52}^{\scriptscriptstyle .60}$,\cellcolor{OIverm!50!white}0.52$_{\scriptscriptstyle .51}^{\scriptscriptstyle .56}$
Pythia-160M,MedHallu,\cellcolor{OIgreen!25!white}0.61$_{\scriptscriptstyle .58}^{\scriptscriptstyle .67}$,\cellcolor{OIgreen!50!white}\best{0.62$_{\scriptscriptstyle .58}^{\scriptscriptstyle .68}$},0.60$_{\scriptscriptstyle .57}^{\scriptscriptstyle .65}$,\cellcolor{OIverm!25!white}0.59$_{\scriptscriptstyle .55}^{\scriptscriptstyle .65}$,0.61$_{\scriptscriptstyle .57}^{\scriptscriptstyle .66}$,\cellcolor{OIgreen!50!white}0.62$_{\scriptscriptstyle .57}^{\scriptscriptstyle .69}$,\cellcolor{OIgreen!50!white}0.62$_{\scriptscriptstyle .57}^{\scriptscriptstyle .69}$,\cellcolor{OIverm!50!white}0.58$_{\scriptscriptstyle .54}^{\scriptscriptstyle .65}$,\cellcolor{OIgreen!25!white}0.61$_{\scriptscriptstyle .56}^{\scriptscriptstyle .68}$
Flan-T5 (cross),HaluEval,\cellcolor{OIverm!25!white}0.62$_{\scriptscriptstyle .59}^{\scriptscriptstyle .64}$,\cellcolor{OIgreen!50!white}\best{0.71$_{\scriptscriptstyle .68}^{\scriptscriptstyle .73}$},0.70$_{\scriptscriptstyle .67}^{\scriptscriptstyle .72}$,\cellcolor{OIgreen!25!white}0.66$_{\scriptscriptstyle .64}^{\scriptscriptstyle .69}$,0.53$_{\scriptscriptstyle .52}^{\scriptscriptstyle .56}$,\cellcolor{OIverm!25!white}0.62$_{\scriptscriptstyle .55}^{\scriptscriptstyle .74}$,,,\cellcolor{OIverm!50!white}0.50$_{\scriptscriptstyle .50}^{\scriptscriptstyle .50}$
Flan-T5 (cross),TruthfulQA,\cellcolor{OIverm!50!white}0.53$_{\scriptscriptstyle .52}^{\scriptscriptstyle .56}$,\cellcolor{OIgreen!25!white}0.58$_{\scriptscriptstyle .55}^{\scriptscriptstyle .60}$,0.57$_{\scriptscriptstyle .55}^{\scriptscriptstyle .60}$,\cellcolor{OIverm!25!white}0.57$_{\scriptscriptstyle .55}^{\scriptscriptstyle .60}$,0.53$_{\scriptscriptstyle .52}^{\scriptscriptstyle .56}$,\cellcolor{OIgreen!50!white}\best{0.59$_{\scriptscriptstyle .55}^{\scriptscriptstyle .67}$},,,\cellcolor{OIverm!50!white}0.53$_{\scriptscriptstyle .52}^{\scriptscriptstyle .57}$
Flan-T5 (cross),MedHallu,0.56$_{\scriptscriptstyle .53}^{\scriptscriptstyle .61}$,\cellcolor{OIgreen!25!white}0.67$_{\scriptscriptstyle .63}^{\scriptscriptstyle .72}$,0.62$_{\scriptscriptstyle .59}^{\scriptscriptstyle .68}$,\cellcolor{OIgreen!50!white}\best{0.68$_{\scriptscriptstyle .63}^{\scriptscriptstyle .73}$},0.58$_{\scriptscriptstyle .55}^{\scriptscriptstyle .64}$,\cellcolor{OIverm!25!white}0.55$_{\scriptscriptstyle .53}^{\scriptscriptstyle .64}$,,,\cellcolor{OIverm!50!white}0.50$_{\scriptscriptstyle .50}^{\scriptscriptstyle .50}$
Flan-T5 (dec),HaluEval,\cellcolor{OIgreen!25!white}0.73$_{\scriptscriptstyle .71}^{\scriptscriptstyle .75}$,\cellcolor{OIgreen!25!white}0.73$_{\scriptscriptstyle .71}^{\scriptscriptstyle .76}$,0.76$_{\scriptscriptstyle .73}^{\scriptscriptstyle .78}$,0.69$_{\scriptscriptstyle .67}^{\scriptscriptstyle .72}$,0.78$_{\scriptscriptstyle .76}^{\scriptscriptstyle .80}$,\cellcolor{OIgreen!50!white}\best{0.77$_{\scriptscriptstyle .69}^{\scriptscriptstyle .85}$},0.72$_{\scriptscriptstyle .60}^{\scriptscriptstyle .82}$,\cellcolor{OIverm!25!white}0.68$_{\scriptscriptstyle .61}^{\scriptscriptstyle .80}$,\cellcolor{OIverm!50!white}0.50$_{\scriptscriptstyle .50}^{\scriptscriptstyle .50}$
Flan-T5 (dec),TruthfulQA,\cellcolor{OIverm!25!white}0.54$_{\scriptscriptstyle .53}^{\scriptscriptstyle .58}$,0.55$_{\scriptscriptstyle .53}^{\scriptscriptstyle .58}$,0.55$_{\scriptscriptstyle .53}^{\scriptscriptstyle .58}$,\cellcolor{OIverm!50!white}0.53$_{\scriptscriptstyle .51}^{\scriptscriptstyle .56}$,0.57$_{\scriptscriptstyle .54}^{\scriptscriptstyle .59}$,\cellcolor{OIgreen!50!white}\best{0.60$_{\scriptscriptstyle .57}^{\scriptscriptstyle .67}$},\cellcolor{OIgreen!25!white}0.59$_{\scriptscriptstyle .56}^{\scriptscriptstyle .67}$,0.56$_{\scriptscriptstyle .53}^{\scriptscriptstyle .64}$,\cellcolor{OIverm!50!white}0.53$_{\scriptscriptstyle .52}^{\scriptscriptstyle .57}$
Flan-T5 (dec),MedHallu,\cellcolor{OIverm!25!white}0.53$_{\scriptscriptstyle .52}^{\scriptscriptstyle .58}$,0.57$_{\scriptscriptstyle .53}^{\scriptscriptstyle .61}$,0.55$_{\scriptscriptstyle .53}^{\scriptscriptstyle .61}$,0.56$_{\scriptscriptstyle .53}^{\scriptscriptstyle .61}$,0.58$_{\scriptscriptstyle .55}^{\scriptscriptstyle .64}$,\cellcolor{OIgreen!25!white}0.59$_{\scriptscriptstyle .55}^{\scriptscriptstyle .66}$,\cellcolor{OIgreen!50!white}0.60$_{\scriptscriptstyle .55}^{\scriptscriptstyle .67}$,\cellcolor{OIgreen!50!white}\best{0.60$_{\scriptscriptstyle .56}^{\scriptscriptstyle .69}$},\cellcolor{OIverm!50!white}0.50$_{\scriptscriptstyle .50}^{\scriptscriptstyle .50}$
\end{filecontents*}
\pgfplotstabletypeset[
  col sep=comma,
  string type,
  skip first n=5,  
  columns={Model, Dataset, phi_cvar_top, phi_cvar_bottom, sigma2_std, LLM_attn, LLM_hidden, LLM_logit, EigenTrack},
  every head row/.style={
    before row={\toprule
      & & \multicolumn{3}{c}{\textbf{Transport features (ours)}}
        & \multicolumn{4}{c}{\textbf{Baselines}} \\
      \cmidrule(lr){3-5}\cmidrule(lr){6-9}},
    after row=\midrule
  },
  every last row/.style={
    after row=\bottomrule
  },
  columns/Model/.style={column name=Model, column type=l},
  columns/Dataset/.style={column name=Dataset, column type=l},
  columns/phi_cvar_top/.style={column name=$\phihat$ C$_{75}$, column type=c},
  columns/phi_cvar_bottom/.style={column name=$\phihat$ C$_{25}$, column type=c},
  columns/sigma2_std/.style={column name=$\sigma_2$ std, column type=c|},  
  columns/LLM_attn/.style={column name=LLM$_\text{attn}$, column type=c},
  columns/LLM_hidden/.style={column name=LLM$_\text{hid}$, column type=c},
  columns/LLM_logit/.style={column name=LLM$_\text{log}$, column type=c},
  columns/EigenTrack/.style={column name=ET, column type=c},
]{data/tables/main_results_combined.csv}
\end{table*}

\subsection{Q1: Does the Symmetric Axis Capture Capacity Failures under Length Control?}
\label{sec:eval-q1}

\paragraph{Response-length baseline.}
Response length alone achieves raw AUROC of $\lenAurocHalueval$ on HaluEval, $\lenAurocMedhallu$ on MedHallu,
and $\lenAurocTruthfulqa$ on TruthfulQA (\cref{sec:length-confound}). Under LC-AUROC, response
length collapses to ${\approx}0.50$ by construction, since within-bin residualization
removes the signal that length provides. Transport features retain meaningful
LC-AUROC (0.62--0.84), indicating that their discriminative power is not reducible
to length encoding.

\paragraph{Two failure regimes persist after length control.}
Table~\ref{tab:main-results} presents LC-AUROC for attention transport
features aggregated across all layers and heads.
CVaR tail analysis decomposes the conductance distribution into its
structural extremes: high-conductance regions (CVaR$_{75}$, top 25\%)
capture diffuse routing where attention dilutes across too many tokens;
low-conductance regions (CVaR$_{25}$, bottom 25\%) capture
over-constrained bottlenecks where attention concentrates on too few
tokens.  On Pythia-160M/HaluEval the diffuse tail produces a stronger hallucination signal
(\pythiaHalPhiDiffuse{} [\pythiaHalPhiDiffuseLo{}, \pythiaHalPhiDiffuseHi{}]) than the bottleneck tail (\pythiaHalPhiBottleneck{} [\pythiaHalPhiBottleneckLo{}, \pythiaHalPhiBottleneckHi{}]),
and overall spectral norm variability reaches \pythiaHalSigmaTwoStd{} [\pythiaHalSigmaTwoStdLo{}, \pythiaHalSigmaTwoStdHi{}] on
HaluEval.  This ranking concerns which extreme of the per-head conductance
distribution makes the better \emph{aggregated feature}; it is a separate
question from which routing regime a benchmark occupies, which we test below by
per-sample stratification.  The two need not coincide: a tail aggregation can
discriminate even where hallucination mass concentrates at the opposite
extreme, so the polarity prediction is evaluated on the stratification, not on
this tail ranking.


\begin{figure*}[!ht]
\centering
\begin{tikzpicture}
\begin{axis}[
    academic-single,
    name=phiplot,
    width=0.48\textwidth,
    height=0.35\textwidth,
    xlabel={Layer index},
    ylabel={Conductance $\phihat$ (layer mean)},
    xmin=-0.5, xmax=11.5,
    ymin=0, ymax=0.32,
    xtick={0,2,...,10},
    title={\small (a) Per-layer conductance},
    title style={at={(0.5,1.02)}, anchor=south},
    legend to name=sharedlegend,
    legend style={
        font=\scriptsize,
        draw=none,
        legend columns=5,
        /tikz/every even column/.append style={column sep=4pt},
    },
]

\addplot[name path=fac_hi, draw=none, forget plot]
    table[x=layer, y=hi] {content/figures/per_layer_profiles/data/phi_pop_factual.dat};
\addplot[name path=fac_lo, draw=none, forget plot]
    table[x=layer, y=lo] {content/figures/per_layer_profiles/data/phi_pop_factual.dat};
\addplot[fill=factual, fill opacity=0.12, forget plot]
    fill between[of=fac_hi and fac_lo];
\addplot[factual, thin, densely dotted, forget plot]
    table[x=layer, y=mean] {content/figures/per_layer_profiles/data/phi_pop_factual.dat};

\addplot[name path=hal_hi, draw=none, forget plot]
    table[x=layer, y=hi] {content/figures/per_layer_profiles/data/phi_pop_hallucinated.dat};
\addplot[name path=hal_lo, draw=none, forget plot]
    table[x=layer, y=lo] {content/figures/per_layer_profiles/data/phi_pop_hallucinated.dat};
\addplot[fill=hallucinate, fill opacity=0.12, forget plot]
    fill between[of=hal_hi and hal_lo];
\addplot[hallucinate, thin, densely dotted, forget plot]
    table[x=layer, y=mean] {content/figures/per_layer_profiles/data/phi_pop_hallucinated.dat};

\addplot[hallucinate, thick, mark=*, mark size=1.5pt]
    table[x=layer, y=phi] {content/figures/per_layer_profiles/data/phi_bottleneck.dat};
\addlegendentry{Bottleneck (hal.)}

\addplot[spectral, thick, mark=square*, mark size=1.5pt]
    table[x=layer, y=phi] {content/figures/per_layer_profiles/data/phi_diffuse.dat};
\addlegendentry{Diffuse (hal.)}

\addplot[factual, thick, dashed, mark=triangle*, mark size=1.8pt]
    table[x=layer, y=phi] {content/figures/per_layer_profiles/data/phi_factual_baseline.dat};
\addlegendentry{Factual baseline}

\addplot[factual, thin, densely dotted]
    coordinates {(-1,-1)};
\addlegendentry{Pop.\ mean (fac.)}
\addplot[hallucinate, thin, densely dotted]
    coordinates {(-1,-1)};
\addlegendentry{Pop.\ mean (hal.)}

\end{axis}

\begin{axis}[
    academic-single,
    at={(phiplot.outer east)},
    anchor=outer west,
    xshift=0.4cm,
    width=0.48\textwidth,
    height=0.35\textwidth,
    xlabel={Layer index},
    ylabel={Spectral norm $\sigma_2$ (layer mean)},
    xmin=-0.5, xmax=11.5,
    ymin=0, ymax=1.05,
    xtick={0,2,...,10},
    title={\small (b) Per-layer spectral norm},
    title style={at={(0.5,1.02)}, anchor=south},
]

\addplot[name path=sfac_hi, draw=none, forget plot]
    table[x=layer, y=hi] {content/figures/per_layer_profiles/data/sigma2_pop_factual.dat};
\addplot[name path=sfac_lo, draw=none, forget plot]
    table[x=layer, y=lo] {content/figures/per_layer_profiles/data/sigma2_pop_factual.dat};
\addplot[fill=factual, fill opacity=0.12, forget plot]
    fill between[of=sfac_hi and sfac_lo];
\addplot[factual, thin, densely dotted, forget plot]
    table[x=layer, y=mean] {content/figures/per_layer_profiles/data/sigma2_pop_factual.dat};

\addplot[name path=shal_hi, draw=none, forget plot]
    table[x=layer, y=hi] {content/figures/per_layer_profiles/data/sigma2_pop_hallucinated.dat};
\addplot[name path=shal_lo, draw=none, forget plot]
    table[x=layer, y=lo] {content/figures/per_layer_profiles/data/sigma2_pop_hallucinated.dat};
\addplot[fill=hallucinate, fill opacity=0.12, forget plot]
    fill between[of=shal_hi and shal_lo];
\addplot[hallucinate, thin, densely dotted, forget plot]
    table[x=layer, y=mean] {content/figures/per_layer_profiles/data/sigma2_pop_hallucinated.dat};

\addplot[hallucinate, thick, mark=*, mark size=1.5pt]
    table[x=layer, y=sigma2] {content/figures/per_layer_profiles/data/sigma2_bottleneck.dat};

\addplot[spectral, thick, mark=square*, mark size=1.5pt]
    table[x=layer, y=sigma2] {content/figures/per_layer_profiles/data/sigma2_diffuse.dat};

\addplot[factual, thick, dashed, mark=triangle*, mark size=1.8pt]
    table[x=layer, y=sigma2] {content/figures/per_layer_profiles/data/sigma2_factual_baseline.dat};

\end{axis}
\end{tikzpicture}

\vspace{-2pt}
\pgfplotslegendfromname{sharedlegend}

\caption{\textbf{Per-layer transport profiles reveal distinct failure signatures.}
Conductance $\phihat$ \textbf{(a)} and spectral norm $\sigma_2$ \textbf{(b)}
averaged across heads within each layer (Pythia-160M).
Shaded bands show population $\pm 1$ std around dotted mean lines
(green: factual, red: hallucinated from HaluEval).
\textcolor{hallucinate}{\textbf{Bottleneck}} (HaluEval, hallucinated): uniformly depressed
$\phihat$ and elevated $\sigma_2$ across all layers---the spectral gap is small everywhere,
indicating global over-concentration.
\textcolor{spectral}{\textbf{Diffuse}} (MedHallu, hallucinated): elevated $\phihat$ in
middle layers with $\sigma_2$ slightly lower---attention spreads too broadly.
\textcolor{factual}{\textbf{Factual}}: intermediate values within population bands.}
\label{fig:per-layer-profiles}
\end{figure*}
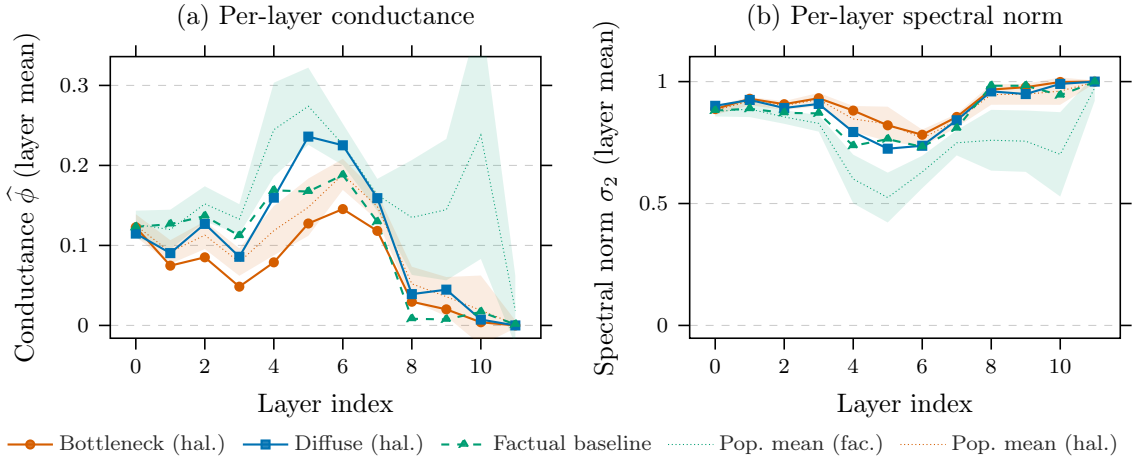

Per-layer profiles are consistent with these failures being structurally distinct
across the full transformer depth, rather than artifacts of individual layers
(Figure~\ref{fig:per-layer-profiles}): bottleneck samples show uniformly depressed
$\phihat$ across all 12 layers, while diffuse samples show selectively elevated
$\phihat$ in middle layers.


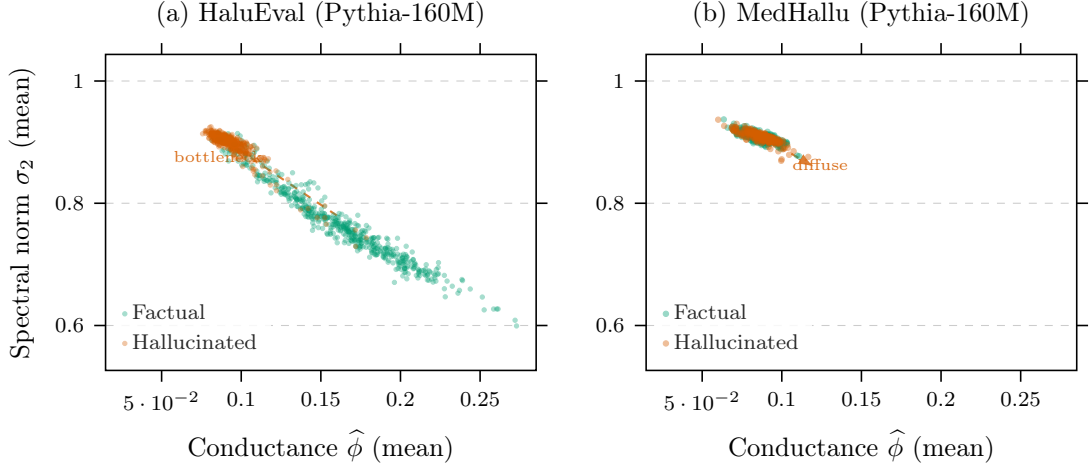
\begin{figure*}[!ht]
\centering
\begin{tikzpicture}
\begin{axis}[
    academic-single,
    name=haluplot,
    width=0.48\textwidth,
    height=0.38\textwidth,
    xlabel={Conductance $\phihat$ (mean)},
    ylabel={Spectral norm $\sigma_2$ (mean)},
    xmin=0.02, xmax=0.28,
    ymin=0.55, ymax=1.02,
    legend style={
        at={(0.02,0.02)},
        anchor=south west,
        font=\scriptsize,
        draw=none,
        fill=white,
        fill opacity=0.85,
    },
    title={\small (a) HaluEval (Pythia-160M)},
    title style={at={(0.5,1.02)}, anchor=south},
]

\addplot[
    only marks,
    mark=*,
    mark size=0.8pt,
    factual,
    opacity=0.35,
] table[x=phi, y=sigma2] {content/figures/scatter_comparison/data/halueval_factual.dat};
\addlegendentry{Factual}

\addplot[
    only marks,
    mark=*,
    mark size=0.8pt,
    hallucinate,
    opacity=0.35,
] table[x=phi, y=sigma2] {content/figures/scatter_comparison/data/halueval_hallucinated.dat};
\addlegendentry{Hallucinated}

\draw[-{Latex[length=2mm]}, thick, hallucinate!80, dashed]
    (axis cs:0.16,0.78) -- (axis cs:0.08,0.92)
    node[pos=0.5, above left, font=\tiny, text=hallucinate!90] {bottleneck};

\end{axis}

\begin{axis}[
    academic-single,
    at={(haluplot.outer east)},
    anchor=outer west,
    xshift=0.5cm,
    width=0.48\textwidth,
    height=0.38\textwidth,
    xlabel={Conductance $\phihat$ (mean)},
    ylabel={},
    xmin=0.02, xmax=0.28,
    ymin=0.55, ymax=1.02,
    legend style={
        at={(0.02,0.02)},
        anchor=south west,
        font=\scriptsize,
        draw=none,
        fill=white,
        fill opacity=0.85,
    },
    title={\small (b) MedHallu (Pythia-160M)},
    title style={at={(0.5,1.02)}, anchor=south},
]

\addplot[
    only marks,
    mark=*,
    mark size=1.0pt,
    factual,
    opacity=0.4,
] table[x=phi, y=sigma2] {content/figures/scatter_comparison/data/medhallu_factual.dat};
\addlegendentry{Factual}

\addplot[
    only marks,
    mark=*,
    mark size=1.0pt,
    hallucinate,
    opacity=0.4,
] table[x=phi, y=sigma2] {content/figures/scatter_comparison/data/medhallu_hallucinated.dat};
\addlegendentry{Hallucinated}

\draw[-{Latex[length=2mm]}, thick, hallucinate!80, dashed]
    (axis cs:0.08,0.92) -- (axis cs:0.12,0.86)
    node[pos=0.5, below right, font=\tiny, text=hallucinate!90] {diffuse};

\end{axis}
\end{tikzpicture}

\caption{\textbf{Conductance--spectral norm scatter reveals regime-dependent polarity.}
Each point is one sample; axes show conductance $\phihat$ and spectral norm $\sigma_2$
averaged across all heads and layers (Pythia-160M, 500 subsampled per class for HaluEval).
The near-perfect anti-correlation ($\rho{=}{-}0.99$) confirms the Cheeger inequality:
low $\phihat$ implies high $\sigma_2$ (small spectral gap).
\textbf{(a)}~HaluEval: hallucinated samples cluster at low $\phihat$ / high $\sigma_2$
(bottleneck regime, Cohen's $d{=}{-}2.6$).
\textbf{(b)}~MedHallu: hallucinated samples shift toward higher $\phihat$ / lower $\sigma_2$
(diffuse regime, Cohen's $d{=}{+}0.3$).
The polarity reversal between data sets visualizes the two-sided diagnostic:
the same spectral signature detects opposite failure modes across data sets.}
\label{fig:phi-sigma2-scatter}
\end{figure*}

Conductance also dissociates from simpler attention statistics: mean attention
entropy and conductance are uncorrelated ($\rho{<}0.04$) and entropy itself
has near-zero discriminative power (Cohen's $d{=}0.04$ on HaluEval vs.\
$d{=}{-}2.6$ for $\phihat$), so conductance captures
graph-theoretic transport structure not reducible to per-row entropy.

\paragraph{Polarity variation reflects regime-dependent failure modes.}
Between-dataset polarity variation is \emph{consistent with} the two-sided theory
and \emph{supported} by tercile analysis. Stratifying samples by OC tercile
reveals which failure mode dominates per data set.
On HaluEval, hallucinations cluster in the \emph{low}-OC tercile (mean 47.5\%
vs.\ 33\% expected), consistent with bottleneck routing (attention over-concentrating
on a few tokens).
On MedHallu, hallucinations cluster in the \emph{high}-OC tercile (mean 45.1\%),
consistent with diffuse routing (attention spread broadly across tokens).
The same diagnostic (OC) carries signal for failure in both cases; the
\emph{direction} of pathology differs (visualized in \cref{fig:phi-sigma2-scatter}).
The reversal is directional as predicted, but \emph{asymmetric in strength}: the
HaluEval bottleneck signal is markedly stronger than the MedHallu diffuse signal
(\cref{tab:main-results}), so the prediction is borne out in sign more sharply
than in magnitude.
Full tercile distributions appear in the Online Supplement, \suppref{app:complementary-diagnostics}.
Matched null baselines (\cref{sec:cheeger}) confirm these signatures reflect
learned structure rather than finite-size estimator artifacts: z-score
normalization against entropy-matched and degree-preserving nulls improves
AUROC by 6--8 points.

\paragraph{Why the polarity reverses: the architectural-signature account.}
The architectural-fraction analysis of \cref{sec:cheeger}
(\cref{tab:landscape-empirical-summary}) offers a structural account
of the dataset-specific polarity reversal.  Architectures differ
distributionally in the fraction of heads falling below the $1/5$ temporal-cut floor
(GPT-2 $\floorViolGPTLo$--$\floorViolGPTHi\%$, Pythia-160M $\floorViolPyLo$--$\floorViolPyHi\%$, Flan-T5 decoder $\floorViolFlanLo$--$\floorViolFlanHi\%$),
not head-by-head.  On HaluEval (hallucinations in the low-OC tercile) the bottleneck-routing
signature is the discriminative one, while on MedHallu (hallucinations in the high-OC
tercile) the diffuse-routing signature is.
Architectures with a larger bottleneck-regime fraction therefore polarise more
on HaluEval; architectures with fewer leave more diffuse-signal weight and
polarise more on MedHallu.  The reversal is distributional, not per-head.

\paragraph{Degree-preserving null decomposition.}
\cref{prop:degree-sufficiency} bounds $\sigma_2$'s coupling-beyond-degree term by
$\sqrt{\kappa}-1$, predicting $\sigma_2$ should retain signal once degrees are controlled;
the degree-preserving null tests this and supplies the complementary \emph{empirical} half
about $\phihat$, which the theorem does not predict.  In the configurations studied,
$\phihat$'s discrimination is largely the \emph{degree distribution} (which tokens receive
attention): its z-AUROC drops to near chance (0.51--0.57) under degree-preserving nulls,
while $\sigma_2$ retains 0.72--0.80, capturing coupling beyond degrees
(\suppref{app:null-empirical}).  The separation holds across length quartiles within each
data set (ANOVA $p>0.05$), so it is not a sequence-length artifact.

\subsection{Q2: Does the Antisymmetric Axis Detect Directional Failures?}
\label{sec:eval-q2}



\begin{figure*}[!ht]
\centering
\resizebox{\textwidth}{!}{%
\begin{tikzpicture}
\begin{axis}[
    academic-single-bar,
    name=barchart,
    scale only axis,
    width=0.24\textwidth,
    height=0.22\textwidth,
    ybar=1.5pt,
    bar width=5pt,
    ylabel={$G$ std LC-AUROC},
    ymin=0.40, ymax=0.92,
    ytick={0.4, 0.5, 0.6, 0.7, 0.8, 0.9},
    ymajorgrids=true,
    xmajorgrids=false,
    symbolic x coords={GPT-2, BERT, Pythia, FT5-C, FT5-D},
    xtick=data,
    x tick label style={font=\scriptsize, rotate=35, anchor=north east},
    enlarge x limits=0.15,
    every axis plot/.append style={fill opacity=0.85},
    title={\small (a) When $G$ activates},
    title style={at={(0.5,1.05)}, anchor=south},
    legend style={
        font=\scriptsize,
        draw=none,
        legend columns=3,
        /tikz/every even column/.append style={column sep=4pt},
        at={(0.5,-0.32)},
        anchor=north,
    },
    legend cell align=left,
    every error bar/.style={line width=0.7pt},
]

\addplot[fill=hallucinate, draw=hallucinate!80!black,
    error bars/y dir=both, error bars/y explicit,
    error bars/error bar style={hallucinate!80!black}]
    table[x=model, y=value, y error=err]
    {content/figures/g_temporal_isolation/data/g_std_halueval.dat};
\addlegendentry{HaluEval}

\addplot[fill=spectral, draw=spectral!80!black,
    error bars/y dir=both, error bars/y explicit,
    error bars/error bar style={spectral!80!black}]
    table[x=model, y=value, y error=err]
    {content/figures/g_temporal_isolation/data/g_std_truthfulqa.dat};
\addlegendentry{TruthfulQA}

\addplot[fill=factual, draw=factual!80!black,
    error bars/y dir=both, error bars/y explicit,
    error bars/error bar style={factual!80!black}]
    table[x=model, y=value, y error=err]
    {content/figures/g_temporal_isolation/data/g_std_medhallu.dat};
\addlegendentry{MedHallu}

\draw[dashed, black!40, line width=0.6pt]
    ([xshift=-12pt]axis cs:GPT-2, 0.5) -- ([xshift=12pt]axis cs:FT5-D, 0.5);
\node[anchor=east, font=\tiny, text=black!50] at ([xshift=-14pt]axis cs:GPT-2, 0.5) {chance};

\end{axis}

\begin{axis}[
    academic-single,
    at={(barchart.outer east)},
    anchor=outer west,
    xshift=0.1cm,
    name=panelb,
    scale only axis,
    width=0.24\textwidth,
    height=0.22\textwidth,
    xlabel={Depth percentile (\%)},
    ylabel={$G$ (layer mean)},
    xmin=-3, xmax=103,
    ymin=0.30, ymax=0.75,
    xtick={0,25,50,75,100},
    title={\small (b) Cross-attn.\ grounding},
    title style={at={(0.5,1.02)}, anchor=south},
]


\addplot[name path=ft5_fac_hi, draw=none, forget plot]
    table[x=depth, y=hi] {content/figures/g_temporal_isolation/data/g_depth_flant5dec_factual.dat};
\addplot[name path=ft5_fac_lo, draw=none, forget plot]
    table[x=depth, y=lo] {content/figures/g_temporal_isolation/data/g_depth_flant5dec_factual.dat};
\addplot[fill=factual, fill opacity=0.15, forget plot]
    fill between[of=ft5_fac_hi and ft5_fac_lo];
\addplot[factual, thick, forget plot]
    table[x=depth, y=mean] {content/figures/g_temporal_isolation/data/g_depth_flant5dec_factual.dat};

\addplot[name path=ft5_hal_hi, draw=none, forget plot]
    table[x=depth, y=hi] {content/figures/g_temporal_isolation/data/g_depth_flant5dec_hallucinated.dat};
\addplot[name path=ft5_hal_lo, draw=none, forget plot]
    table[x=depth, y=lo] {content/figures/g_temporal_isolation/data/g_depth_flant5dec_hallucinated.dat};
\addplot[fill=hallucinate, fill opacity=0.15, forget plot]
    fill between[of=ft5_hal_hi and ft5_hal_lo];
\addplot[hallucinate, thick, forget plot]
    table[x=depth, y=mean] {content/figures/g_temporal_isolation/data/g_depth_flant5dec_hallucinated.dat};


\addplot[name path=gpt2b_fac_hi, draw=none, forget plot]
    table[x=depth, y=hi] {content/figures/g_temporal_isolation/data/g_depth_gpt2_factual.dat};
\addplot[name path=gpt2b_fac_lo, draw=none, forget plot]
    table[x=depth, y=lo] {content/figures/g_temporal_isolation/data/g_depth_gpt2_factual.dat};
\addplot[fill=factual, fill opacity=0.07, forget plot]
    fill between[of=gpt2b_fac_hi and gpt2b_fac_lo];
\addplot[factual, thick, dashed, forget plot]
    table[x=depth, y=mean] {content/figures/g_temporal_isolation/data/g_depth_gpt2_factual.dat};

\addplot[name path=gpt2b_hal_hi, draw=none, forget plot]
    table[x=depth, y=hi] {content/figures/g_temporal_isolation/data/g_depth_gpt2_hallucinated.dat};
\addplot[name path=gpt2b_hal_lo, draw=none, forget plot]
    table[x=depth, y=lo] {content/figures/g_temporal_isolation/data/g_depth_gpt2_hallucinated.dat};
\addplot[fill=hallucinate, fill opacity=0.07, forget plot]
    fill between[of=gpt2b_hal_hi and gpt2b_hal_lo];
\addplot[hallucinate, thick, dashed, forget plot]
    table[x=depth, y=mean] {content/figures/g_temporal_isolation/data/g_depth_gpt2_hallucinated.dat};

\draw[dashed, black!40, line width=0.6pt] (axis cs:-3,0.5) -- (axis cs:103,0.5);
\node[anchor=south west, font=\tiny, text=black!50] at (axis cs:1,0.505) {$G{=}0.5$};

\end{axis}

\begin{axis}[
    academic-single,
    at={(panelb.outer east)},
    anchor=outer west,
    xshift=0.1cm,
    name=panelc,
    scale only axis,
    width=0.24\textwidth,
    height=0.22\textwidth,
    xlabel={Depth percentile (\%)},
    ylabel={},
    yticklabel=\empty,
    xmin=-3, xmax=103,
    ymin=0.30, ymax=0.75,
    xtick={0,25,50,75,100},
    title={\small (c) RoPE positional decay},
    title style={at={(0.5,1.02)}, anchor=south},
]


\addplot[name path=py_fac_hi, draw=none, forget plot]
    table[x=depth, y=hi] {content/figures/g_temporal_isolation/data/g_depth_pythia160m_factual.dat};
\addplot[name path=py_fac_lo, draw=none, forget plot]
    table[x=depth, y=lo] {content/figures/g_temporal_isolation/data/g_depth_pythia160m_factual.dat};
\addplot[fill=factual, fill opacity=0.15, forget plot]
    fill between[of=py_fac_hi and py_fac_lo];
\addplot[factual, thick, forget plot]
    table[x=depth, y=mean] {content/figures/g_temporal_isolation/data/g_depth_pythia160m_factual.dat};

\addplot[name path=py_hal_hi, draw=none, forget plot]
    table[x=depth, y=hi] {content/figures/g_temporal_isolation/data/g_depth_pythia160m_hallucinated.dat};
\addplot[name path=py_hal_lo, draw=none, forget plot]
    table[x=depth, y=lo] {content/figures/g_temporal_isolation/data/g_depth_pythia160m_hallucinated.dat};
\addplot[fill=hallucinate, fill opacity=0.15, forget plot]
    fill between[of=py_hal_hi and py_hal_lo];
\addplot[hallucinate, thick, forget plot]
    table[x=depth, y=mean] {content/figures/g_temporal_isolation/data/g_depth_pythia160m_hallucinated.dat};


\addplot[name path=gpt2c_fac_hi, draw=none, forget plot]
    table[x=depth, y=hi] {content/figures/g_temporal_isolation/data/g_depth_gpt2_factual.dat};
\addplot[name path=gpt2c_fac_lo, draw=none, forget plot]
    table[x=depth, y=lo] {content/figures/g_temporal_isolation/data/g_depth_gpt2_factual.dat};
\addplot[fill=factual, fill opacity=0.07, forget plot]
    fill between[of=gpt2c_fac_hi and gpt2c_fac_lo];
\addplot[factual, thick, dashed, forget plot]
    table[x=depth, y=mean] {content/figures/g_temporal_isolation/data/g_depth_gpt2_factual.dat};

\addplot[name path=gpt2c_hal_hi, draw=none, forget plot]
    table[x=depth, y=hi] {content/figures/g_temporal_isolation/data/g_depth_gpt2_hallucinated.dat};
\addplot[name path=gpt2c_hal_lo, draw=none, forget plot]
    table[x=depth, y=lo] {content/figures/g_temporal_isolation/data/g_depth_gpt2_hallucinated.dat};
\addplot[fill=hallucinate, fill opacity=0.07, forget plot]
    fill between[of=gpt2c_hal_hi and gpt2c_hal_lo];
\addplot[hallucinate, thick, dashed, forget plot]
    table[x=depth, y=mean] {content/figures/g_temporal_isolation/data/g_depth_gpt2_hallucinated.dat};

\draw[dashed, black!40, line width=0.6pt] (axis cs:-3,0.5) -- (axis cs:103,0.5);
\node[anchor=south west, font=\tiny, text=black!50] at (axis cs:1,0.505) {$G{=}0.5$};

\end{axis}

\coordinate (legendanchor) at ($(panelb.below south east)!0.5!(panelc.below south west)+(0,-6pt)$);
\node[anchor=north, font=\scriptsize] at (legendanchor) {%
    \tikz[baseline=-0.5ex]{%
        \draw[factual, thick] (0,0) -- (0.4,0); \node[right, inner sep=1.5pt] at (0.4,0) {Factual};
        \draw[hallucinate, thick] (2.0,0) -- (2.4,0); \node[right, inner sep=1.5pt] at (2.4,0) {Halluc.};
        \draw[black!50, thick, dashed] (4.0,0) -- (4.4,0); \node[right, inner sep=1.5pt] at (4.4,0) {GPT-2 (dashed)};
    }%
};
\end{tikzpicture}%
}

\caption{\textbf{Temporal isolation ($G$) is sparse but architecture- and position-encoding-dependent.}
\textbf{(a)}~LC-AUROC for $G$ std across all 15 model-dataset combinations with 95\%
bootstrap CIs.  Most configurations cluster near chance; Flan-T5 decoder/HaluEval
(0.78) and Pythia/HaluEval (0.82) are exceptions.
\textbf{(b)}~Per-layer $G$ profiles on HaluEval (mean $\pm\,1$ std ribbons).
Flan-T5 decoder (solid): clear class separation; factual $G$ dips below
0.5 in middle layers, consistent with decoder self-attention relaxing toward
symmetric transport when cross-attention grounding succeeds.
GPT-2 (dashed): ribbons overlap (mean gap ${<}\,0.01$), the null case.
\textbf{(c)}~Pythia-160M (solid) shows \emph{reversed polarity}: factual $G >$
hallucinated $G$ (mean gap $-0.025$), concentrated in middle-to-late layers.
This is consistent with a RoPE positional decay hypothesis: rotary embeddings
impose a structured recency bias that factual generation must override to
attend to earlier context, while hallucinated generation defaults to the
decay pattern. GPT-2 (dashed, repeated as null reference) uses learned absolute
positions and shows no such separation.}
\label{fig:g-barchart}
\end{figure*}
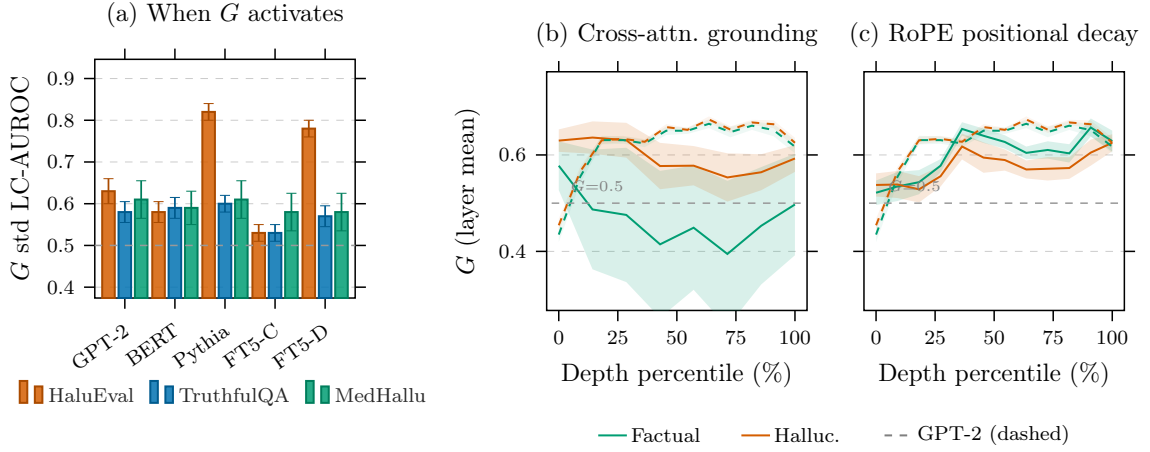

The asymmetric coefficient $G$ shows weak discrimination for most configurations
(0.53--0.63 LC-AUROC; Figure~\ref{fig:g-barchart}), indicating that temporal
isolation is \emph{not} the dominant failure mode across the architectures studied.
Two exceptions on HaluEval stand out: Flan-T5 decoder self-attention (0.78) and
Pythia-160M (0.82).  We read the accounts that follow as explanatory hypotheses
consistent with the patterns, not as validated mechanisms.

The orientation blindness theorem (\cref{thm:orientation-blindness}) is established by proof; the experiments characterize the conditions under which the antisymmetric axis carries discriminative signal.  They do not test the theorem, which holds regardless of whether temporal isolation is currently active.

For Flan-T5, an interpretable hypothesis: decoder self-attention is informed by
cross-attention to encoder representations, so its $G$ may track whether grounding
succeeds.  Factual $G$ dips below $0.5$ in middle layers (Figure~\ref{fig:g-barchart}b),
consistent with relaxation toward symmetric transport when context is adequate; the
separation is absent in GPT-2 (gap ${<}\,0.01$), which lacks that pathway
(\cref{prop:asymmetry-energy}).

Pythia's high $G$ on HaluEval ($0.82$ [$0.80$, $0.84$]) does not fit this account, since
it is decoder-only.  Instead its polarity is \emph{reversed}: factual samples show higher
$G$ than hallucinated (gap $-0.025$, layers 4--10; Figure~\ref{fig:g-barchart}c).  A
position-encoding hypothesis is consistent: RoPE~\citep{Su2024RoFormer} induces a recency
bias that factual retrieval must override, raising temporal directionality, whereas
GPT-2's learned absolute embeddings show no such separation.  Both accounts are
unconfirmed and dataset-specific (Pythia's $G$ is near chance on TruthfulQA $0.60$ and
MedHallu $0.61$); a response-length explanation does not apply, since HaluEval responses
are the shortest of the three benchmarks (median $7$ tokens vs.\ $10$ and $31$).

\begin{table*}[t]
\centering
\caption{\textbf{Feature-set ablation: does the asymmetry axis $G$ add
  discriminative signal beyond the symmetric axis?} Length-controlled AUROC
  (LC-AUROC) with 95\% bootstrap confidence intervals (super/subscript
  notation) of logistic-regression probes fit on z-scored
  feature sets via 5-fold cross-validation (out-of-fold scores; same supervised
  convention as the EigenTrack baseline), using the main-table aggregations
  ($\phihat$ CVaR$_{75}$, $\phihat$ CVaR$_{25}$, $\sigma_2$ std, $G$ std).
  $\Delta$ is the paired difference (full minus symmetric); $\Delta > 0$ means
  adding $G$ improves discrimination, with a paired-bootstrap CI on the difference.
  \textbf{Bold} $\Delta$: CI excludes zero in favor of $G$;
  \emph{italic} $\Delta$: CI excludes zero against $G$ (reversal).
  Llama-3.1-8B rows are the modern-model generalization check, not part of the
  main results grid.}
\label{tab:g-ablation}
\footnotesize
\begin{filecontents*}{data/tables/g_ablation.csv}
# Feature-set G-ablation table (LaTeX-ready, super/subscript CIs)
# Generated by: scripts/collect_ablation_macros.py -- DO NOT EDIT BY HAND
# Source: data/ablation/feature_set_ablation.csv
# Run date: 2026-06-10; bold = CI excludes zero in favor of G,
# italic = CI excludes zero against G; Llama rows = generalization check
Model,Dataset,n,sym,full,gonly,delta
GPT-2,HaluEval,5904,0.750$_{\scriptscriptstyle .726}^{\scriptscriptstyle .780}$,0.750$_{\scriptscriptstyle .725}^{\scriptscriptstyle .779}$,0.627$_{\scriptscriptstyle .603}^{\scriptscriptstyle .656}$,-0.000$_{\scriptscriptstyle -.006}^{\scriptscriptstyle +.005}$
,TruthfulQA,1963,0.594$_{\scriptscriptstyle .572}^{\scriptscriptstyle .621}$,0.604$_{\scriptscriptstyle .580}^{\scriptscriptstyle .631}$,0.581$_{\scriptscriptstyle .558}^{\scriptscriptstyle .606}$,+0.010$_{\scriptscriptstyle -.002}^{\scriptscriptstyle +.021}$
,MedHallu,618,0.588$_{\scriptscriptstyle .549}^{\scriptscriptstyle .642}$,0.521$_{\scriptscriptstyle .519}^{\scriptscriptstyle .588}$,0.609$_{\scriptscriptstyle .562}^{\scriptscriptstyle .658}$,-0.068$_{\scriptscriptstyle -.094}^{\scriptscriptstyle +.012}$
Pythia,HaluEval,5904,0.828$_{\scriptscriptstyle .802}^{\scriptscriptstyle .856}$,0.837$_{\scriptscriptstyle .813}^{\scriptscriptstyle .864}$,0.816$_{\scriptscriptstyle .790}^{\scriptscriptstyle .841}$,\textbf{+0.009$_{\scriptscriptstyle +.002}^{\scriptscriptstyle +.016}$}
,TruthfulQA,1963,0.543$_{\scriptscriptstyle .527}^{\scriptscriptstyle .570}$,0.587$_{\scriptscriptstyle .562}^{\scriptscriptstyle .611}$,0.594$_{\scriptscriptstyle .572}^{\scriptscriptstyle .621}$,\textbf{+0.044$_{\scriptscriptstyle +.013}^{\scriptscriptstyle +.066}$}
,MedHallu,618,0.530$_{\scriptscriptstyle .519}^{\scriptscriptstyle .593}$,0.618$_{\scriptscriptstyle .578}^{\scriptscriptstyle .672}$,0.605$_{\scriptscriptstyle .571}^{\scriptscriptstyle .661}$,\textbf{+0.089$_{\scriptscriptstyle +.013}^{\scriptscriptstyle +.126}$}
BERT,HaluEval,5904,0.547$_{\scriptscriptstyle .527}^{\scriptscriptstyle .581}$,0.569$_{\scriptscriptstyle .542}^{\scriptscriptstyle .600}$,0.578$_{\scriptscriptstyle .552}^{\scriptscriptstyle .608}$,+0.022$_{\scriptscriptstyle -.028}^{\scriptscriptstyle +.061}$
,TruthfulQA,1963,0.538$_{\scriptscriptstyle .522}^{\scriptscriptstyle .565}$,0.581$_{\scriptscriptstyle .558}^{\scriptscriptstyle .608}$,0.592$_{\scriptscriptstyle .571}^{\scriptscriptstyle .618}$,\textbf{+0.044$_{\scriptscriptstyle +.011}^{\scriptscriptstyle +.067}$}
,MedHallu,618,0.566$_{\scriptscriptstyle .539}^{\scriptscriptstyle .624}$,0.576$_{\scriptscriptstyle .544}^{\scriptscriptstyle .633}$,0.592$_{\scriptscriptstyle .557}^{\scriptscriptstyle .645}$,+0.010$_{\scriptscriptstyle -.011}^{\scriptscriptstyle +.024}$
Flan-T5 (cross),HaluEval,5904,0.723$_{\scriptscriptstyle .696}^{\scriptscriptstyle .748}$,0.722$_{\scriptscriptstyle .695}^{\scriptscriptstyle .748}$,0.535$_{\scriptscriptstyle .521}^{\scriptscriptstyle .565}$,-0.001$_{\scriptscriptstyle -.004}^{\scriptscriptstyle +.004}$
,TruthfulQA,1951,0.538$_{\scriptscriptstyle .520}^{\scriptscriptstyle .566}$,0.562$_{\scriptscriptstyle .540}^{\scriptscriptstyle .588}$,0.525$_{\scriptscriptstyle .517}^{\scriptscriptstyle .557}$,\textbf{+0.024$_{\scriptscriptstyle +.000}^{\scriptscriptstyle +.048}$}
,MedHallu,618,0.679$_{\scriptscriptstyle .630}^{\scriptscriptstyle .724}$,0.686$_{\scriptscriptstyle .636}^{\scriptscriptstyle .730}$,0.581$_{\scriptscriptstyle .549}^{\scriptscriptstyle .640}$,+0.007$_{\scriptscriptstyle -.005}^{\scriptscriptstyle +.020}$
Flan-T5 (dec),HaluEval,5904,0.768$_{\scriptscriptstyle .745}^{\scriptscriptstyle .793}$,0.756$_{\scriptscriptstyle .731}^{\scriptscriptstyle .781}$,0.779$_{\scriptscriptstyle .756}^{\scriptscriptstyle .800}$,\emph{-0.012$_{\scriptscriptstyle -.017}^{\scriptscriptstyle -.008}$}
,TruthfulQA,1951,0.550$_{\scriptscriptstyle .531}^{\scriptscriptstyle .578}$,0.562$_{\scriptscriptstyle .538}^{\scriptscriptstyle .590}$,0.563$_{\scriptscriptstyle .537}^{\scriptscriptstyle .590}$,+0.012$_{\scriptscriptstyle -.011}^{\scriptscriptstyle +.031}$
,MedHallu,618,0.546$_{\scriptscriptstyle .526}^{\scriptscriptstyle .598}$,0.541$_{\scriptscriptstyle .523}^{\scriptscriptstyle .592}$,0.582$_{\scriptscriptstyle .551}^{\scriptscriptstyle .635}$,-0.005$_{\scriptscriptstyle -.035}^{\scriptscriptstyle +.027}$
Llama-3.1-8B,HaluEval,5904,0.721$_{\scriptscriptstyle .696}^{\scriptscriptstyle .747}$,0.745$_{\scriptscriptstyle .721}^{\scriptscriptstyle .770}$,0.539$_{\scriptscriptstyle .521}^{\scriptscriptstyle .566}$,\textbf{+0.025$_{\scriptscriptstyle +.014}^{\scriptscriptstyle +.034}$}
,TruthfulQA,1963,0.609$_{\scriptscriptstyle .585}^{\scriptscriptstyle .635}$,0.608$_{\scriptscriptstyle .584}^{\scriptscriptstyle .635}$,0.541$_{\scriptscriptstyle .525}^{\scriptscriptstyle .567}$,-0.001$_{\scriptscriptstyle -.008}^{\scriptscriptstyle +.006}$
,MedHallu,618,0.547$_{\scriptscriptstyle .530}^{\scriptscriptstyle .606}$,0.538$_{\scriptscriptstyle .524}^{\scriptscriptstyle .595}$,0.588$_{\scriptscriptstyle .546}^{\scriptscriptstyle .640}$,-0.009$_{\scriptscriptstyle -.037}^{\scriptscriptstyle +.019}$
\end{filecontents*}
\pgfplotstabletypeset[
  col sep=comma,
  string type,
  skip first n=5,  
  columns={Model, Dataset, n, sym, full, gonly, delta},
  every head row/.style={before row=\toprule, after row=\midrule},
  every last row/.style={after row=\bottomrule},
  every row no 15/.style={before row=\midrule},  
  columns/Model/.style={column name=Model, column type=l},
  columns/Dataset/.style={column name=Data set, column type=l},
  columns/n/.style={column name=$n$, column type=r},
  columns/sym/.style={column name={$\{\phihat, \sigma_2\}$}, column type=c},
  columns/full/.style={column name={$\{\phihat, \sigma_2, G\}$}, column type=c},
  columns/gonly/.style={column name={$\{G\}$}, column type=c},
  columns/delta/.style={column name={$\Delta$ (full $-$ sym)}, column type=c},
]{data/tables/g_ablation.csv}
\end{table*}

\paragraph{Feature-set ablation: the antisymmetric axis adds signal where
directional failure modes are active.}
The orientation-blindness theorem guarantees that symmetric functionals
cannot access $\Masym$; it does not by itself imply that the inaccessible
component carries usable signal.  Table~\ref{tab:g-ablation} tests this
directly, comparing logistic probes fit on the symmetric feature set
$\{\phihat\,\text{CVaR}_{75}, \phihat\,\text{CVaR}_{25}, \sigma_2\,\text{std}\}$
against the same set augmented with $G$ std (5-fold out-of-fold scoring,
the same supervised convention as the EigenTrack baseline; LC-AUROC
throughout).  Adding $G$ yields CI-significant gains in
\gabNumSigPositive{} of \gabNumCells{} main-grid cells spanning three
architectures, largest on \gabMaxDeltaCell{}
($\Delta = \gabMaxDelta$ [\gabMaxDeltaLo, \gabMaxDeltaHi]); notably,
in the Llama-3.1-8B/HaluEval generalization cell $G$ alone is near chance
(\gabllamaHaluevalGonly{}) yet still contributes
$\Delta = \gabllamaHaluevalDelta$ [\gabllamaHaluevalDeltaLo,
\gabllamaHaluevalDeltaHi] to the pooled probe: orientation information
the symmetric axis does not carry.  Cross-attention cells are neutral, as
expected: $G$ is unreliable for rectangular attention matrices
(\suppref{app:prompt-response-ratio-confounding}).  One reversal occurs
(Flan-T5 decoder on HaluEval,
$\Delta = \gabflantFiveDecoderHaluevalDelta$
[\gabflantFiveDecoderHaluevalDeltaLo,
\gabflantFiveDecoderHaluevalDeltaHi]): there $G$ alone is already strong
(\gabflantFiveDecoderHaluevalGonly{}), and the two axes carry overlapping
rather than complementary signal in that configuration.  The mean paired
difference across the main grid is $\gabMeanDelta$.  The pattern matches
the two-axis reading of the theory: the antisymmetric axis adds
discriminative signal precisely where directional failure modes are
active, and is inert where the theory predicts no orientation signal
exists.

\subsection{Q3: What Does Each Method Class Measure?}
\label{sec:eval-q3}

\paragraph{Different diagnostic objects have structurally different
length exposure.}
Length confounding inflates raw AUROC by up to \confoundMaxDelta{} points
(compare with Table~\ref{tab:main-results}).  The confounding structure is
informative: spectral methods show substantial exposure
($\Delta{=}{+}\confoundOCPythiaHal$ for OC on Pythia/HaluEval), hidden-state features
show moderate exposure ($\Delta \approx \confoundHiddenMeanDelta$), and output entropy
shows near-zero exposure ($\Delta{\approx}\confoundLogitMeanDelta$), as predicted by
the intensive/extensive distinction (\cref{sec:length-confound}).
Raw and per-quartile AUROC breakdowns appear in the Online Supplement (\suppref{suppl:detailed-auroc}, \suppref{app:stratified-auroc}).

\paragraph{Hidden-state probes.}
Hidden-state features detect \emph{that} generation is unreliable
without distinguishing \emph{how}: they do not recover the
bottleneck--diffuse--knowledge-gap taxonomy that the transport
framework provides.  Across 7 of 11 applicable configurations,
LLM-Check hidden-state features achieve best or tied-best LC-AUROC
(0.62--0.72), making them the strongest general-purpose discrimination
signal, though one that is structurally agnostic to failure mode.
Hidden features show substantial length confounding
($\Delta \approx \confoundHiddenMeanDelta$); output entropy is length-robust
($\Delta \approx \confoundLogitMeanDelta$) but more variable (0.52--0.80 LC-AUROC).

\paragraph{Learned classifiers.}
Supervised spectral classifiers access the same diagnostic object
(attention spectra) but through a learned lens, making them vulnerable
to dataset-specific overfitting.  EigenTrack achieves marginal signal
on HaluEval for small models (GPT-2: 0.63, BERT: 0.62) but
near-chance on TruthfulQA and MedHallu across most architectures (see
Online Supplement, \suppref{app:eigentrack}), consistent with overfitting to
dataset-specific spectral signatures, though short-sequence effects
cannot be excluded.  Transport diagnostics (OC, $G$) are zero-shot
and provide interpretable failure mode characterization independent of
training data size or dataset-specific tuning.

\paragraph{TruthfulQA as stress test.}
TruthfulQA induces near-uniform routing statistics across all methods
(0.50--0.57), with hallucinations distributing uniformly across OC terciles
($\approx 33\%$ each).  Short responses (median 10 tokens) limit spectral
observables, so the null hypothesis ``spectral methods fail on short
sequences'' cannot be cleanly distinguished from ``spectral methods correctly
return null for non-routing failures.''  We interpret the near-chance result
as consistent with specificity (knowledge gaps produce no detectable
routing pathology) while acknowledging that formal validation requires a
data set with no length-label correlation, sufficiently long responses, and
known knowledge-gap failures.

\subsection{Transfer to the Modern Panel}
\label{sec:llama-generalization}

The mechanistic grid maximizes measurement depth; the modern panel tests the
diagnostics on the models practitioners use.  We evaluate (i)~modern
decoder-only models in current use: Qwen2.5 (0.5B, 1.5B, 3B) and
SmolLM2-1.7B; and (ii)~LLaMA-3.1-8B (grouped-query attention), roughly an order
of magnitude larger than the mechanistic-grid decoders.  All use the same
canonical spectral-sweep estimator and length-controlled protocol; we report
transport diagnostics only (the LLM-Check and EigenTrack baselines were not
re-run on these models).

\begin{figure*}[tbp]
\centering
\resizebox{\textwidth}{!}{%
\begin{tikzpicture}
\begin{axis}[
    academic-single-bar,
    name=modern,
    scale only axis,
    width=0.34\textwidth,
    height=0.24\textwidth,
    ybar=1.2pt,
    bar width=6pt,
    ylabel={LC-AUROC (HaluEval)},
    ymin=0.40, ymax=0.92,
    ytick={0.4,0.5,0.6,0.7,0.8,0.9},
    ymajorgrids=true, xmajorgrids=false,
    symbolic x coords={Qwen2.5-0.5B,Qwen2.5-1.5B,Qwen2.5-3B,SmolLM2-1.7B,LLaMA-3.1-8B},
    xtick=data,
    x tick label style={font=\scriptsize, rotate=35, anchor=north east},
    enlarge x limits=0.18,
    every axis plot/.append style={fill opacity=0.85},
    title={\small (a) Modern decoder-only models},
    title style={at={(0.5,1.04)}, anchor=south},
    every error bar/.style={line width=0.6pt},
]
\addplot[fill=spectral, draw=spectral!80!black,
    error bars/y dir=both, error bars/y explicit,
    error bars/error bar style={spectral!80!black}]
    table[x=model, y=value, y error=err] {content/figures/model_panel/data/modern_phi.dat};
\addplot[fill=factual, draw=factual!80!black,
    error bars/y dir=both, error bars/y explicit,
    error bars/error bar style={factual!80!black}]
    table[x=model, y=value, y error=err] {content/figures/model_panel/data/modern_sigma2.dat};
\addplot[fill=hallucinate, draw=hallucinate!80!black,
    error bars/y dir=both, error bars/y explicit,
    error bars/error bar style={hallucinate!80!black}]
    table[x=model, y=value, y error=err] {content/figures/model_panel/data/modern_g.dat};
\draw[dashed, black!40, line width=0.6pt]
    ([xshift=-14pt]axis cs:Qwen2.5-0.5B,0.5) -- ([xshift=14pt]axis cs:LLaMA-3.1-8B,0.5);
\node[anchor=south east, font=\tiny, text=black!50]
    at ([xshift=14pt,yshift=1pt]axis cs:LLaMA-3.1-8B,0.5) {chance};
\end{axis}

\begin{axis}[
    academic-single-bar,
    at={(modern.east)}, anchor=west, xshift=0.35cm,
    name=pythia,
    scale only axis,
    width=0.34\textwidth,
    height=0.24\textwidth,
    ybar=1.2pt,
    bar width=6pt,
    ylabel={},
    yticklabel=\empty,
    ymin=0.40, ymax=0.92,
    ytick={0.4,0.5,0.6,0.7,0.8,0.9},
    ymajorgrids=true, xmajorgrids=false,
    symbolic x coords={70M,160M,410M,1B,1.4B},
    xtick=data,
    x tick label style={font=\scriptsize},
    xlabel={Pythia size},
    enlarge x limits=0.14,
    every axis plot/.append style={fill opacity=0.85},
    title={\small (b) Within-family size sweep (Pythia)},
    title style={at={(0.5,1.04)}, anchor=south},
    every error bar/.style={line width=0.6pt},
]
\addplot[fill=spectral, draw=spectral!80!black,
    error bars/y dir=both, error bars/y explicit,
    error bars/error bar style={spectral!80!black}]
    table[x=model, y=value, y error=err] {content/figures/model_panel/data/pythia_phi.dat};
\addplot[fill=factual, draw=factual!80!black,
    error bars/y dir=both, error bars/y explicit,
    error bars/error bar style={factual!80!black}]
    table[x=model, y=value, y error=err] {content/figures/model_panel/data/pythia_sigma2.dat};
\addplot[fill=hallucinate, draw=hallucinate!80!black,
    error bars/y dir=both, error bars/y explicit,
    error bars/error bar style={hallucinate!80!black}]
    table[x=model, y=value, y error=err] {content/figures/model_panel/data/pythia_g.dat};
\draw[dashed, black!40, line width=0.6pt]
    ([xshift=-12pt]axis cs:70M,0.5) -- ([xshift=12pt]axis cs:1.4B,0.5);
\end{axis}

\coordinate (legmid) at
    ($(modern.below south east)!0.5!(pythia.below south west)+(0,-10pt)$);
\node[anchor=north, font=\scriptsize] at (legmid) {%
    \tikz[baseline=-0.4ex]{\fill[spectral,fill opacity=0.85](0,0)rectangle(1.3ex,1.3ex);}\,$\phihat$ CVaR$_{75}$\quad
    \tikz[baseline=-0.4ex]{\fill[factual,fill opacity=0.85](0,0)rectangle(1.3ex,1.3ex);}\,$\sigma_2$ std\,\;(capacity)\qquad
    \tikz[baseline=-0.4ex]{\fill[hallucinate,fill opacity=0.85](0,0)rectangle(1.3ex,1.3ex);}\,$G$ std (direction)%
};
\end{tikzpicture}%
}
\caption{\textbf{Transport diagnostics across modern models and a within-family
size sweep (HaluEval LC-AUROC, 95\% bootstrap CIs).}  Bars are discrete per
model, not a continuous trend.  \textbf{(a)}~On the modern panel the capacity
axis ($\phihat$ CVaR$_{75}$, $\sigma_2$ std; cool bars) is high and stable across
Qwen2.5 $0.5$--$3$B, SmolLM2-1.7B, and LLaMA-3.1-8B (spanning $0.5$--$8$B), while
the asymmetry axis $G$ (vermillion) stays near chance.  \textbf{(b)}~Within the Pythia family the capacity axis is
non-monotonic (peaking at 160M), and $G$ is elevated only at 160M---a
checkpoint-specific spike that does not recur on the modern models.  Both panels
share the $y$-axis.  Exact values and the other benchmarks appear in the Online
Supplement.}
\label{fig:model-panel-bars}
\end{figure*}
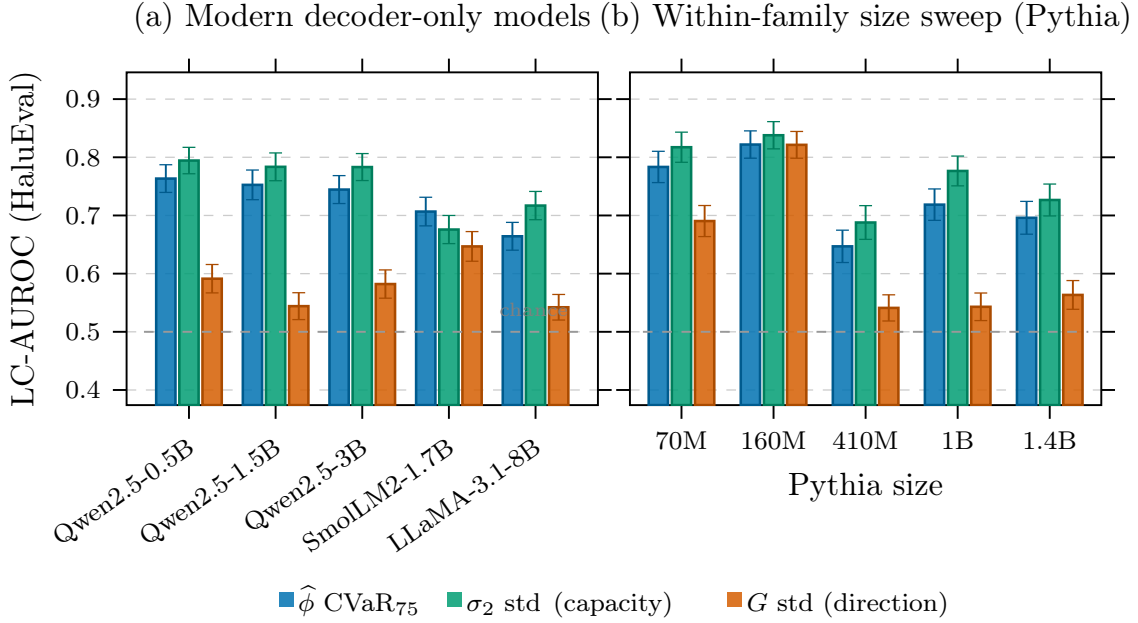

The capacity axis transfers cleanly across the modern panel: on HaluEval
$\sigma_2$~std and $\phihat$~CVaR$_{75}$ sit in a stable $0.66$--$0.79$ band across
Qwen2.5 at $0.5$/$1.5$/$3$B, SmolLM2-1.7B, and LLaMA-3.1-8B
(\cref{fig:model-panel-bars}a), with no collapse out to 8B.  The asymmetry axis
$G$ stays near chance ($0.54$--$0.65$) throughout, mirroring the sparse-$G$ pattern
of the mechanistic-grid decoders; notably the elevated $G$ of Pythia-160M does \emph{not}
recur on the modern models, so that spike is checkpoint-specific rather than a
general RoPE-decoder property.  As elsewhere, signal is strongest on HaluEval and
weaker on the less length-confounded TruthfulQA and MedHallu.  We therefore read
the diagnostics as retaining interpretable signal across model families and
scales, robustly so on the modern panel, rather than improving monotonically with
size; a within-family Pythia sweep (\cref{fig:model-panel-bars}b) shows the same
capacity signal is itself non-monotonic in scale.  Exact per-benchmark values for
the full panel (all models, all benchmarks) are in the Online Supplement
(\suppref{sec:model-panel}).

\subsection{Summary}
\label{sec:eval-summary}
In short: conductance ($\phihat$, $\sigma_2$) captures capacity failures with strong
LC-AUROC (0.62--0.84) and CVaR-tail failure-mode characterization; $G$ is empirically
sparse (0.53--0.63) outside two architecture-specific HaluEval exceptions; hidden-state
probes give the strongest general-purpose discrimination (0.62--0.72) but are
failure-mode-agnostic; the capacity axis transfers to the modern panel (Qwen2.5,
SmolLM2-1.7B, LLaMA-3.1-8B) in a stable 0.66--0.79 band with no collapse out to 8B;
and the near-chance TruthfulQA result is consistent with the
framework returning null for non-routing failures (finite-size effects at median 10 tokens
aside).  The contribution is not a single best-performing method but a framework that
characterizes \emph{how} attention routing fails.

\section{Discussion and Implications}
\label{sec:discussion}

The paper's organizing result is an identifiability limit: symmetric spectral
diagnostics of an attention operator are orientation-blind
(\cref{thm:orientation-blindness,prop:ob-converse}), so distinguishing
routing-failure shapes (bottleneck versus diffuse) is a question about which
mathematical object a diagnostic analyses, not which statistic is extracted from
it.  Everything else follows.  The symmetric--antisymmetric partition is a
consequence of Hilbert--Schmidt orthogonality, not a design choice; conductance is
the strongest quantity the surviving symmetric axis supports; the closed-form
bipartite-Cheeger landscape
(\cref{lem:uniform-causal-conductance,cor:uc-one-fifth,cor:uc-sharp-floor}) says
what it means architecturally; and the empirics characterise where each axis
carries signal.  The resulting two-axis diagnostic $(\phi, G)$ yields a falsifiable
polarity prediction (bottleneck routing on HaluEval, diffuse on MedHallu),
borne out in sign under length-controlled evaluation across the tested decoder-only,
encoder-only, and encoder--decoder models.

\paragraph{What the analysis adds over a single statistic.}
Beyond a single SVD statistic, the framework (i) gives a structural mechanism for
the cross-dataset polarity prediction via the two-sided conductance theory; (ii)
separates bottleneck from diffuse failures through CVaR tail analysis; (iii)
explains the $\sigma_2$/$\phihat$ ordering through the Cheeger inequality
$\phi^2/2\le 1-\sigma_2\le 2\phi$; (iv) shows via degree-preserving nulls that
$\phihat$ is largely a degree proxy while $\sigma_2$ retains coupling signal; and
(v) turns the existence of a conductance floor into a per-cut closed-form
landscape (\cref{lem:uniform-causal-conductance}) against which empirical
$\phihat$ reads position-encoding regimes
(\cref{tab:landscape-empirical-summary}).  Conductance $\phi$ is
architecture-universal (Cheeger holds for any weighted graph); $G$ is
architecture-dependent and, in the decoder-only models studied, sparse
(0.53--0.63 LC-AUROC); attention asymmetry there is dominated by positional
structure (causal mask, attention sinks) rather than content-dependent
directional routing.  By the orientation-blindness theorem no symmetric method
can access that axis, so its sparsity reflects the architectures, not the
diagnostic.

\paragraph{Limitations.}
Diagnostics are computed under forced scoring of benchmark-provided responses in a
single forward pass (\cref{sec:eval-protocol}); the results are associational and
concern co-occurrence with the labelled response class after length control, not
generation dynamics.  Matched nulls separate estimator artifacts from learned
structure but do not exclude third-variable confounds.  The diagnostics require
attention tensors (so are inapplicable to API-only models) and are blind to value
geometry and layer normalization; diffuse-but-ungrounded routing and
value-projection errors can yield hallucinations with no routing pathology.
``Zero-shot'' refers to feature computation; calibrating decision polarity needs a
modest set (50--100 examples), and OC polarity is bin-consistent in only about
half (7 of 12) of model--dataset pairings, a finite-size effect (estimator
variance correlates with response length, Spearman $r=-0.55$).  This calibration
requirement is structural, not incidental: recovering orientation is \emph{necessary
but not sufficient} for a regime-stable diagnostic.  A signed statistic resolves the
orientation that \cref{cor:transpose-rigidity} shows no transpose-invariant functional
can; yet the sign of the orientation-to-failure association is itself a regime-dependent,
unidentified parameter---the cross-benchmark polarity reversal between HaluEval and
MedHallu---so decision polarity must be fixed per regime.  This is a second identifiability
gap, stacked on the transpose-invariance limit of \cref{thm:orientation-blindness}: the
first bounds what a symmetric functional can read, the second bounds what even a signed one
delivers without calibration.  Coverage is partial: the diagnostics transfer to LLaMA-3.1-8B (GQA, RoPE;
\cref{sec:llama-generalization}), but sliding-window and mixture-of-experts
routing are untested.

\paragraph{Future directions.}
The clearest next step is to recover the orientation $G$ discards.  Since $G$ is a
\emph{magnitude} certificate (transpose-invariant), the sign of $\Masym$ is read only by
a genuinely signed statistic: aligning $\M-\M^\top$ with an orientation template
separates forward from reverse routing exactly where $G$ and every transpose-invariant
functional cannot.  A synthetic check and an exploratory cross-architecture evaluation
(\suppref{app:signed-directionality}) confirm such a signed score recovers the discarded
orientation and carries complementary signal where content-directional routing is strong
(though not as a uniform detector).  Lifting this to the spectral level (a directional
theory that retains $\Masym$ in the bipartite dilation $\Hh(\M)$ while keeping the
$\phi$-side analysis intact, as magnetic-Laplacian and directed/Hodge Cheeger
inequalities~\citep{LangeLiuPeyerimhoffPost2015MagneticCheeger,chung2005laplacians} do
for orientation alone) remains open.  Separately, the cut/volume template
(\cref{prop:uc-cut-vol-identities,prop:dilation-total-vol}) extends to other masks
(window, exponential, RoPE) and to mixture-of-experts dispatch as a second bipartite
routing operator, suggesting a taxonomy of attention by conductance-landscape geometry.

\medskip\noindent
The orientation-blindness theorem and the closed-form Cheeger landscape together
establish that attention diagnostics partition by which axis they access, and that
partition has architectural consequences empirical conductance signatures can read.


\section{Related Work}
\label{sec:related-work}

We conclude by situating the transport framework within the broader landscape of hallucination detection methods.
The closest prior work (Lookback Lens~\citep{Chuang2024LookbackLens} and
LapEigvals~\citep{Binkowski2025LapEigvals}) also analyzes attention, but without
degree normalization (introducing length dependence) or orthogonal decomposition
(conflating capacity and orientation failures). LLM-Check~\citep{Sriramanan2024LLMCheck}
achieves competitive detection through hidden states but does not structurally distinguish routing failure modes.
The transport framework's contribution is not ``better detection'' but \emph{diagnostic
decomposition}: partitioning routing failures into interpretable, orthogonal axes with
formal guarantees (Cheeger inequality, orientation blindness theorem) that explain what
each axis can and cannot detect.

We organize the landscape by \emph{diagnostic object} (the representation from which the
signal is derived), progressing from output-level (cheapest access) to transport-level
(richest structure): methods with deeper access can distinguish \emph{how} generation fails,
lighter ones only \emph{that} it is unreliable.

\subsection{Output-Level Methods}

Output-level methods require only generation samples or token probabilities, making them
applicable to black-box APIs.

\paragraph{Self-consistency.}
SelfCheckGPT~\citep{Manakul2023SelfCheckGPT} generates multiple responses to the same
prompt and uses inter-sample variance (measured via BERTScore, QA, or n-gram overlap) as
a hallucination proxy. This approach is principled for \emph{confabulations} (arbitrary
incorrect generations) but misses \emph{consistent} errors where the model reliably
produces the same wrong answer. It also inherits the cost of multiple forward passes.

\paragraph{Semantic entropy.}
\citet{farquhar2024probes} measure uncertainty at the \emph{meaning} level by clustering
sampled responses into semantic equivalence classes and computing entropy over these clusters.
This addresses a fundamental limitation of token-level entropy: paraphrases of the same
correct answer inflate token entropy without indicating unreliability. Semantic entropy
achieves strong detection across tasks and generalizes to unseen prompts. However, it
requires multiple generation passes and semantic similarity computation, and cannot
distinguish failure \emph{modes}: high semantic entropy indicates uncertainty, not the
structural cause (routing failure vs.\ knowledge gap vs.\ calibration error).

\subsection{External Verification}

\paragraph{Factual decomposition.}
FActScore~\citep{Min2023FActScore} decomposes long-form generations into atomic factual
claims and verifies each against Wikipedia, achieving high precision on biographical
generation. SAFE~\citep{Wei2024SAFE} extends this to search-augmented verification using
language models as judges. These methods achieve high precision on factual claims but require
external knowledge bases, scale poorly to real-time detection, and cannot assess claims
outside the knowledge base's coverage. Importantly, they evaluate \emph{factual accuracy}
rather than \emph{generation process integrity}: a factually correct response generated
through pathological routing would pass verification, while our transport diagnostics would
flag the routing anomaly.

\subsection{Internal-State Methods}

Methods that probe hidden representations occupy an intermediate position: they access
richer structure than output-level methods but do not model the attention mechanism's
routing function.

\paragraph{Knowledge localization and hidden-state probes.}
Feed-forward layers store factual associations as key-value
memories~\citep{Geva2021Transformer,Meng2022Locating}, suggesting that some
hallucinations arise from retrieval failures in MLP parameters rather than
routing failures in attention, a distinction our framework makes explicit.

LLM-Check~\citep{Sriramanan2024LLMCheck} and INSIDE
(EigenScore)~\citep{Chen2024INSIDE} extract covariance spectra from
hidden-state activations, using eigenvalue statistics as hallucination
features. Both approaches capture population-level representational
change correlated with hallucination and achieve strong detection
performance (LLM-Check is best or tied-best in 7/11 model--dataset
configurations in our evaluation; INSIDE is structurally analogous to
LLM-Check's hidden branch (both read eigenvalues of a hidden-state covariance
matrix) and falls in the same orientation-blind
class by \cref{thm:orientation-blindness}). However, the signal
indicates \emph{that} generation is unreliable without distinguishing
\emph{how}: bottleneck routing, diffuse routing, and knowledge gaps
all produce similar representational change. Hidden-state features
also show substantial length confounding ($\Delta \approx \confoundHiddenMeanDelta$
between raw and length-controlled AUROC; \cref{sec:length-confound}).
The confound itself is established prior art: \citet{Santilli2025LengthBias}
formally prove that a shared length bias in the uncertainty score and the
correctness function non-randomly skews AUROC rankings, and
\citet{Janiak2025IllusionProgress} show that simple response-length heuristics
can match complex detectors under standard automatic evaluation; our length-controlled
protocol operationalizes the control these works call for.

\paragraph{Supervised spectral classifiers.}
EigenTrack~\citep{Ettori2025EigenTrack} streams covariance-spectrum statistics of
hidden activations into a learned (recurrent) classifier; because these statistics
are referenced to a Marchenko--Pastur baseline with aspect ratio $\gamma = D/n$,
they inherit irreducible length dependence (derivation in the Online Supplement,
\suppref{app:eigentrack}). In our evaluation, a linear-classifier port achieves
marginal signal on HaluEval for small models (GPT-2: 0.63, BERT: 0.62) but
near-chance on TruthfulQA and MedHallu.

\subsection{Attention-Based Methods}

Attention matrices encode the routing decisions that determine information flow through
the transformer. A growing body of work analyzes attention patterns directly.

\paragraph{Attention as grounding signal.}
Lookback Lens~\citep{Chuang2024LookbackLens} identifies contextual hallucinations
from the ratio of attention to context versus newly generated tokens, closest in
spirit to our work, since both read hallucination signal from attention.  Two
differences matter.  \textit{Structurally}, it reduces each attention map to a
single scalar, whereas the transport view treats the full operator $\M$ and its
$\Msym/\Masym$ decomposition; the lookback ratio is itself transpose-invariant and
falls in the orientation-blind class of \cref{thm:orientation-blindness}.
\textit{Methodologically}, it trains a supervised classifier on the lookback
features, placing it outside the zero-shot regime that frames our empirical
comparison; we therefore retain it as a related-work comparison rather than a
baseline.  Its orientation-blind prediction is testable independently of the
classifier and would manifest as failure on the $G$-dominant Pythia/HaluEval
regime.

\paragraph{Attention head analysis.}
\citet{Voita2019Analyzing} demonstrate that multi-head attention exhibits functional
specialization: a small number of heads perform critical functions (positional, syntactic,
rare-word), while the majority can be pruned without performance degradation. This head
specialization is relevant to our aggregation strategy: rather than pre-selecting
critical heads (which would require supervised identification), we aggregate
conductance across all heads using CVaR and robust statistics,
capturing the distributional structure of routing quality across the full set of heads.

\paragraph{Spectral attention phenomena.}
Attention entropy collapse during training~\citep{Zhai2023EntropyCollapse} (where attention
distributions sharpen to near-deterministic patterns) and doubly-exponential rank loss in
deep self-attention~\citep{Dong2021RankCollapse} are spectral phenomena that affect the
transport properties we measure. \citet{NaitSaada2025MindTheGap} provide a spectral analysis
of rank collapse and signal propagation in attention layers, relating the spectral gap
to how representations collapse in width. Our Cheeger inequality connection formalizes a complementary
aspect: the spectral gap of the degree-normalized operator bounds conductance, which in
turn bounds how quickly routing can mix information across token positions.

\paragraph{Self-attention as a transport operator.}
\citet{Geshkovski2023Emergence} analyse self-attention as a continuous
mean-field transport flow on the sphere and prove that token
representations cluster asymptotically. Their framing motivates the
transport view we adopt: where they study \emph{trajectories} of
tokens under iterated attention, we study the
\emph{instantaneous spectral structure} of a single normalised
attention map and ask which routing failures are detectable from it.
The two analyses are complementary; our orientation-blindness theorem
identifies a structural ceiling that any symmetric-spectral diagnostic
of an attention transport step must respect, regardless of whether it
is read off a single layer or aggregated across the dynamical
trajectory.
The operator-level reading of attention also has direct antecedents:
Sinkformers~\citep{Sander2021Sinkformers} interpret Sinkhorn-normalized
attention as a doubly stochastic transport plan between queries and keys,
and \citet{Erel2025AttentionMarkovChains} analyze the row-stochastic
attention matrix as a discrete-time Markov chain, deriving structural
properties of its steady-state distribution.
We adopt the same object but ask a different question: which routing
failures its spectral geometry can and cannot certify.
The transport geometry of attention has also been used directly for hallucination
detection: \citet{Guerreiro2023OTHallucination} detect hallucinations in machine
translation by the Wasserstein distance between a generation's cross-attention mass
distribution and reference distributions, an optimal-transport anomaly reading and
the origin of transport-on-attention hallucination detection.  Our use of transport
is structural rather than metric: we solve no transport problem and use no reference
distribution, reading capacity and orientation limits off the operator's spectral
geometry instead.

\paragraph{Attention asymmetry as a spectral object.}
Primal-Attention~\citep{Chen2023PrimalAttention} treats self-attention as an
asymmetric kernel and applies kernel SVD to it, the published antecedent for
taking the asymmetry of attention seriously as a spectral object.  The delta
is direction of use: they build a new attention mechanism on the asymmetric
representation, whereas we prove limits of symmetric functionals of the
existing mechanism and quantify the discarded component through $G$.
\citet{Saponati2025SelfAttentionStructures} make the symmetric/antisymmetric
structure of attention an object of study in its own right: they derive how
training objectives shape the decomposition, show that autoregressive training
induces directional, skew-dominated structure while bidirectional training
promotes symmetry, and define per-head symmetry and directionality scores,
validated across modern encoder and decoder families.  Their finding that the
antisymmetric component carries the directionality of information flow is the
empirical prior for our asymmetry axis.  What we add is the identifiability
limit: no transpose-invariant spectral diagnostic can recover that component's
sign (\cref{thm:orientation-blindness}), and its magnitude is exactly the
transpose-sensitivity capacity $\|\Masym\|_F$ (\cref{prop:ob-converse-tight}).

\paragraph{Spectral features from attention maps.}
\citet{Binkowski2025LapEigvals} extract graph-Laplacian eigenvalues ($L = D - A$)
from raw attention matrices as hallucination features.  This inherits length
dependence through the degree matrix $D$ (whose trace grows with sequence length)
and discards off-diagonal routing mass by Laplacian construction; our
degree-normalized operator $\M = D_Q^{-1/2} A D_K^{-1/2}$ absorbs the trace scaling
by construction.  We do not include LapEigvals as an empirical baseline, for two
reasons.  First, it is a supervised method: it trains a probe on attention-map
Laplacian eigenvalues and is evaluated under its own protocol, so a faithful port
requires reproducing the full training protocol (probe training, candidate-eigenvalue
subset, length-binning scheme), not a drop-in feature swap.  Second, and more
importantly, the orientation-blindness corollary
(\cref{cor:transpose-invariant-diagnostics}) applies to it structurally regardless
of empirical performance: its features are functions of symmetrized spectra, and the
Online Supplement (\suppref{app:lapeigvals}) derives this collapse explicitly.  The
structural argument is the claim this paper makes about LapEigvals; an empirical
port would test the method, not the theorem.

\paragraph{Topological attention-graph detectors.}
TOHA~\citep{Bazarova2025TOHA} scores responses by a topological divergence between
prompt and response subgraphs of the attention graph: training-free, like our
diagnostics, but developed for the retrieval-augmented setting with grounded
prompts, which our forced-scoring protocol does not model.  Its divergence is not a
spectral functional of $\M$ or $\Msym$, so it lies outside the class
\cref{thm:orientation-blindness} bounds; we cite it as complementary evidence that
structural readings of the attention graph carry hallucination signal, not as a
baseline within our regime.

\subsection{Graph-Theoretic and Transport Perspectives}

Our framework draws on spectral graph theory, connecting attention analysis to a mature
mathematical tradition.

\paragraph{Spectral graph theory foundations.}
The Cheeger inequality~\citep{cheegerLowerBoundSmallest2015,Chung1997} relates the spectral
gap of a graph's Laplacian to its conductance (minimum normalized cut). Originally developed
for Riemannian manifolds, its discrete analogue underpins rapidly-mixing Markov-chain
analysis~\citep{sinclairApproximateCountingUniform1989} and provides the theoretical backbone
of our conductance diagnostic: the two-sided bound $\phi^2/2 \leq 1 - \sigma_2 \leq 2\phi$
(\cref{eq:cheeger-inequality}) guarantees that spectral structure reveals transport failures.
Higher-order extensions~\citep{Lee2014HigherOrderCheeger,Kwok2013ImprovedCheeger} relate
higher eigenvalues to multi-way partitioning; we restrict to the second singular value as
it provides the most robust single diagnostic.

\paragraph{Spectral clustering and random walks.}
The mixing time--spectral gap connection~\citep{vonluxburgTutorialSpectralClustering2007,
Lovasz1996,LevinPeresWilmer2006} underlies our interpretation of conductance as a
transport capacity measure.  Our framework extends this to the non-reversible setting
of causal attention, where the antisymmetric component $\Masym$ captures the departure
from reversibility.

\paragraph{Directed graph spectra.}
\citet{chung2005laplacians} extends the Cheeger inequality to directed graphs, defining
a circulation-based notion of conductance for non-reversible Markov chains. This is
the natural setting for causal attention, where information flows forward in time.
\citet{Lau2023DirectedCheeger} give the current state-of-the-art directed (and
hypergraph) Cheeger inequalities via reweighted eigenvalues. Our aim is orthogonal:
not a tighter directed bound but a transport \emph{diagnostic} on attention operators,
for which the classical bipartite-dilation Cheeger machinery suffices; our contribution
is the orientation-blindness limit (\cref{thm:orientation-blindness}) and the $G$
complement, evaluated under length-corrected AUROC.
\citet{Meila2007DirectedClustering} develops weighted cuts for directed graphs, and
\citet{Fill1991EigenvalueBounds} provides eigenvalue bounds on convergence for non-reversible
chains. \citet{Cucuringu2020Hermitian} show that Hermitian (complex-valued) matrix
representations of directed graphs recover the orientation structure that symmetrization
destroys; we inherit this insight directly. What we contribute beyond it is the
attention-specific blindness partition, naming the affected detectors (LLM-Check,
EigenTrack, LapEigvals), together with the closed-form conductance landscape for canonical
causal architectures. Our asymmetric coefficient $G$ (\cref{sec:asymmetric-guessing}) measures the degree
of non-reversibility; the orientation blindness theorem (\cref{thm:orientation-blindness})
proves that this quantity is structurally invisible to symmetric spectral methods, which is
what motivates $G$ as a necessary complement to conductance.

\bibliography{references}

\clearpage
\appendix
\crefalias{section}{appendix}  

\noindent
\textbf{Roadmap.}  This appendix collects the longer proofs behind the paper's
headline structural results.  The orientation-blindness theorem and its Lipschitz
converse (\cref{thm:orientation-blindness}, \cref{prop:ob-converse}) and the
dilation-spectrum lemma (\cref{lem:dilation-spectrum}) are stated and proved in
place in the main text.  \cref{app:proofs_qk_regimes} gives the two-sided
conductance bounds with the $n$-independent $1/5$ floor, the causal
asymmetry-energy and quantitative $G$ bounds, and a restatement of degree
sufficiency; \cref{app:g-formal-verification} proves the asymmetry-coefficient
characterization (\cref{prop:g-extremes}: $G=0$ iff symmetric, the causal
$G\le 1/\sqrt2$ ceiling, and the encoder/decoder distinction).  Only the
supporting lemmas for the architecture-specific quantitative
$G$-bounds (softmax order preservation and the Toeplitz--Frobenius identity)
remain in the Online Supplement (\S18).  The Reproducibility Statement
(\cref{app:reproducibility}) documents code, data provenance, and compute.

\paragraph{Online Supplement (a separate companion document).}
The Online Supplement collects the extended empirical material that
\emph{supports} the main argument but is not needed to follow it; the main text
references it once per topic rather than per claim.  It contains: detailed AUROC
tables and per-quartile breakdowns (\S\S1,4); complementary diagnostics (\S2);
the length-confounding analysis and the per-method derivations for LLM-Check,
EigenTrack, and LapEigvals (\S\S3,6--8); dataset statistics (\S5); LC-AUROC
sensitivity (\S9); finite-size and random-matrix analysis (\S10); the
degree-sufficiency and matched-null validations (\S\S11,16); alternative
$G$-norms and the signed-directionality extension (\S\S12--13); supplementary
figures (\S14); the
bipartite-SVD analysis underlying the Hermitian dilation (\S15);
evaluation-metric definitions (\S17); supporting cross-attention proofs
(\S18); and diagnostics across model families and scales, including a modern-model
panel and the Pythia size sweep (\S19).

\ifnum\arxivversion=1

\section{Reproducibility}
\label{app:reproducibility}

\paragraph{Data provenance.}
HaluEval ($n{=}10{,}000$, HuggingFace \texttt{pminervini/HaluEval});
TruthfulQA ($n{=}817$, \texttt{truthfulqa/truthful\_qa});
MedHallu ($n{=}1{,}000$, \texttt{UTHealth/MedHallu}).
Preprocessing is deterministic via hash-based splitting.

\paragraph{Model checkpoints.}
All models are publicly available on HuggingFace. Mechanistic grid: GPT-2,
Pythia-160M, BERT-base-uncased, Flan-T5-base. Modern panel: Qwen2.5-0.5B/1.5B/3B,
SmolLM2-1.7B, LLaMA-3.1-8B; plus the Pythia 70M--1.4B within-family scaling sweep.

\paragraph{Compute environment.}
Experiments were run on a single NVIDIA A10G GPU (24\,GB GDDR6) with
16 vCPUs (AMD EPYC 7R32) and 64\,GiB system RAM. A single GPU is
sufficient for the full experiment suite at the model scales reported.
Software dependencies are pinned (Python 3.11, PyTorch 2.1,
Transformers 4.36). Random seeds are fixed per experiment configuration.

\else
  \InputIfFileExists{content/appendix_reproducibility}{}{%
    \PackageError{main}{content/appendix_reproducibility.tex is missing
      in a non-arXiv build}{Restore the file or build the arXiv wrapper.}}
\fi

\section{Structural Regime Propositions and Proofs}
\label{app:proofs_qk_regimes}

This appendix contains the proofs of the headline theorems cited in the main text:
the two-sided diagnostic theorem (\cref{app:proof-two-sided-diagnostic}) and the
asymmetry-energy proposition (\cref{app:proof-asymmetry-energy}); the
asymmetry-coefficient characterization ($G=0$ iff symmetric, the causal
$1/\sqrt{2}$ ceiling, and the encoder--decoder distinction) is proved in
\cref{app:g-formal-verification}.
The supporting cross-attention analysis ($n_q \neq n_k$, relevant only to
encoder--decoder architectures), the foundational concentration lemmas, and
the architecture-dependent quantitative $G$ bounds are in the Online Supplement
(\suppref{suppl:structural-regimes-extra}); decoder-only results in the main text do not depend
on that material.

\subsection{Proposition Statements}

\begin{proposition}[Degree imbalance induces length entanglement of $\sig_2$]
\label{prop:length-entanglement}
\emph{In cross-attention regimes:}
Assume $\B$ is row-stochastic and each query distributes mass approximately uniformly
over an effective support of size $s$.
As $n_q$ increases with $n_k$ fixed (regime $n_q\gg n_k$),
the spectrum of $\M$ (and in particular $\sig_2(\M)$) becomes strongly
entangled with generation length through column normalization.
\end{proposition}

\noindent\textbf{Theorem~\ref{thm:two-sided-diagnostic}} (Temporal-cut separation: window vs.\ uniform-causal).
\emph{See \cref{sec:cheeger} for the main statement.
The detailed proof with explicit conductance bounds appears in \cref{app:proof-two-sided-diagnostic}.}

\noindent\textbf{Proposition~\ref{prop:asymmetry-energy}} (Causal attention implies $G > 0$).
\emph{See \cref{sec:asymmetric-guessing} for the statement; proof follows in \cref{app:proof-asymmetry-energy}.}

\noindent\textbf{Theorem~\ref{prop:degree-sufficiency}} (Degree sufficiency for near-regular attention).
\emph{See \cref{sec:cheeger} for the full statement and proof.}

\medskip
\noindent\emph{Conductance transfer corollary.}
The Cheeger bridge inequality $\delta/2 \leq \phi \leq \sqrt{2\delta}$ (where
$\delta = 1 - \sigma_2$) transfers the spectral perturbation to conductance.
If the spectral gaps of $\M$ and $A/\!\sqrt{\bar{d}}$ differ by at most
$\varepsilon = \max_j|1 - \sqrt{d_j/\bar{d}}|$ ($\leq \sqrt{\kappa}-1$),
then for the upper Cheeger bound
$\phi(\M) \leq \sqrt{2\delta_\M}$ and the lower bound
$\delta_{\mathrm{ref}}/2 \leq \phi_{\mathrm{ref}}$, the conductance perturbation satisfies
\[
\phi(\M) - \phi_{\mathrm{ref}} \;\leq\; \sqrt{2(\delta_{\mathrm{ref}}+\varepsilon)} - \delta_{\mathrm{ref}}/2.
\]
For small $\varepsilon$, this is $O(\sqrt{\varepsilon})$.

\medskip
\noindent
The remainder of the structural-regime analysis (the foundational
technical lemmas ($D_K$-scaling and degree-balance), the proof of
\Cref{prop:length-entanglement}, the degree-ratio Cheeger bound, the
diffuse-key regime ($n_k \gg n_q$) analysis, the positional-encoding
quantitative $G$ bounds, and the empirical polarity analysis) is
deferred to the Online Supplement (\suppref{suppl:structural-regimes-extra}).
\subsection{Conductance Bounds and Proofs}
\label{app:conductance-bounds}

This section establishes conductance bounds for causal attention patterns, including
the window attention bottleneck bound and uniform causal conductance lower bound that
support Theorem~\ref{thm:two-sided-diagnostic}.

\subsubsection{Window Attention Conductance Bound}
\label{app:window-conductance}

\begin{lemma}[Window attention conductance bound]
\label{lem:window-conductance}
For $w$-window causal attention $\B_{ij} = (1/\min(w,i+1))\mathbf{1}\{\max(0,i-w+1)\le j\le i\}$
and any nontrivial temporal bipartite cut $S_t$ with $0<t<n$:
\[
\phi(S_t) \le \frac{w}{\min(t,n-t)}.
\]
In particular, for late-half cuts satisfying $n-t\le t$,
$\phi(S_t)\le w/(n-t)$.
\end{lemma}

\begin{proof}
Only queries $i \in [t, t+w-1]$ can contribute to the cut (by window structure).
Each contributes at most mass 1 (row-stochastic), so $\mathrm{cut}(S_t) \le w$.
The query side contributes volume at least $t$, and the query-side complement
contributes volume at least $n-t$; hence the Cheeger denominator is at least
\(\min(t,n-t)\).  If $n-t\le t$, this specializes to the denominator $n-t$.
\end{proof}

\subsubsection{Uniform Causal Conductance: Cut, Volume, and Conductance Identities}
\label{app:uniform-causal-conductance}

We work in the symmetric bipartite graph $\Hh(\B)$ with vertex set
$\{Q_0,\ldots,Q_{n-1}\}\cup\{K_0,\ldots,K_{n-1}\}$ and edge weights
$W_{Q_iK_j}=\B_{ij}$. The temporal cut $S_t=\{Q_0,\ldots,Q_{t-1}\}\cup
\{K_0,\ldots,K_{t-1}\}$ partitions vertices by index. The following
three propositions establish the closed-form arithmetic of conductance
on this cut family for uniform causal attention; the lemma cited by the
two-sided diagnostic is then a one-line corollary.

\begin{proposition}[Cut and volume identities for uniform causal attention]
\label{prop:uc-cut-vol-identities}
For uniform causal attention $\B_{ij}=(1/(i+1))\,\mathbf 1\{j\le i\}$
(rows and columns $0$-indexed, $i,j\in\{0,\dots,n-1\}$, so query $i$ attends
uniformly to its $i+1$ admissible keys; harmonic numbers
$H_m=\sum_{k=1}^{m}1/k$ use the standard $1$-indexed convention)
and the temporal cut $S_t$ with $0\le t\le n$:
\begin{equation}
\mathrm{cut}(S_t)\;=\;t\,(H_n-H_t),
\qquad
\vol(S_t)\;=\;t\,(2+H_n-H_t).
\label{eq:uc-cut-vol}
\end{equation}
\end{proposition}

\begin{proof}
The causal mask forces only $(Q_i,K_j)$ with $i\ge t,\,j<t$ to cross
the cut, so $\mathrm{cut}(S_t)=\sum_{i=t}^{n-1}\sum_{j=0}^{t-1}1/(i+1)
=t\,(H_n-H_t)$.
Vertex weights on the dilation graph are $W(Q_i)=\sum_j\B_{ij}=1$
(row-stochasticity) and $W(K_j)=\sum_i\B_{ij}=H_n-H_j$
(column sum). The volume of $S_t$ is $\sum_{i<t}1+\sum_{j<t}(H_n-H_j)$;
applying the harmonic-of-harmonic identity
$\sum_{j=0}^{t-1}H_j=tH_t-t$ \cite[Eq.~(6.69)]{GrahamKnuthPatashnik1994}
gives $\vol(S_t)=t+tH_n-(tH_t-t)=t\,(2+H_n-H_t)$.
\end{proof}

\begin{proposition}[Total volume of the bipartite dilation graph]
\label{prop:dilation-total-vol}
For any non-negative row-stochastic matrix $\B\in\reals^{n\times n}$,
the total volume of its bipartite dilation graph $\Hh(\B)$ equals $2n$.
Consequently, for any cut $S$ on the dilation graph,
$\min(\vol(S),\vol(\bar S))\le n$.
\end{proposition}

\begin{proof}
The total volume is $\sum_{i=0}^{n-1}W(Q_i)+\sum_{j=0}^{n-1}W(K_j)
=\sum_i\sum_j\B_{ij}+\sum_j\sum_i\B_{ij}=2\sum_{i,j}\B_{ij}=2n$,
using row-stochasticity. The min bound follows since
$\vol(S)+\vol(\bar S)=2n$.
\end{proof}

\paragraph{Proof structure (closed-form landscape).}
\begin{description}[nosep,leftmargin=*]
\item[Foundations.]
Harmonic-number inequality $H_n - H_{\lfloor n/2\rfloor}\ge 1/2$
\cite[Eq.~(6.69)]{GrahamKnuthPatashnik1994};
the bipartite Cheeger conductance definition
$\phi(S)=\mathrm{cut}(S)/\min(\vol(S),\vol(\bar S))$.
\item[Bridge.]
Bipartite dilation $\Hh(\B)$ as the symmetric graph carrying the
asymmetric attention pattern (\cref{prop:dilation-total-vol}); cut and
volume identities $\mathrm{cut}(S_t)=t(H_n-H_t)$,
$\vol(S_t)=t(2+H_n-H_t)$ (\cref{prop:uc-cut-vol-identities}).
\item[Contribution.]
The closed-form functional $\phi(S_t)\ge u(t)/(2+u(t))$
(\cref{lem:uniform-causal-conductance}) and its $n$-independent floor
$\phi\ge 1/5$ for every $n\ge 2$ (\cref{cor:uc-one-fifth}).  This is
the architectural benchmark against which the empirical landscape
signature is read: real attention heads piercing $1/5$ on HaluEval
form the population fraction (\cref{tab:landscape-empirical-summary})
that distinguishes position-encoding regimes.  Together with the
window-attention bottleneck (\cref{lem:window-conductance}), it
delivers the polarity reversal of the two-sided diagnostic
(\cref{thm:two-sided-diagnostic}).
\end{description}

\begin{lemma}[Uniform causal conductance, functional form]
\label{lem:uniform-causal-conductance}
For uniform causal attention and any temporal cut $0<t<n$, writing
$u(t)=H_n-H_t$,
\begin{equation}
\phi(S_t)\;\ge\;\frac{u(t)}{2+u(t)}.
\label{eq:phi-uc-functional}
\end{equation}
This bound is uniform: it holds in both the small-side regime
($\vol(S_t)\le\vol(\bar S_t)$, where it is achieved with equality)
and the large-side regime (where it is strict).
\end{lemma}

\begin{proof}
By definition $\phi(S_t)=\mathrm{cut}(S_t)/\min(\vol(S_t),\vol(\bar S_t))$.
\Cref{prop:uc-cut-vol-identities} gives $\mathrm{cut}(S_t)=t\,u$ and
$\vol(S_t)=t\,(2+u)$. Since $\min(\vol(S_t),\vol(\bar S_t))\le\vol(S_t)$,
\[
\phi(S_t)\;\ge\;\frac{\mathrm{cut}(S_t)}{\vol(S_t)}
\;=\;\frac{t\,u}{t\,(2+u)}
\;=\;\frac{u}{2+u}.
\]
In the small-side regime, $\min=\vol(S_t)$ and the inequality is
equality; in the large-side regime, $\min=\vol(\bar S_t)<\vol(S_t)$,
so dividing by the smaller denominator gives strict inequality.
\end{proof}

\Cref{lem:uniform-causal-conductance} characterises conductance as a
monotone function of the harmonic-mass-above-$t$ parameter
$u=H_n-H_t$.  Two corollaries record the numerical consequences.

\begin{corollary}[$n$-independent conductance lower bound]
\label{cor:uc-one-fifth}
For uniform causal attention and any $0<t<n$ with $n\ge 2$,
\[
\phi(S_t)\;\ge\;\frac{1}{5}.
\]
\end{corollary}

\begin{proof}
The proof is by a two-branch split on the volume ordering.
If $\vol(S_t)\le\vol(\bar S_t)$, then the volume identity
\eqref{eq:uc-cut-vol} forces $2t\le n$. Hence
$t\le\lfloor n/2\rfloor$ and
\[
u(t)=H_n-H_t \;\ge\; H_n-H_{\lfloor n/2\rfloor}
\;\ge\; \frac12,
\]
where the last inequality is the elementary harmonic-tail bound that
the sum has at least $\lceil n/2\rceil$ terms, each at least $1/n$.
The small-side formula therefore gives
$\phi(S_t)\ge u(t)/(2+u(t))\ge 1/5$.

If instead $\vol(\bar S_t)\le\vol(S_t)$, then
$\phi(S_t)=t\,u(t)/(2n-t(2+u(t)))$, and $\phi(S_t)\ge1/5$ is equivalent to
$2n\le t\,(2+6u(t))$.  The volume ordering itself gives
$2n\le 2\vol(S_t)=t\,(4+2u(t))$.  For $2t\le n$ the harmonic-tail bound
$u(t)\ge1/2$ yields $4+2u(t)\le 2+6u(t)$, so
$2n\le t\,(4+2u(t))\le t\,(2+6u(t))$.  For $n<2t$ the lower-tail estimate
$u(t)=H_n-H_t\ge(n-t)/n$ reduces the target $2n\le t\,(2+6u(t))$ to
$(n-t)(n-3t)\le 0$, which holds because $t>n/2>n/3$.  In both subcases
$\phi(S_t)\ge1/5$.
\end{proof}

\begin{corollary}[Sharp asymptotic sweep floor]
\label{cor:uc-sharp-floor}
In the $n\to\infty$ limit, writing $x=t/n$, the closed-form sweep-conductance landscape
$\phi(S_t)$ converges to a continuous profile whose greatest lower bound is the exact
constant
\[
c^\star \;=\; 1 - 2x^\star \;=\; 1 + \frac{2}{W_{-1}(-e^{-2})}\;\approx\;0.36431,
\]
where $x^\star\in(0,1)$ is the unique root of $x(2-\ln x)=1$ (so $x^\star\approx 0.31784$)
and $W_{-1}$ denotes the lower branch of the Lambert $W$ function. The bound is attained
in the limit at the volume crossover $t/n\to x^\star$.  The constant $c^\star$ is the exact
asymptotic value of the landscape minimum, not a finite-$n$ lower bound: at small $n$ the
landscape dips below it ($5/17\approx0.294$ at $n=3$), and the uniform finite-$n$ floor
remains the $1/5$ of \cref{cor:uc-one-fifth}.
\end{corollary}

\begin{proof}
By \cref{lem:uniform-causal-conductance} the small-side value $u(t)/(2+u(t))$ is
decreasing in the cut index and the large-side value is increasing, so the landscape
minimum sits at the volume crossover, where $x^\star(2-\ln x^\star)=1$ forces
$c^\star=1-2x^\star$.  The large-side inequality $\phi(S_t)\ge c^\star$ is the concavity
gap of $x\mapsto -x\ln x$ above the chord through $(x^\star,c^\star)$ and $(1,0)$, with
denominator positivity from $\ln x\ge 1-1/x$.  The harmonic asymptotic
$H_n-H_{t_n}\to-\ln x$ for $t_n/n\to x$ identifies this continuous profile as the
$n\to\infty$ limit of $\phi(S_t)$.
\end{proof}

\subsubsection{Proof of Two-Sided Diagnostic Theorem}
\label{app:proof-two-sided-diagnostic}

\paragraph{Statement (restated).}
In square, causally masked attention ($n_q=n_k=n$ with lower-triangular support), incorrect behavior
can arise from (i) concentration on a narrow temporal band (bottleneck) yielding $\phihat\downarrow$,
or (ii) coherent but misguided routing over a broad temporal region yielding $\phihat$ moderate/high.
Therefore $\phihat$ does not admit a universal polarity with correctness in decoder self-attention.

\paragraph{Proof structure.}
\begin{description}[nosep,leftmargin=*]
\item[Foundations.]
The bipartite Cheeger conductance $\phi(S)=\mathrm{cut}(S)/\min(\vol(S),\vol(\bar S))$
on the symmetric dilation graph $\Hh(\B)$.
\item[Bridge.]
Two canonical causal families serving as failure-mode prototypes:
$w$-window attention (\cref{lem:window-conductance}) and uniform
causal attention (\cref{prop:uc-cut-vol-identities,lem:uniform-causal-conductance,cor:uc-one-fifth});
the temporal cut $S_t$ as the index-aligned partition revealing routing
behaviour at every position.
\item[Contribution.]
Both families satisfy the same causal mask, yet their conductance
landscapes are qualitatively opposite: window attention pierces the
Cheeger floor uniformly ($\phi\to 0$), while uniform causal attention
respects an $n$-independent floor $\phi\ge 1/5$.  Polarity of $\phihat$
with respect to correctness therefore cannot be universal in decoder
self-attention; it is determined by which failure-mode prototype
dominates, not by task semantics.  This is the structural reason the
empirical polarity reverses across HaluEval (bottleneck-dominated) and
MedHallu (diffuse-dominated).
\end{description}

\paragraph{Proof.}
We establish quantitative bounds for two canonical causal attention families, both satisfying
the causal constraint $\B_{ij} = 0$ for $j > i$.

\textbf{(i) Bottleneck family (Lemma~\ref{lem:window-conductance}).}
Let $\B$ be $w$-window causal attention:
\[
\B_{ij} = \frac{1}{\min(w, i+1)}\mathbf{1}\{\max(0,i-w+1)\le j\le i\}.
\]
For the temporal cut $S_t = \{k_0,\ldots,k_{t-1}\}$ separating early from late keys:
\begin{itemize}[nosep,leftmargin=*]
\item Only queries $i \in [t, t+w-1]$ have windows crossing the cut boundary
\item Each such query contributes at most mass 1 (row-stochastic)
\item Total cut weight: $\mathrm{cut}(S_t) \le w$
\item The Cheeger denominator is at least $\min(t,n-t)$
\end{itemize}
Thus the conductance satisfies:
\[
\phi(S_t) \le \frac{w}{\min(t,n-t)}.
\]
For late-half cuts $n-t\le t$, this becomes
\(\phi(S_t)\le w/(n-t)\); choosing balanced cuts gives \(O(w/n)\).

\textbf{(ii) Diffuse family
(\cref{prop:uc-cut-vol-identities,lem:uniform-causal-conductance,cor:uc-one-fifth}).}
Let $\B$ be uniform causal attention:
\[
\B_{ij} = \frac{1}{i+1}\mathbf{1}\{j\le i\}.
\]
\Cref{prop:uc-cut-vol-identities} delivers the closed-form
$\mathrm{cut}(S_t)=t\,(H_n-H_t)$ and $\vol(S_t)=t\,(2+H_n-H_t)$.
\Cref{lem:uniform-causal-conductance} composes these into the
functional bound $\phi(S_t)\ge u(t)/(2+u(t))$ where $u(t)=H_n-H_t$.
\Cref{cor:uc-one-fifth} uses the formal small-side/large-side split
and harmonic-tail lower bounds to yield the $n$-independent constant
\[
\phi(S_t) \;\ge\; \tfrac{1}{5} \;>\; 0\qquad\text{for all }0<t<n,\,n\ge 2.
\]

Both families satisfy the same causal constraint but exhibit opposite
conductance behaviour. The bottleneck family pierces the Cheeger floor on
balanced cuts, while the diffuse family has $\phi(S_t)\ge 1/5$ on
\emph{every} cut. For such balanced cuts,
\[
\frac{\phi_{\text{uniform}}}{\phi_{\text{window}}}
\;\ge\; \frac{\min(t,n-t)}{5w} \;\to\; \infty
\quad\text{as }w/n\to 0
\]
whenever $\phi_{\text{window}}>0$; if $\phi_{\text{window}}=0$, the
separation is even stronger.
This establishes that $\phihat$ does not admit a universal polarity in decoder self-attention:
the conductance value depends on the attention \emph{pattern} within the causal mask,
not on task correctness.
\hfill$\square$

The Online Supplement (\suppref{suppl:fig-decoder-failure-modes}) illustrates how decoder self-attention mixes multiple failure modes, explaining polarity instability.

\subsection{Proofs: Asymmetric Transport}
\label{app:asymmetric-transport-proofs}

\subsubsection{Proof of Asymmetry Energy Proposition}
\label{app:proof-asymmetry-energy}

\paragraph{Statement (restated).}
Let $\Masym=\tfrac12(\M-\M^\top)$ and $G=\|\Masym\|_F/(\|\M\|_F+\varepsilon)$. In square causal regimes,
$G$ decreases when transport becomes effectively reversible (forward/backward cancellation after normalization),
even if conductance remains moderate.

\paragraph{Proof structure.}
\begin{description}[nosep,leftmargin=*]
\item[Foundations.]
Orthogonality of the symmetric and antisymmetric subspaces under the
Frobenius inner product, yielding the Pythagorean identity
$\|\M\|_F^2 = \|\Msym\|_F^2 + \|\Masym\|_F^2$.
\item[Bridge.]
Definition of $G(\M,\varepsilon) = \|\Masym\|_F/(\|\M\|_F+\varepsilon)$
as the antisymmetric-residual norm (\cref{sec:asymmetric-guessing});
restriction to square regimes ($n_q=n_k$) where the symmetric--antisymmetric
decomposition is well-posed.
\item[Contribution.]
$G$ is monotone non-increasing under symmetrisation of $\M$ and
detects orientation collapse that conductance cannot see: $\phihat$ can
remain moderate while $G\downarrow$, because conductance depends on
cut capacity (a symmetric notion) rather than on irreversibility.
This is the converse statement to orientation-blindness:
\emph{symmetric methods cannot detect orientation; $G$ can}.
\end{description}

\paragraph{Proof.}
By definition,
\[
G = \frac{\|\M-\M^\top\|_F}{2(\|\M\|_F+\varepsilon)}.
\]
Thus $G=0$ if and only if $\M$ is symmetric. More generally, if $\M$ admits a decomposition
\[
\M = S + R,\quad S^\top=S,\quad R^\top=-R,
\]
then $\Masym = R$ and $\|\M\|_F^2=\|S\|_F^2+\|R\|_F^2$ (orthogonality of symmetric and skew parts under Frobenius inner product).
Hence
\[
G = \frac{\|R\|_F}{\sqrt{\|S\|_F^2+\|R\|_F^2}+\varepsilon}.
\]
If transport becomes more reversible in the sense that $\|R\|_F$ decreases while $\|S\|_F$ stays bounded away from zero,
then $G$ decreases. This can occur without inducing a bottleneck (i.e., without reducing conductance), because conductance
depends primarily on cut capacity (a symmetric notion) rather than on irreversibility.

\paragraph{Link to symmetric limitation.}
By \cref{thm:orientation-blindness}, any symmetric embedding derived from $H(\M)$ is invariant to transposition
and therefore cannot encode irreversibility. Thus $\phihat$ can remain moderate while $G$ decreases, motivating $G$ as a
complementary axis.
\hfill$\square$

\paragraph{Domain restriction (square matrices only).}
The asymmetry energy $G = \|\Masym\|_F/(\|\M\|_F+\varepsilon) = \|\M - \M^\top\|_F / (2(\|\M\|_F + \varepsilon))$ is defined only when $\M$ is square ($n_q = n_k$).
In cross-attention with $n_q \neq n_k$, the difference $\M - \M^\top$ is undefined because the transpose $\M^\top \in \mathbb{R}^{n_k \times n_q}$
has different dimensions than $\M \in \mathbb{R}^{n_q \times n_k}$.
Therefore, $G$ applies exclusively to:
\begin{enumerate}[nosep,leftmargin=*]
\item \textbf{Decoder self-attention}: inherently square due to causal masking ($n_q = n_k = n$ at each position).
\item \textbf{Encoder self-attention}: square by construction (queries and keys from the same sequence).
\end{enumerate}
For non-square cross-attention matrices, orientation diagnostics require alternative measures such as comparing
the left and right singular vectors of $\M$ directly, rather than relying on the symmetric--antisymmetric decomposition.

The Online Supplement (\suppref{suppl:fig-asymmetry-energy}) illustrates how $G$ captures orientation failures that $\phihat$ misses due to its symmetry.

\section{Characterization of the Asymmetry Coefficient $G$}
\label{app:g-formal-verification}

The elementary characterizations of $G$ are recorded here. For a non-zero
operator $\M$,
\begin{equation}
G(\M,0)=0 \iff \M=\M^\top,
\qquad
G(\M,0)=\frac{\|\M-\Msym\|_F}{\|\M\|_F}=\frac{d(\M,\mathcal S)}{\|\M\|_F},
\end{equation}
so $G$ is exactly the normalised Frobenius distance to the symmetric subspace
$\mathcal S$ (immediate from the orthogonality of $\Msym$ and $\Masym$).  Degree
normalization preserves lower-triangularity and symmetry, so these statements
apply verbatim to the degree-normalized operator.  Two consequences are worth
isolating.

\paragraph{A causal ceiling.}
\begin{proposition}[Asymmetric ceiling for causal Toeplitz attention]
\label{prop:g-causal-toeplitz-ceiling}
For a causal Toeplitz operator $\M_{ij}=f(i-j)\,\mathbf 1\{j\le i\}$ with $\M\neq0$,
$G(\M,0)\le 1/\sqrt2$, with equality iff the diagonal kernel vanishes ($f(0)=0$).
\end{proposition}
\begin{proof}
Writing $S=\sum_{i>j}f(i-j)^2>0$, the causal Toeplitz Frobenius identity gives
$\|\Masym\|_F^2=\tfrac12 S$ and $\|\M\|_F^2=S+n\,f(0)^2$, so
$G(\M,0)^2=\tfrac12 S/(S+n f(0)^2)\le \tfrac12$, with equality iff $f(0)=0$.  The
value $G=1$ would require a purely skew $\M$, impossible for a non-negative
attention matrix with shared-support mass.
\end{proof}
\noindent Causal masking therefore does \emph{not} drive $G$ to its maximum: the
causal class is capped at $1/\sqrt2$ (the strict-lower-triangular point), and any
self-attention mass pulls $G$ strictly below it.

\paragraph{Encoder vs.\ decoder.}
Because $G=0$ means symmetric, its diagnostic reading is architecture-specific.  A
causal $\M$ with $G=0$ must be diagonal (a symmetric causal matrix has no
off-diagonal mass), so for decoder self-attention $G=0$ flags temporal isolation;
for encoder attention $G=0$ is the expected symmetric baseline and $G>0$ a
directional bias.  This split matches the training-dynamics finding that
bidirectional objectives drive attention toward symmetry and autoregressive
objectives toward directional structure~\citep{Saponati2025SelfAttentionStructures}.
\begin{center}
\footnotesize
\begin{tabular}{@{}llcl@{}}
\toprule
\textbf{Attention} & \textbf{Pattern} & $G$ & \textbf{Interpretation} \\
\midrule
Decoder & Diagonal & $0$ & Temporal isolation \\
Decoder & Causal + history & ${>}0$ & Using context \\
Encoder & Symmetric & $0$ & Normal \\
Encoder & Asymmetric & ${>}0$ & Directional bias \\
\bottomrule
\end{tabular}
\end{center}


\end{document}